\documentclass[twoside]{article}
\usepackage[accepted]{aistats2024}
\usepackage[final,pagebackref=true]{hyperref}   
\usepackage{url}
\usepackage{textcomp}
\usepackage{amsthm}
\usepackage{mathtools}
\usepackage{algpseudocode}
\usepackage{algorithm}
\usepackage{thmtools} 
\usepackage{thm-restate}
\usepackage{wrapfig}
\usepackage{enumitem}
\usepackage{natbib}
\usepackage{dsfont}
\usepackage{fontenc}
\usepackage[rgb,dvipsnames]{xcolor}
\usepackage{caption}
\usepackage{subcaption}
\usepackage{nicefrac}
\usepackage{float}

\hypersetup{
	colorlinks=true,       
	linkcolor=Blue,        
	citecolor=Blue,        
	filecolor=magenta,     
	urlcolor=Blue         
}

\usepackage{amsmath,amsfonts,bm,dsfont}

\def\eqref#1{equation~\ref{#1}}

\def\1{\bm{1}}

\DeclareMathAlphabet{\mathsfit}{\encodingdefault}{\sfdefault}{m}{sl}
\SetMathAlphabet{\mathsfit}{bold}{\encodingdefault}{\sfdefault}{bx}{n}

\def\gA{{\mathcal{A}}}
\def\gB{{\mathcal{B}}}

\def\gI{{\mathcal{I}}}

\def\gL{{\mathcal{L}}}
\def\gM{{\mathcal{M}}}
\def\gN{{\mathcal{N}}}
\def\gO{{\mathcal{O}}}
\def\gP{{\mathcal{P}}}

\def\gS{{\mathcal{S}}}

\def\gW{{\mathcal{W}}}
\def\gX{{\mathcal{X}}}

\def\sE{{\E}}

\def\sN{{\mathbb{N}}}

\def\sP{{\mathbb{P}}}

\def\sR{{\mathbb{R}}}

\newcommand{\E}{\mathbb{E}}

\DeclareMathOperator*{\argmin}{arg\,min}

\renewcommand{\epsilon}{\varepsilon}
\newcommand{\eqdef}{:=}

\newcommand{\nnopt}[1]{\mathbb{E}_{(x, y)\sim P}[\mathcal{L}(\hat{f}(x, #1), y)]}

\usepackage[framemethod=TikZ]{mdframed}
\usepackage{cleveref}
\usepackage{accents}

\mdfdefinestyle{MyFrame}{%
    linecolor=black,
    outerlinewidth=.3pt,
    roundcorner=5pt,
    innertopmargin=1pt, 
    innerbottommargin=1pt, 
    innerrightmargin=1pt,
    innerleftmargin=1pt,
    backgroundcolor=black!0!white}

\mdfdefinestyle{MyFrame2}{%
    linecolor=white,
    outerlinewidth=1pt,
    roundcorner=2pt,
    innertopmargin=\baselineskip,
    innerbottommargin=\baselineskip,
    innerrightmargin=10pt,
    innerleftmargin=10pt,
    backgroundcolor=black!3!white}

\DeclareMathOperator\Diag{Diag}
\DeclareMathOperator\tr{tr}

\newcommand{\ubar}[1]{\underaccent{\bar}{#1}}

\newtheorem{theorem}{Theorem}[section]
\newtheorem{lemma}[theorem]{Lemma}

\newtheorem{definition}[theorem]{Definition}
\newcommand{\cut}[1]{}

\newtheorem{property}{Property}
\newtheorem{assumption}{Assumption}
\crefname{assumption}{Assumption}{Assumptions}
\Crefname{fig}{Figure}{Figures}
\crefname{property}{Property}{Properties}

\usepackage[framemethod=TikZ]{mdframed}
\mdfdefinestyle{MyFrame}{%
    linecolor=black,
    outerlinewidth=.3pt,
    roundcorner=5pt,
    innertopmargin=1pt, 
    innerbottommargin=1pt, 
    innerrightmargin=1pt,
    innerleftmargin=1pt,
    backgroundcolor=black!0!white}

\mdfdefinestyle{MyFrame2}{%
    linecolor=white,
    outerlinewidth=1pt,
    roundcorner=2pt,
    innertopmargin=\baselineskip,
    innerbottommargin=\baselineskip,
    innerrightmargin=10pt,
    innerleftmargin=10pt,
    backgroundcolor=black!3!white}

\def\diam{{\text{diam}}}
\def\Supp{{\text{Supp}}}
\def\Lip{{\text{Lip}}}
\def\err{{\text{err}}}
\def\rank{{\text{rank}}}
\def\Leb{{\text{Leb}}}
\def\Dim{{\text{Dim}}}
\def\Span{{\text{Span}}}

\def\Case{{\text{Case}}}
\usepackage{mdframed}
\newmdenv[topline=false,rightline=false]{leftbot}

\newcommand{\ssim}{\overset{\text{i.i.d.}}{\sim}}

\begin{document}

\twocolumn[

\aistatstitle{Proving Linear Mode Connectivity of Neural Networks \\ via Optimal Transport}

\aistatsauthor{Damien Ferbach$^{1,3}$
\And Baptiste Goujaud$^{2}$ \And Gauthier Gidel $^{1,\dag}$ \And Aymeric Dieuleveut$^{2}$}

\aistatsaddress{Mila, Université de Montréal$^1$ \And CMAP, Ecole Polytechnique, IPP$^{2}$ \And ENS Paris, PSL$^3$ }

]

\begin{abstract}
The energy landscape of high-dimensional non-convex optimization problems is crucial to understanding the effectiveness of modern deep neural network architectures. Recent works have experimentally shown that two different solutions found after two runs of a stochastic training are often connected by very simple continuous paths (e.g., linear) modulo a permutation of the weights. In this paper, we provide a framework theoretically explaining this empirical observation. Based on convergence rates in Wasserstein distance of empirical measures, we show that, with high probability, two wide enough two-layer neural networks trained with stochastic gradient descent are linearly connected.
Additionally, we express upper and lower bounds on the width of each layer of two deep neural networks with independent neuron weights to be linearly connected. Finally, we empirically demonstrate the validity of our approach by showing how the dimension of the support of the weight distribution of neurons, which dictates Wasserstein convergence rates is correlated with linear mode connectivity.
\end{abstract}
\section{INTRODUCTION AND RELATED WORK}
\label{section:intro}
Training deep neural networks on complex tasks is a high-dimensional, non-convex optimization problem. While stochastic gradient-based methods (i.e., SGD and its derivatives) have proven highly efficient in finding a local minimum with low test error, the loss landscape of deep neural networks (DNNs) still contains numerous open questions. In particular,~\citet{goodfellow2014qualitatively} try to find ways to connect two local minima
reached by two independent runs of the same
stochastic algorithm with different initialization
and data orders. This problem has applications in diverse domains such as model averaging \citep{izmailov2018averaging,rame2022diverse,wortsman2022model}, loss landscape study \citep{gotmare2018using,vlaar2022can,lucas2021analyzing}, adversarial robustness \citep{zhao2020bridging} or generalization theory \citep{pittorino2022deep,juneja2022linear,lubana2023mechanistic}.

An answer to this question is the \emph{mode connectivity phenomenon}. It suggests the existence of a continuous low-loss path connecting all the local minima found by a given optimization procedure. The mode connectivity phenomenon has extensively been studied in the literature~\citep{goodfellow2014qualitatively,keskar2016large,sagun2017empirical,venturi2019spurious,neyshabur2020being,tatro2020optimizing,yunis2022convexity,zhou2023going} and \emph{non-linear connecting paths} have been evidenced for DNNs trained on MNIST and CIFAR10 by~\citet{freeman2016topology,garipov2018loss, pmlr-v80-draxler18a}.

\paragraph{(Linear) mode connectivity.}
Formally, let $A \eqdef \hat{f}(., \theta_A)$ and $B \eqdef \hat{f}(., \theta_B)$ two neural networks sharing a common architecture $\hat{f}$. They are parametrized by $\theta_A$ and $\theta_B$ after training those networks on a data distribution $P$ with loss $\mathcal{L}$, i.e. by minimizing $\mathcal{E}(\theta) \eqdef \nnopt{\theta}$ over $\theta$.
Let $p$ be a continuous path connecting $\theta_A$ and $\theta_B$, i.e. a continuous function defined on $[0, 1]$ with $p(0) = \theta_A$ and $p(1) = \theta_B$. \citet{frankle2020linear} initially identified the problem of linear mode connectivity and defined the \emph{error barrier height}~\citep{frankle2020linear,entezari2021role} of $p$ as $\sup_{t \in[0, 1]} \mathcal{E}\left(p(t)\right) - \left((1-t)\mathcal{E}\left(\theta_A\right) + t\mathcal{E}\left(\theta_B\right)\right)$.
The two found solutions $\theta_A$ and $\theta_B$ are said to be \emph{mode connected} if there is a continuous path with zero error barrier height connecting them.
Furthermore if $p$ is linear, that is $p(t) = (1-t)\theta_A + t \theta_B,\,$ $\theta_A$ and $\theta_B$ are said to be \emph{linearly mode connected (LMC)}.

\begin{figure}
\begin{center}
    \includegraphics[width=\linewidth]{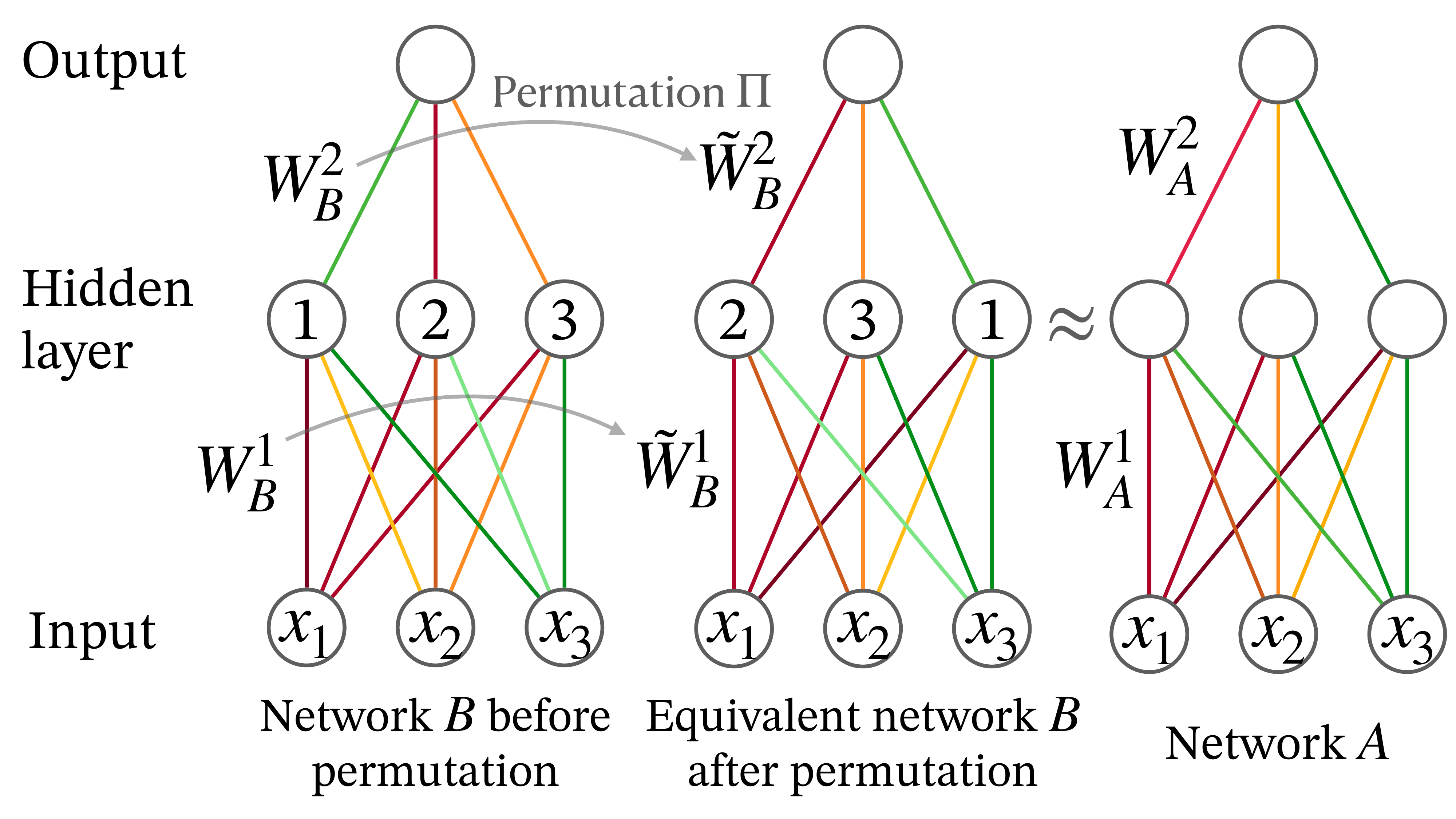}
    \vspace{-23pt}
    \caption{Permuting the neurons in the hidden layer of network $B$ to align them on network $A$}
    \label{fig:example_permutation}
    \vspace{-23pt}
\end{center}
\end{figure}

\paragraph{Permutation invariance.}
Recently,~\cite{singh2020model,ainsworth2022git} 
highlighted the fact that the units in a hidden layer of a given model can be permuted while preserving the network's functionality. \Cref{fig:example_permutation} shows how one can permute the hidden layer of a two-layer network to match a different target network without changing the source function. From now on, \textbf{we will understand LMC modulo permutation invariance}, i.e. two networks $A,B$ are said to be linear mode connected whenever there exists a permutation of neurons in each hidden layer of network $B$ such that the linear path in parameter space between network $A$ and $B$ permuted has low loss.

\paragraph{Linear mode connectivity up to permutation.}

\looseness=-1
\citet{singh2020model} proposed to use optimal transport (OT) theory to find soft alignment providing a ``good match'' (in a certain sense) between the neurons of two trained DNNs. Furthermore, the authors propose ways to fusion the aligned networks together in a federated learning context with local-steps. \citet{ainsworth2022git} further experimentally studied linear mode connectivity between two pre-aligned networks.
The authors first align network B's weights on the weights of network A before connecting both of them by a linear path in the parameter space. 
They notably achieved zero-loss barrier for two trained Resnets with SGD on CIFAR10. Moreover, their experiments strongly suggest that the error barrier on a linear path gets smaller for wider networks, with a detrimental effect of big depth.

\paragraph{Prior theoretical explanations.} A recent work by \cite{kuditipudi2019explaining} shows that dropout stable networks (i.e. networks that are functionally stable to the action of randomly setting a fraction of their weights and normalizing the others) exhibit mode connectivity. \cite{shevchenko2020landscape} use a mean field viewpoint to show that wide two-layer neural networks trained with SGD are dropout stable and hence show (non-linear) mode connectivity for two-layer neural networks in the mean field regime (i.e. one single wide hidden layer).
Finally~\cite{entezari2021role} show that two-layer neural networks exhibit linear mode connectivity up to permutation at initialization for parameters initialized following uniform independent distribution properly scaled. They highlight that this result could be extended to networks trained in the Neural Tangent Kernel regime where parameters stay close to initialization \citep{jacot2018neural}.

\paragraph{Contributions.} This paper aims at building theoretical foundations on the phenomenon of linear mode connectivity up to permutation. More precisely, we theoretically prove this phenomenon arises naturally on multi-layer perceptrons (MLPs), which goes beyond two-layer networks on which theoretical works focused so far. We also provide a new efficient way to find the right permutation to apply on the units of a neural network's layer. The paper is organized as follow:
\begin{itemize}[leftmargin=*]
    \item In \Cref{section:mean_field}, we focus on two-layer neural networks in the mean field regime. While~\cite{shevchenko2020landscape} proved \emph{non-linear} mode connectivity in this setting; we go further by proving \emph{linear mode connectivity up to permutation}. Moreover, we provide an upper bound on the minimal width of the hidden layer to guarantee linear mode connectivity.
    \item In \Cref{section:general_study}, we use general OT theory to exhibit tight asymptotics on the minimal width of a multilayer perceptron (MLP) to ensure LMC.
    \item In \Cref{section:given_distribution}, we apply our general results to networks with parameters following sub-Gaussian distribution. Our result holds for deep networks, generalizing the result of~\cite{entezari2021role} with better bounds. We shed light on the dependence in the dimension of the underlying distributions of the weights in each layer and explain how it connects with previous empirical observations~\citep{ainsworth2022git}. Using a model of approximately low dimensional weight distribution as a proxy of sparse feature learning, we yield more realistic bounds on the architectures of DNNs to ensure linear mode connectivity. We therefore, show why LMC is possible after training and how it depends on the complexity of the task. Finally we unify our framework with dropout stability.
    \item In~\Cref{section:experiments}, we validate our theoretical framework by showing how the implicit dimension of the weight distribution is correlated with linear mode connectivity for MLPs trained on MNIST with SGD and propose a new weight matching method.
\end{itemize}

\section{PRELIMINARIES AND NOTATIONS}
\label{section:preliminaries}

\looseness=-1
\textbf{Notations.}
Let two multilayer perceptrons (MLP) $A$ and $B$ with the same depth $L+1$ ($L$ hidden layers), an input dimension $m_0$, intermediate widths $m_1,...,m_L$ and an output dimension $m_{L+1}$. Given $2(L+1)$ weights matrices $W_{A,B}^{1,...,L+1}$, and a non-linearity $\sigma$, we define the neural network function of network A by $\hat f_{A}$ (respectively $\hat{f}_B$): $\forall x \in \sR^{m_0}$,
\begin{equation}
\label{eq:network}
    \hat{f}_A(x)\eqdef \hat{f}(x;\theta_{A})\eqdef W_A^{L+1}\sigma\left(W_A^L\ldots \sigma(W_A^1x)\right)
\end{equation}

To $W_A^\ell\in \gM_{m_{\ell}, m_{\ell-1}}(\sR)$ we associate $\hat{\mu}_{A,\ell}$  the empirical measure of its rows $[W^\ell_{A}]_{i:} \in \sR^{m_{\ell-1}}$: $\frac{1}{m_{\ell}}\sum_{i=1}^{m_{\ell}}\delta_{[W^\ell_{A}]_{i:}}$ 
which belongs to the space of probability measures 
$\gP_1(\sR^{m_{\ell-1}})$, where $[W^\ell_{A}]_{i:}$ is the $i$-th row of the matrix and $\delta$ denotes the Dirac measure. Note that $[W^\ell_{A}]_{i:}$ is also the weights vector of the $i$-th neuron of the layer $\ell$ of network $A$. Given an equi-partition\footnote{All subsets have the same number of elements} $\gI^{\ell-1}=\{I^{\ell-1}_1,...I^{\ell-1}_{\Tilde{m}_{\ell-1}}\}$ of $[m_{\ell-1}]$ we denote $W_A^{\gI^{\ell-1}}\in \gM_{m_{\ell}, \Tilde{m}_{\ell-1}}(\sR)$ the matrix issued from $W_A^\ell$ where we have summed the columns being in the same set of the partition $\gI^{\ell-1}$. In that case $\hat{\mu}_A^{\gI^{\ell-1}}\in \gP_1\left(\sR^{\Tilde{m}_\ell}\right)$ denotes the associate empirical measure of its rows.

\looseness=-1
Denote $\phi_A^\ell(x)\eqdef \sigma\left(W_A^{\ell}\ldots \sigma(W_A^1x)\right)$ (respectively $\phi^\ell_B$) the activations of neurons at layer $\ell$ of network $A$ on input $x$.
The data $x$ follows a distribution $P$ in $\sR^{m_0}$.

Given permutations matrices $\Pi_\ell \in \gS_{m_\ell},$\footnote{We use interchangeably $\gS_m$ to denote the space of permutations of $\{1,\ldots,m\}$ and the corresponding space of permutations matrices. Given $\pi \in \gS_m$ its corresponding permutation matrix $\Pi$ is defined as $\Pi_{ij} = 1 \iff \pi(i) = j$.} $\ell=1,\ldots,L$ of each hidden layer of network $B$, the weight matrix at layer $\ell$ of the permuted network $B$ is $\Tilde{W}_B^\ell \eqdef \Pi_\ell W_B^\ell \Pi_{\ell-1}^T$ and its new activation vector is $\Tilde{\phi}_B^\ell(x) \eqdef \Pi_\ell \phi^\ell_B(x)$. 
Finally, $\forall t \in [0,1]$ we define $M_t$ the convex combination of network $A$ and $B$ permuted, with weights matrices $tW^\ell_A+(1-t)\Tilde{W}^\ell_B$ and $\phi_{M_t}^\ell$ its activations at layer $\ell$.

\textbf{Preliminaries.}
We consider networks $A$ and $B$ to be independently chosen from the same distribution $Q$ on parameters. This is coherent with considering two networks initialized independently or trained independently with the same optimization procedure (\S\ref{section:mean_field}). We additionally suppose the choice of $A$ and $B$ to be independent of the choice of $x\sim P$, which is valid when evaluating models on test data not seen during training. We denote $\E_Q, \E_P, \E_{P,Q}$ expectations with respect to the choice of the networks, the data, or both.

To show linear mode connectivity of networks $A$ and $B$ we will show the existence of permutations $\Pi_{1},...,\Pi_{L}$ of layers $1,...,L$ that align the neurons of network $B$ on the closest neurons weights of network $A$ at the same layer as shown in \Cref{fig:example_permutation}. In other words, we want to find permutations that minimize for each layer $\ell \in [L]$ the norm $\|W_A^\ell-\Pi_\ell W_B^\ell \Pi_{\ell -1}^T\|_2$.
Recursively on $\ell$, we solve the following optimization problem:
\begin{equation}
    \label{eq:P_l}
    \begin{split}
            \Pi_\ell&=\argmin_{\Pi\in \gS_{m_\ell}} \|W_A^\ell-\Pi W_B^\ell \Pi_{\ell -1}^T\|_2^2\\
            &=\argmin_{\pi \in \gS_{m_\ell}} \frac{1}{m_\ell}\sum_{i=1}^{m_\ell}\|[W_{A}^\ell]_{i:}-[W_B^\ell \Pi^T_{\ell -1}]_{\pi_i:}\|_2^2
    \end{split}
\end{equation}
For each layer, the problem can be cast as finding a pairing of weights neurons $[W_A^\ell]_{i:}$ and $[W_B^\ell \Pi_{\ell-1}^T]_{\pi_i:}$ to minimize the sum of their Euclidean distances. It is known as the Monge problem is the optimal transport literature~\cite {peyre2019computational}. More precisely \Cref{eq:P_l} can be formulated as finding an optimal transport plan corresponding to the Wasserstein distance between the empirical measures of the rows of $W_A^\ell$ and $W_B^\ell \Pi_{\ell -1}^T$. We provide more details about this connection between \Cref{eq:P_l} and optimal transport in \Cref{validity_OT_viewpoint}.  In the following, the $p-$Wasserstein distance will be denoted $\gW_p(\cdot,\cdot)$ and defined with the underlying distance $\|\cdot\|_2$ unless expressed otherwise.

By controlling the cost in~\Cref{eq:P_l} at every layer, we show that the permuted s of networks $A$ and $B$ are approximately equal. Linearly interpolating both networks will therefore keep activations of all hidden layers unchanged except the last layer which acts as a linear function of the interpolation parameter $t\in [0,1]$.

\section{LMC FOR TWO-LAYER NNs IN THE MEAN FIELD REGIME}
\label{section:mean_field}
\looseness=-1
We will first study linear mode connectivity between a pair of two-layer neural networks independently trained with SGD for the same number of steps.
\subsection{Background on the Mean Field Regime}
\looseness=-1
We will use some notations from \cite{mei2019mean} and consider a two-layer neural network,
\begin{equation}
\label{network_form}
    \hat{f}_N(x;\theta)=\frac{1}{N}\sum_{i=1}^N\sigma_*(x;\theta_i)
\end{equation}
parametrized by $\theta_i=(a_i,w_i) \in \sR\times \sR^d$ and where $\sigma_*(x;\theta_i)=a_i\sigma(w_ix)$. The parameters evolve as to minimize the following regularized cost $R_N(\theta)=\E_{(x, y)\sim P}[(y-\hat{f}_N(x_;\theta))^2]+\lambda\|\theta\|_2^2$. Define noisy regularized stochastic gradient descent (or noiseless regularization-free when $\lambda=0,\tau=0$) with step size $s_k$, and i.i.d.\@ Gaussian noise $g^k\sim\gN(0,I_d)$:
\begin{align}
        &\theta_i^{k+1}=(1-2\lambda s_k)\theta_i^k \tag{SGD} \label{eq:sgd} \\
        &+2s_k(y_k-\hat{f}_N(x_k;\theta^k))\nabla_\theta\sigma_*(x_k;\theta_i^k)
        +\sqrt{\tfrac{2s_k\tau}{d}}g_i^k \notag
\end{align}

It will be useful to consider  $\rho_{N}^k\eqdef\tfrac{1}{N}\sum_{i=1}^N\delta_{\theta^k_i}$ the empirical distribution of the weights after $k$ SGD steps. Indeed some recent works \citep{chizat2018global,mei2018mean,mei2019mean} have shown that when setting the width $N$ to be large and the step size $s_k$ to be small, the empirical distribution of weights during training remains close to an empirical measure drawn from the solution of a partial differential equation (PDE) we explicit in \Cref{appendix:background_mean_field}. Especially, the parameters $\left\{\theta_i^k, i\in [N]\right\}$ evolve approximately independently.

\subsection{Proving LMC in the mean field setting}

Define respectively the alignment of a neuron function on the data and the correlation between two neurons:
\begin{equation*}
    \begin{aligned}
        &V(\theta_1) \eqdef av(w) \eqdef -\sE_P[y\sigma_*(x;\theta_1)]\\
           &U(\theta_1,\theta_2)\eqdef a_1a_2u(w_1,w_2)\eqdef\sE_P[\sigma_*(x;\theta_1)\sigma_*(x;\theta_2)]\,.
    \end{aligned}
\end{equation*} and we for $\epsilon>0$ fixed we note the step size
\begin{equation}
\label{eq:learning_rate}
    s_k=\epsilon\xi(k\epsilon)
\end{equation}
where $\xi$ is a positive scaling function. The underlying training time up to step $k_T$ is defined as $T\eqdef \sum_{k=1}^{k_T}s_k$. We now state the standard assumptions to work on the mean field regime~\citep{mei2019mean},
\begin{assumption}
\label{ass:noiseless_SGD} The function 
    $t\mapsto \xi(t)$ is bounded Lipschitz.
    The non-linearity $\sigma$ is bounded Lipschitz and the data distribution has a bounded support. 
    The functions $w \mapsto v(w)$ and $(w_1,w_2)\mapsto u(w_1,w_2)$ are  differentiable, with bounded
and Lipschitz continuous gradient.
The weights at initialization $\theta_i^0$ are i.i.d. with distribution $\rho_0$ which has bounded support.
\end{assumption}

\looseness=-1
\Cref{ass:noiseless_SGD} imposes that the step size is of order $\gO(\epsilon)$ and its variations are of order $\gO(\epsilon^2)$. Constant step size $\epsilon$ will work.
Bounded non-linearity include arctan and sigmoid but excludes ReLU. While it is a standard assumption in mean field theory (\citep{mei2018mean,mei2019mean}), we mention in \S\ref{appendix:mean_field} that this assumption can be relaxed by the weaker assumption that the non-linearity stays small on some big enough compact set.

The second assumption is technical and only used for studying noisy regularized SGD.
\begin{assumption}
\label{ass:noisy_SGD}
$V,U$ are four times continuously differentiable and $\nabla_1^ku(\theta_1,\theta_2)$ is uniformly bounded for $0\leq k\leq 4$.
\end{assumption}
The following theorem states that two wide enough two-layer neural networks trained independently with SGD exhibit, with high probability, a linear connection of the prediction modulo permutations for all data.

\begin{mdframed}[style=MyFrame2]
\begin{restatable}{theorem}{theoremmeanfield}
\label{theorem:mean_field}
    Consider two two-layer neural networks as in \Cref{network_form} trained with \eqref{eq:sgd} with the same initialization over the weights independently and for the same underlying time $T$. Suppose 
    \Cref{ass:noiseless_SGD,ass:noisy_SGD} to hold.
    Then $\forall \delta, \err, \exists N_{min}$ such that if $N\geq N_{min}, \exists \epsilon_{max}(N)$ such that if $\epsilon \leq \epsilon_{max}(N)$ in \Cref{eq:learning_rate}, then with probability at least $1-\delta$ over the training process, there exists a permutation of the second network's hidden layer such that for almost every $x \sim P$:
    \begin{equation*}
        \begin{split}
            &|t\hat{f}_N(x;\theta_{A})+(1-t)\hat{f}_N(x;\theta_{B})\\
            &-\hat{f}_N(x;t\theta_{A}+(1-t)\Tilde{\theta}_B)|\leq \err \,,  \quad \forall t \in [0,1]\,.
        \end{split}
    \end{equation*}
\end{restatable}
\end{mdframed}

\textbf{Remark.} \Cref{ass:noisy_SGD} is not used when studying noiseless regularization-free SGD ($\lambda=0,\,\tau=0$).

\begin{mdframed}[style=MyFrame2]
\begin{restatable}{corollary}{corollarymeanfield}
\label{corollary:mean_field}
    Under assumptions of $\Cref{theorem:mean_field}$, $\forall \delta, \err>0,\, \exists N'_{\min},\, \forall N \geq N'_{\min}, \, \exists \epsilon'_{\max}(N),\, \forall \epsilon \leq \epsilon'_{\max}(N)$ in \Cref{eq:learning_rate}, then with probability at least $1-\delta$ over the training process, there exists a permutation of the second network's hidden layer such that $\forall t \in [0,1]$:
    \begin{equation*}
        \begin{split}
             &\sE_P\big[\big(\hat{f}_N(x;t\theta_A+(1-t)\Tilde{\theta}_B)-y\big)^2\big]\leq \err\\
             &+\sE_P\big[t(\hat{f}_N(x;\theta_A)-y)^2+(1-t)(\hat{f}_N(x;\theta_B)-y\big)^2\big]
        \end{split}
    \end{equation*}
\end{restatable}
\end{mdframed}

\textbf{Discussion.} Two wide enough two-layer neural networks wide enough trained with SGD are therefore Linear Mode Connected with an upper bound on the error tolerance we explicit in \Cref{appendix:mean_field}. We have extensively used the independence between weights in the mean field regime to apply OT bounds on convergence rates of empirical measures. To go beyond the two-layer case, we will need to make such an assumption on the distribution of weights. Note that this is true at initialization and after training for two-layer networks. Studying the independence of weights in the multi-layer case is a natural avenue for future work, already studied in \citet{nguyen2020rigorous}.

\section{GENERAL STRATEGY FOR MULTI-LAYER NETWORKS}
\label{section:general_study}
We now build the foundations to study the case of multi-layer neural networks (see~\Cref{eq:network}).

We first write one formal property expressing the existence of permutations of neurons of network $B$ up to layer $\ell$ such that the activations of network $A$, network $B$ permuted and the mean network $M_t$ are close up to layer $\ell$. This property is trivially satisfied at the input layer. We then show that under two formal assumptions on the weights matrices of networks $A$ and $B$, this property still hold at layer $\ell+1$.

\subsection{Formal Property at layer $\ell$}

Let $\epsilon > 0, m_\ell \geq \Tilde{m}_l$ and $m_{\ell+1} \geq \Tilde{m}_{\ell+1}$. Assume $\frac{m_\ell}{\Tilde{m}_\ell}, \frac{m_{\ell+1}}{\Tilde{m}_{\ell+1}}\in \sN$ to simplify technical details but this hypothesis can easily be removed.

\begin{property}\label{new_property}
There exists two constants $\ubar{E}_\ell,E_\ell$ such that given weight matrices up to layer $\ell$, $W_{A,B}^{1,\cdots,\ell},W_B^{1,\cdots,\ell}$ one can find $\ell$ permutations $\Pi_1,\cdots,\Pi_\ell$  of the neurons in the hidden layers $1$ to $\ell$ of network $B$, an equi-partition $\gI^\ell=\{I_1^\ell,\ldots,I^\ell_{\Tilde{m}_\ell}\}$, and a map $\ubar{\phi}^\ell(x)\in \sR^{n}$ such that $\forall k \in [\Tilde{m}_\ell] \,,\,\forall  i,j \in I_k^\ell, \ubar{\phi}^\ell_i(x)=\ubar{\phi}^\ell_j(x)$ such that:
\begin{align*}
    &\E_{P,Q}\|\ubar{\phi}^\ell(x)\|_2^2\leq \ubar{E}_\ell m_\ell\\
    &\E_{P,Q}\|\phi^\ell_A(x)-\ubar{\phi}^\ell(x)\|_2^2\leq E_\ell m_\ell\\
    &\E_{P,Q}\|\Tilde{\phi}^\ell_B(x)-\ubar{\phi}^\ell(x)\|_2^2\leq E_\ell m_\ell\\
   &\E_{P,Q}\|\phi^\ell_{M_t}(x)-\ubar{\phi}^\ell(x)\|_2^2\leq E_\ell m_\ell\,,\quad  \forall t \in [0,1], 
\end{align*}
\end{property}
This property not only requires proximity between activations $\phi_A^\ell(x),\Tilde{\phi}_B^\ell(x)$ at layer $\ell$ but requires the existence of a vector $\ubar\phi^\ell(x)$ whose coefficients in the same groups of the partition $\gI^\ell$ are equal, and therefore lives in a $\Tilde{m}_\ell$. It bounds the size of the function space available at layer $\ell$ and hence allows to use an effective width $\Tilde{m}_\ell$ independent of the real width $m_\ell$, which can be much larger. It is crucial in order to show LMC for neural networks of constant width across layers.
The introduction of such a map $\ubar{\phi}^\ell(x)$ is non trivial and is an important contribution since it allows to extend results of \cite{entezari2021role} beyond two layers.

\subsection{Assumptions on the weight distribution }
We now make an assumption on the empirical distribution of the weights $\hat{\mu}_{A,\ell+1}$ at layer $\ell+1$ of $W_A^{\ell+1}$.

\begin{assumption}\label{new_assumption:1}
    There exists an integer $\Tilde{m}_{\ell+1}$ such that for all equi-partiton $\gI^{\ell}$ of $[m_{\ell}]$ with $\Tilde{m}_{\ell}$ sub-sets, there exists a random empirical measure $\hat{\mu}_{\Tilde{m}_{\ell+1}}$ independent of $A$ and $B$ composed of $\Tilde{m}_{\ell+1}$ vectors in $\sR^{m_{\ell}}$, such that $\E_Q[\gW_2^2(\hat{\mu}^{\gI^{\ell}}_{A,\ell+1}, \hat{\mu}^{\gI^{\ell}}_{\Tilde{m}_{\ell+1}})]\leq C_1$. 

\end{assumption}

This assumption requires that the empirical distribution with $m_{\ell+1}$ points of the neurons' weights of network $A$ at layer $\ell+1$ can be approximated by an empirical measure with a smaller $\Tilde{m}_{\ell+1}$ number of points. Note that it implies proximity in Wasserstein distance between $\hat{\mu}^{\gI^\ell}_A$ and $\hat{\mu}^{\gI^\ell}_B$ by a triangular inequality.

We finally assume some central limit behavior when summing the errors made for each neuron of layer $\ell$.

\begin{assumption}\label{new_assumption:2}
    There exists a constant $C_2$ such that $\forall X\in \sR^{m_l}$ we have:

    \begin{equation*}
    \begin{split}
        \max\bigl\{\E_Q[\|W_A^{\ell+1}X\|_2^2], \E_Q[&\|W_{\Tilde{m}_{l+1}}X\|_2^2]\bigr\}\\
        &\leq C_2\frac{m_{\ell+1}}{m_\ell}\|X\|_2^2
    \end{split}
    \end{equation*}
\end{assumption}

Finally, we consider the following assumption on the non-linearity, verified for example by pointwise ReLU.

\begin{assumption}
\label{non_linearity}
    $\sigma$ is pointwise, $1$-Lipschitz, $\sigma(0)=0$.
\end{assumption}

\subsection{Propagating \Cref{new_property} to layer $\ell+1$}

We state now how \Cref{new_property} propagates throughout the layers using~\Cref{non_linearity,new_assumption:1,new_assumption:2} with new parameters $\ubar{E}_{\ell+1},E_{\ell+1}$. We give a proof in~\Cref{appendix:proof_newbiglemma}.

\begin{mdframed}[style=MyFrame2]
\begin{restatable}{lemma}{newbiglemma}
\label{newbiglemma}
   Let $\ell\in \{0,\cdots,L-1\}$ and suppose \Cref{new_property} to hold at layer $\ell$ and \Cref{non_linearity,new_assumption:1,new_assumption:2} to hold, then \Cref{new_property} still holds at the next layer with $\Tilde{m}_{\ell+1}$ given in \Cref{new_assumption:1} and 
    \begin{equation*}
    \begin{split}
         &\ubar{E}_{\ell+1}=C_2\ubar{E}_{\ell}\\
         &E_{\ell+1}=2C_2E_\ell+2C_1\Tilde{m}_\ell \ubar{E}_\ell
    \end{split}
    \end{equation*}
\end{restatable}
\end{mdframed}

\section{LMC FOR RANDOM MULTI-LAYER NNs}

\label{section:given_distribution}

We will make the following assumption on the empirical distribution of neurons weights $\hat{\mu}_{A,\ell}, \hat{\mu}_{B,\ell}$ of $W_A^{\ell},W_B^{\ell}$ at layer $\ell$.

\begin{assumption}[Independence of neurons weights]
\label{new_assumption:0}
\looseness=-1
    $\hat{\mu}_{A,\ell}, \hat{\mu}_{B,\ell}$ correspond to two i.i.d drawings of vectors with distribution $\mu_{\ell}$ i.e., $\hat{\mu}_{A,\ell}, \hat{\mu}_{B,\ell}$ have the law of 
$\frac{1}{m_{\ell}}\sum_{i=1}^{m_{\ell}}\delta_{x_i}$ where $x_i\sim \mu_{\ell}$ i.i.d.
\end{assumption}

\Cref{new_assumption:0} is verified for example at initialization but more generally when weights do not depend too much one of each other. This case still holds for wide two-layer neural networks trained with SGD and is at the heart of the proof of~\Cref{theorem:mean_field}.

\subsection{Showing LMC for multilayer MLPs under Gaussian distribution}
\label{section:given_distribution_normal}

We first examine the case $\mu_{\ell}=\gN\left(0,\tfrac{I_{m_{\ell-1}}}{m_{\ell-1}}\right)$. We moreover assume that the input data distribution has bounded second moment: $\E_P[\|x\|_2^2]\leq m_0$.

Our strategy detailed in \Cref{appendix:new_LMC_normal} consists in showing that wide enough such networks will satisfy \Cref{new_assumption:1,new_assumption:2} with well controlled constants $C_1,C_2$. We can then apply \Cref{newbiglemma} successively $L$ times to get the following lemma:

\begin{mdframed}[style=MyFrame2]
\begin{restatable}{lemma}{normalnewbiglemmasuccessive}
\label{normal:newbiglemmasuccessive}
    Under normal initialization of the weights, given $\epsilon > 0$, if $m_0\geq 5$, there exists minimal widths $\Tilde{m}_1, \ldots, \Tilde{m}_L$ such that if $m_1\geq \Tilde{m}_1,\ldots,m_L\geq \Tilde{m}_L$, \Cref{new_property} is verified at the last hidden layer $L$ for $\ubar{E}_L=1,E_L=\epsilon^2$.
    Moreover, $\forall \ell \in [L], \exists T_\ell$ which does only depend on $L,\ell$ such that one can define recursively $\Tilde{m}_\ell$ as $\Tilde{m}_0=m_0$ and 
    \begin{equation*}
        \Tilde{m}_{\ell}=\Tilde{\gO}\left(\frac{T_\ell}{\epsilon}\right)^{\Tilde{m}_{\ell-1}}
    \end{equation*}
\end{restatable}
\end{mdframed}

\looseness=-1
\textbf{Discussion.} The hypothesis $m_0\geq5$ is technical and could be relaxed at the price of slightly changing the bound on $\Tilde{m}_1$. \Cref{normal:newbiglemmasuccessive} shows that given two random networks whose widths $m_\ell$ is larger than $\Tilde{m}_\ell$, we can permute neurons of the second one such that their activations at layer $\ell$ are both close to the one of the networks on a linear path in parameter's space.

As $\epsilon$ goes to $0$, the width of the layer $\ell+1$ must scale at least as $\left(\frac{1}{\epsilon}\right)^{\Tilde{m}_{\ell-1}}$. This is a fundamental bound due to the convergence rate in Wasserstein distance of empirical measures. It imposes a recursive exponential growth in the width needed with respect to depth. This condition appears excessive as compared to the typical width of neural networks used in practice.
 We highlight here that \citet{ainsworth2022git} empirically demonstrates that networks at initialization do not exhibit LMC and that the loss barrier is erased only after a sufficient number of SGD steps.

\subsection{Showing Linear Mode Connectivity}

We make the following assumption on the loss function to show LMC from \Cref{normal:newbiglemmasuccessive}.

\begin{assumption}
    \label{loss_convex}
    $\forall y \in \sR^{m_{L+1}}$, the loss $\gL(\cdot,y)$ is convex and $1$-Lipschitz.
\end{assumption}

We finally prove the following bound on the loss of the mean network $M_t$ in \Cref{appendix:new_theorem_normal_LMC}:
\begin{mdframed}[style=MyFrame2]
    \begin{restatable}{theorem}{newnormaltheoremLMC}
    \label{normal:new_theorem_LMC}
            Under normal initialization of the weights, for $m_1\geq \Tilde{m}_1,\cdots,m_L\geq \Tilde{m}_L$ as defined in \Cref{normal:newbiglemmasuccessive}, $m_0\geq 5$, and under \Cref{loss_convex} we know that $\forall t \in [0,1]$, with $Q$-probability at least $1-\delta_{Q}$, there exists permutations of hidden layers $1,\ldots,L$ of network $B$ that are independent of $t$, such that: 
            \begin{multline*}
                \E_P\left[\gL\left(\hat{f}_{M_t}(x),y\right)\right]\leq t\E_P\left[\gL\left(\hat{f}_A(x),y\right)\right]+\\
                (1-t)\E_P\left[\gL\left(\hat{f}_B(x),y\right)\right]+\tfrac{4\sqrt{m_{L+1}}}{\delta_Q^2}\epsilon
            \end{multline*}
    \end{restatable}
\end{mdframed}

\textbf{Discussion.} The minimal width at layer $\ell$ needed for \Cref{normal:new_theorem_LMC} is recursively $\Tilde{m}_l\sim \epsilon^{-\Tilde{m}_{l-1}}$. Applied to randomly initialized two-layer networks, we need a hidden layer's dimension of $\epsilon^{-m_0}$ as opposed to \citet{entezari2021role} which prove a bound of $\epsilon^{-(2m_0+4)}$.

\subsection{Tightness of the bound dependency with respect to the error tolerance}

We discuss here the tightness of the minimal width $\Tilde{m}_\ell$ we require in \Cref{normal:newbiglemmasuccessive} with respect to the error tolerance $\epsilon$. The recursive exponential growth of the width in the form $\Tilde{m}_\ell \sim \left(\frac{1}{\epsilon}\right)^{\Tilde{m}_{\ell-1}}$ is a consequence of the convergence rate of Wasserstein distance of empirical measures in dimension $\Tilde{m}_{\ell-1}$ at the rate $\nicefrac{1}{\Tilde{m}_{\ell-1}}$. \Cref{theorem:lower_bound} provides a corresponding lower bound which shows that this recursive exponential growth is tight at the precise rate $\left(\tfrac{1}{\epsilon}\right)^{\Tilde{m}_{\ell-1}}$ (just take $n=\Tilde{m}_{\ell-1}, \, m=\Tilde{m}_\ell, \, \mu=\mu_\ell, \, x=\phi_A^{\ell-1}(x), W_{A,B}=W_{A,B}^{\ell}$). A proof is given in \Cref{appendix:lower_bound}.

\begin{mdframed}[style=MyFrame2]
    \begin{restatable}{theorem}{theoremlowerbound}
    \label{theorem:lower_bound}
    Let $n\geq 1, x\sim P\in \gP_1(\sR^n)$ and $\mu \in \gP(\sR^n)$ such that $\frac{d\mu}{d\Leb}\leq F_1$. Suppose $\Sigma=\E[xx^T]$ is full rank $n$. Let $m\geq 1$ and $W_A,W_B\in \gM_{m, n}(\sR)$ whose rows are drawn i.i.d. from $\mu$.
    Then, there exists $F_0$ such that 
    \begin{equation*}
        \E_{W_A,W_B}[\min_{\Pi\in \gS_m}\E_{P}\|(W_A-\Pi W_B)x\|_2^2] \geq F_0 \left(\frac{1}{m}\right)^{2/n}
    \end{equation*}
    \end{restatable}
\end{mdframed}

\textbf{Remark 1.} Using an effective width $\Tilde{m}_{\ell-1}$ smaller and independent of the real width $m_{\ell-1}$ allows to show LMC for networks of constant hidden width $m_1=m_2=\ldots=m_L$ as soon as they verify $m_1\geq \Tilde{m}_1, \ldots, m_L\geq \Tilde{m}_L$ where $\Tilde{m}_1,\ldots,\Tilde{m}_L$ are defined in~\Cref{normal:newbiglemmasuccessive}. Without this trick, we need a recursive exponential growth of the real width $m_\ell \sim \left(\tfrac{1}{\epsilon}\right)^{m_{\ell-1}}$. 

\textbf{Remark 2.} Motivated by the fact that feature learning may concentrate the weight distribution on low dimensional sub-space, we could extend our proofs to the case where the underlying weight distribution has a support with smaller dimension to get recursive bounds no longer at rate $\Tilde{m}_{\ell-1}$ but at a smaller one.
Note this is unlikely to happen as we expect the matrix of weight vectors of a given layer to be full rank.
Therefore, we study in the next section the case when this matrix is approximately low rank, or equivalently when the weight distribution is concentrated around a low dimensional approximated support.

\subsection{Approximately low dimensional supported measures}
\label{section:low_dim}

For the sake of clarity, assume from now on that the layer $\ell-1$ of network $A$ has been permuted such that for $\gI^{\ell-1}=\{I_1^{\ell-1},\ldots,I_{\Tilde{m}_{\ell-1}}^{\ell-1}\}$ (given in \Cref{new_property}) we have $I_1^{\ell-1}=\{1, \ldots, p_{\ell-1}\},\,\ldots,\, I_{\Tilde{m}_{\ell-1}}^{\ell-1}=\{m_{\ell-1}-p_{\ell-1}+1, \ldots, m_{\ell-1}\}$ with $p_{\ell-1}=\nicefrac{m_{\ell-1}}{\Tilde{m}_{\ell-1}}\}$. This assumption is mild since we can always consider a permuted version of network $A$ without changing the problem.

Motivated by the discussion in \Cref{appendix:approx_low_dim_cov_matrix} we consider the model where the weights at layer $\ell$ are initialized i.i.d. multivariate Gaussian $\mu_\ell=\gN(0,\Sigma^{\ell-1})$ with
\begin{equation*}
    \Sigma^{\ell-1}\eqdef \Diag\left(\lambda^\ell_1I_{p_{\ell-1}}, \lambda_2^\ell I_{p_{\ell-1}}, \ldots,\lambda^\ell_{\Tilde{m}_{\ell-1}}I_{p_{\ell-1}}\right)
\end{equation*}
with $\tfrac{1}{m_{\ell-1}}\tfrac{\Tilde{m}_{\ell-1}}{k_{\ell-1}}\geq \lambda_1^\ell \geq \lambda_2^\ell \geq \ldots \geq \lambda^\ell_{\Tilde{m}_{\ell-1}}$ with $k_{\ell-1}\leq \Tilde{m}_{\ell-1}$ an approximate dimension of the support of the underlying weights distribution. Note that to balance the low dimensionality of the weights distribution, we have replaced the upper-bound on the eigenvalues $\tfrac{1}{\Tilde{m}_{\ell-1}}$ by the greater value $\tfrac{1}{m_{\ell-1}}\tfrac{\Tilde{m}_{\ell-1}}{k_{\ell-1}}$ to avoid vanishing activations when $\ell$ grows which would have made our result vacuous.

The following assumption states that the weights distribution $\mu_{\ell}^{\gI^{\ell-1}}$ at layer $\ell$ considered in $\gP_1(\sR^{\Tilde{m}_{\ell-1}})$ (with the operation explicited in \Cref{section:preliminaries}) is approximately of dimension $k_{\ell-1}= e\Tilde{m}_{\ell-1}$. The approximation becomes more correct as $\eta \rightarrow 0$.

\begin{assumption}[Approximately low-dimensionality]
    \label{approx_lowdim}
    $\exists \eta,e \in (0,1), \forall \ell \in [L], \, \frac{\sqrt{\sum_{j=k_{\ell-1}}^{\Tilde{m}_{\ell-1}}\lambda^\ell_j}}{4\sqrt{\sum_{j=1}^{k_{\ell-1}}\lambda^\ell_j}}\leq \eta, \, \tfrac{k_{\ell-1}}{\Tilde{m}_{\ell-1}}= e$
\end{assumption}

\begin{mdframed}[style=MyFrame2]
    \begin{restatable}{theorem}{approxlowdimtheoremLMC}
        \label{low_dim:newtheoremLMC}
    Under \Cref{approx_lowdim,loss_convex}, given $\epsilon > 0$, if $em_0\geq 5$ there exists minimal widths $\Tilde{m}_1, \ldots, \Tilde{m}_L$ such that if $\eta^{-k_0}\geq m_1\geq \Tilde{m}_1,\ldots,\eta^{-k_{L-1}}\geq m_L\geq \Tilde{m}_L$, \Cref{new_property} is verified at the last hidden layer $L$ for $\ubar{E}_L=1,E_L=\epsilon^2$.
    Moreover, $\forall \ell \in [L], \,\exists T'_\ell$ which does only depend on $L,e,\ell,$ such that one can define recursively $\Tilde{m}_\ell$ as
    \begin{equation*}
\Tilde{m}_{\ell}=\Tilde{\gO}\left(\frac{T'_\ell}{\epsilon}\right)^{k_{\ell-1}}=\Tilde{\gO}\left(\frac{T'_\ell}{\epsilon}\right)^{e\Tilde{m}_{\ell-1}}
    \end{equation*}
    where $\Tilde{m}_0=m_0$. Moreover there exists permutations of hidden layers $1,\ldots,L$ of network $B$ s.t. $\forall t \in [0,1]$, with $Q$-probability at least $1-\delta_{Q}$: 
    \begin{multline*}
\E_P\left[\gL\left(\hat{f}_{M_t}(x),y\right)\right]\leq t\E_P\left[\gL\left(\hat{f}_A(x),y\right)\right]\\
+(1-t)\E_P\left[\gL\left(\hat{f}_B(x),y\right)\right]+\tfrac{4\sqrt{m_{L+1}}}{\sqrt{e}\delta_Q^2}\epsilon
    \end{multline*}
    \end{restatable}
\end{mdframed}

\textbf{Discussion.} We give a proof in \Cref{appendix:proof_low_dim_LMC}. For $\eta$ small enough, the distribution of weights is approximately lower dimensional. It yields faster convergence rates until $m$ becomes exponentially big in $\eta$. This prevents the previous recursive exponential growth of width with respect to depth, though asymptotically, we recover the same rates as in \Cref{normal:new_theorem_LMC}. 
The smaller $e$, the lower dimensional are the distributions, and the less the width needs to grow when $\epsilon\rightarrow 0$. The problem in that model is that the constant $T'_i$ explodes if $e\rightarrow 0$, which prevents using a model with fixed $k_\ell$ across the layers for the weight distribution. We want to highlight here that the proof can be extended to such a case, but we need to assume that the constant $C_2$ is bounded and not depending on $e$ across the layers in \Cref{newbiglemma} (recall that with our proof, we had $C_2=\tfrac{1}{e}$). This assumption seems coherent because the average activations don't explode across layers in the model.
Assuming this, the bound we obtain for $\Tilde{m}_\ell$ in \Cref{low_dim:newtheoremLMC} is completely independent on $\Tilde{m}_{\ell-1}$, and there is no recursive exponential growth in the width needed across the layers. We give a more explicit discussion in \Cref{apprendix:discuss_low_dim}.

\subsection{LMC for sub-Gaussian distributions}

Still under the setting of~\Cref{new_assumption:0} assume that the underlying distribution $\mu_\ell$ verifies for each layer $\ell\in [L+1]$: if $X\sim \mu_\ell$ then, $\forall j\neq k \in [m_{l-1}], X_j \amalg X_k$. Moreover $\forall i \in [\Tilde{m}_{\ell-1}], \forall j,k \in I^{\ell-1}_i$,
\begin{equation*}
     \E[X_j^2]=\E[X_k^2]=\lambda_i^\ell
\end{equation*}
Finally suppose the variables are sub-Gaussian i.e., $\exists K > 0, \forall i \in [\Tilde{m}_{\ell-1}], \forall j\in I^{\ell-1}_i,\, \forall c > 0$,
\begin{equation*}
    \sP(|X_j|\geq c)\leq 2\exp(-\frac{c^2}{K\lambda_i^\ell})
\end{equation*}
We explain in \Cref{appendix:new_subgaussian} why both \Cref{normal:new_theorem_LMC} (in the case $\lambda_1^\ell=\ldots=\lambda^\ell_{\Tilde{m}_{l-1}}=\nicefrac{1}{m_{\ell-1}}$) and \Cref{low_dim:newtheoremLMC} hold with mild modifications in the constants.

It extends our previous result considerably to LMC for any large enough networks whose weights are i.i.d. and whose underlying distribution has a sub-Gaussian tail (for example uniform distribution).

\subsection{Link with dropout stability}

In \Cref{appendix:dropout_stability}, we build a first step towards unifying our study with the dropout stability viewpoint \citep{kuditipudi2019explaining,shevchenko2020landscape} by showing in a simplified setting how networks become dropout stable in the same asymptotics on the width as the one needed in our~\Cref{normal:new_theorem_LMC}.
\section{EXPERIMENTS}

\label{section:experiments}
Our previous study shows the influence of the dimension of the underlying weight distribution on LMC effectiveness. Based on this insight we develop a new weight matching method at the crossroads between previous naive weight matching (WM) and activation matching (AM) methods \citep{ainsworth2022git}. Given $n$ training points $x_i, \, i\in [n]$, denote $Z_A^\ell \in \gM_{m_{\ell},n}(\sR)$ (respectively $Z_B^\ell$) the activations $\phi^\ell_A(x_i)$ for the $n$ data points $x_i$. Further denote $\Sigma_A^\ell \eqdef \tfrac{1}{n}Z_A^\ell[Z_A^\ell]^T\approx \E_P\left[\phi_A^\ell(x)[\phi_A^{\ell}(x)]^T\right]$. We aim at finding for each layer $\ell$ the optimal permutation $\Pi$ minimizing the cost (respectively for naive WM, our new WM method and AM):
\begin{align}
        &\min_{\Pi \in S_{m_\ell}}  \left\|W_A^\ell-\Pi W_B^\ell \Pi_{\ell-1}^T\right\|_2^2\,, \tag{Naive WM}\\ 
        &\min_{\Pi \in S_{m_\ell}}\left\|W_A^\ell-\Pi W_B^\ell \Pi_{\ell-1}^T\right\|^2_{2,\Sigma_A^{\ell-1}}\,, \tag{WM (ours)}\\
        &\min_{\Pi \in S_{m_\ell}} \left\|Z_A^\ell-\Pi Z_B^\ell \right\|_2^2 \,, \tag{AM}        
\end{align}
where $\|\cdot\|_{2,\Sigma_A^{\ell-1}}$ is the norm\footnote{Semi-norm in full generality (if $\Sigma_A^{\ell-1}$ is not full rank)} induced by the scalar product $(X,Y)\mapsto \tr(X\Sigma_A^{\ell-1} Y^T)$. We both theoretically support the gain of our method in \Cref{theorem:gain} and empirically verify that this method constantly and substantially outperforms naive Weight Matching across different learning rates when training with SGD.

\looseness=-1
We train a three hidden layer MLP of width $512$ on MNIST with learning rates varying between $10^{-4}$ and $10^{-1}$ across $4$ runs. We plot on \Cref{fig:approx_dim_final} the approximate dimension of the considered covariance matrix for each matching method: $W_A^\ell[W_A^\ell]^T$ for WM (naive),  $W_A^\ell\Sigma_A^{\ell-1}[W_A^\ell]^T$ for WM (ours) and $\Sigma_A^{\ell}$ for AM (see~\S\ref{appendix:explanation_method}). Our code is available at \url{https://github.com/damienferbach/OT_LMC/tree/main}.

\begin{figure}[H]
     \centering
     \begin{subfigure}[b]{0.48\textwidth}
         \includegraphics[width=\linewidth]{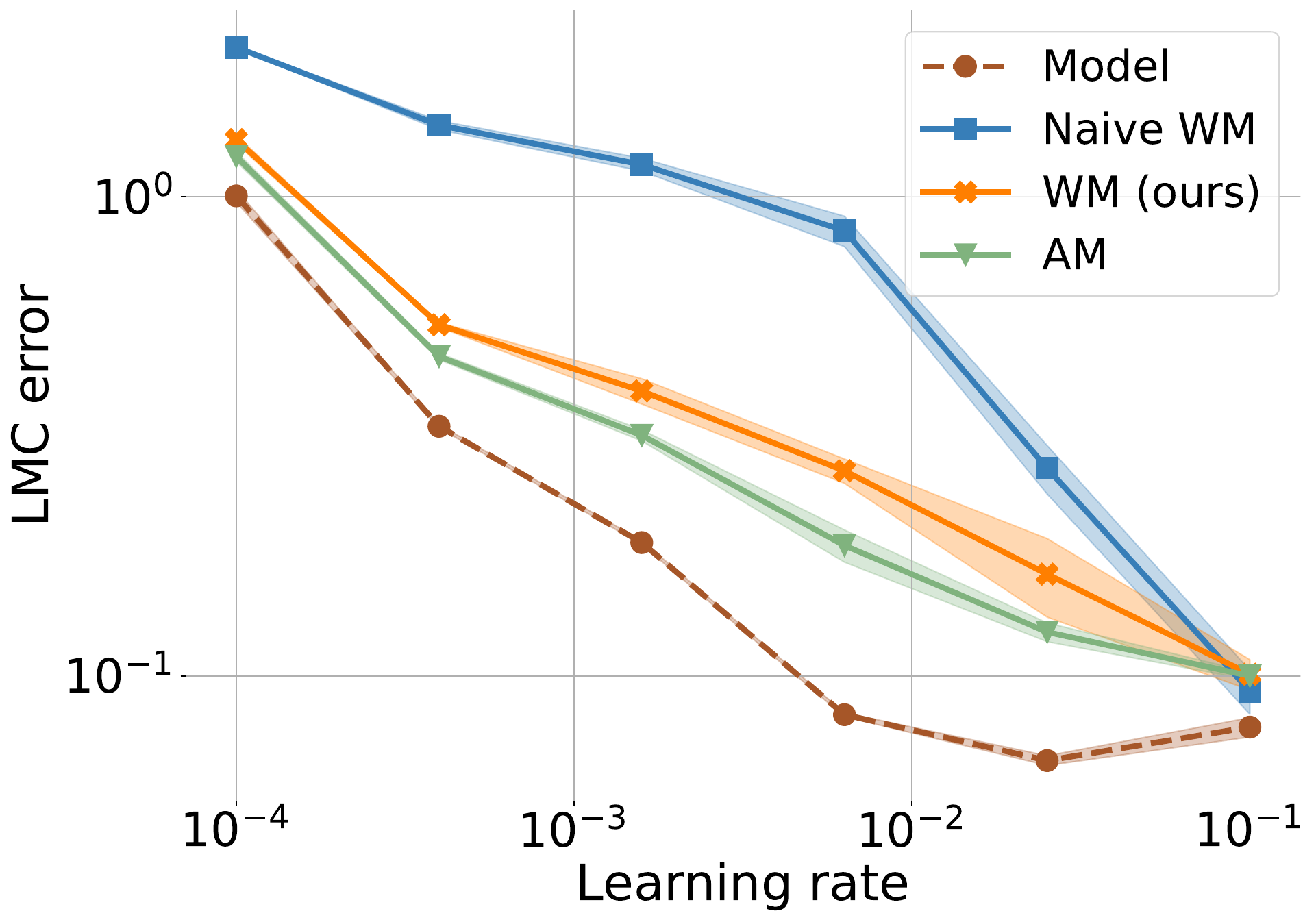}
         \caption{Mean test loss of the trained networks $A$ and $B$ and error barrier on the linear path $M_t, \, t\in [0,1]$ across different learning rate values for each matching problem.}
     \end{subfigure}
     \hfill
     \begin{subfigure}[b]{0.48\textwidth}
              \includegraphics[width=\linewidth]{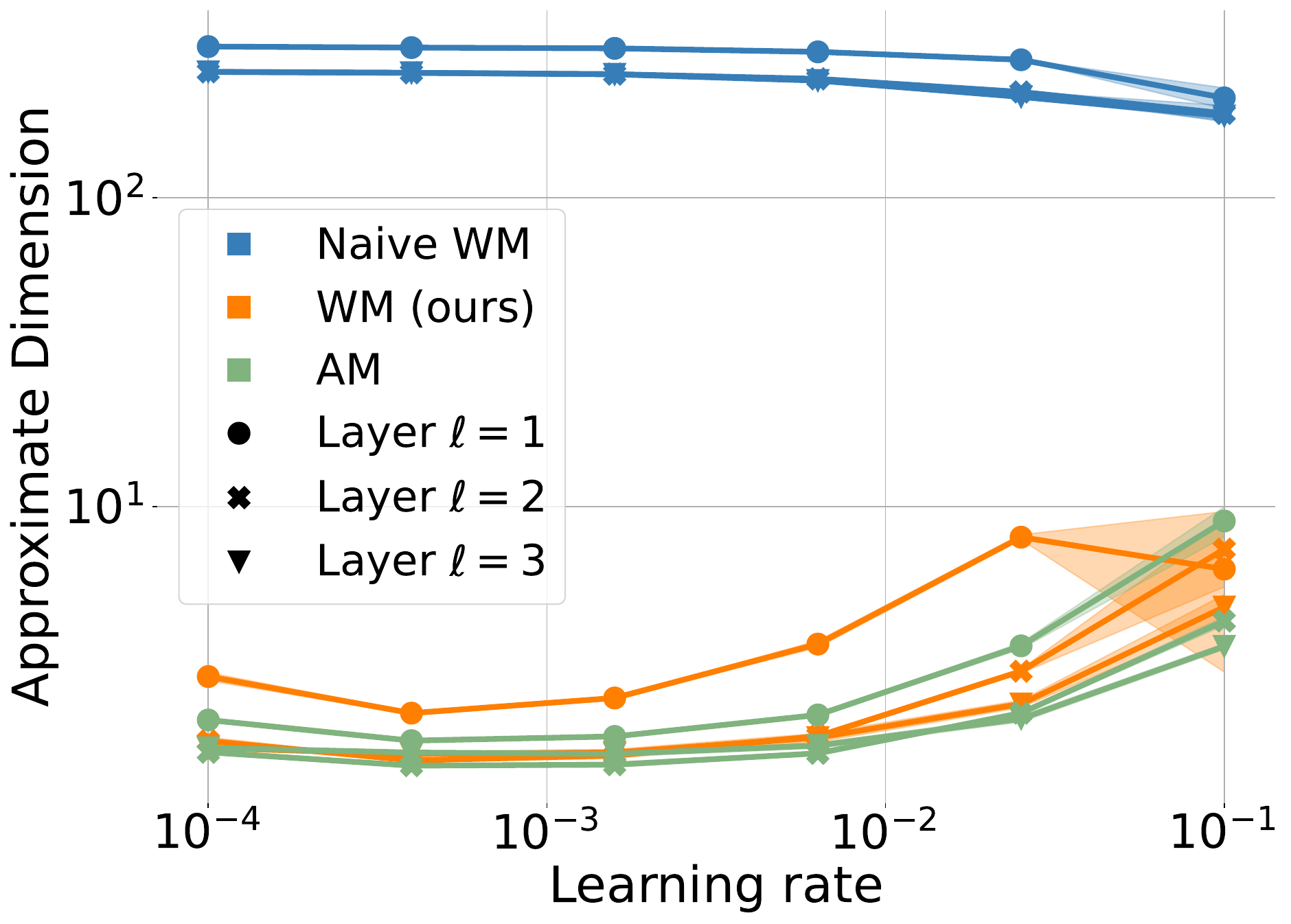}
     \caption{Approximate dimension $\Dim(S) \eqdef \nicefrac{\tr(S)^2}{\tr(S^2)}$ of the matrices considered in the matching problems at each layer.
     \newline ~
     }
     \label{fig:approx_dim_final}
     \end{subfigure}
     \caption{\small
     Statistics of the average network $M$ over the linear path between networks $A$ and $B$ using respectively weight matching (blue), weight matching using covariance of activations and activations (green), and activation matching (orange) \vspace{-.4cm}}
\label{fig:LMC_error}
\end{figure}

We see on \Cref{fig:LMC_error} the detrimental effect of high approximate dimension on LMC effectiveness, therefore validating our theoretical approach. Note that for a learning rate of $10^{-1}$ the correlation is less clear but a trend is visible on decreasing dimension for naive WM as it performs better (and increasing dimension for AM and our WM method as it performs comparatively less well). An alternative would be to use a proxy taking the diameter of the distributions into account (and not only the dimension of their support).
Finally, experiments on Adam lead to less clear results that we did not report as more experimental investigation is needed. In particular, understanding the impact of the optimizer on the independence of weights during training is crucial, as it is a central assumption in our study.

\section{DISCUSSION}

\looseness=-1
Optimal transport serves as a good framework to study linear mode connectivity of neural networks. This paper uses convergence rates of empirical measures in Wasserstein distance to upper bound the test error of the linear combination of two networks in weight space modulo permutation symmetries. Our main assumption is the independence of all neuron's weight vectors inside a given layer. This assumption is trivially true at initialization but remains valid for wide two-layers networks trained with SGD. We experimentally demonstrate the correlation between the dimension of the underlying weight distribution with LMC effectiveness and design a new weight matching method that significantly outperforms existing ones.
A natural direction for future work is to focus on the behaviour of the weights distribution inside each layer of DNNs and their independence. Moreover, extending our results to only assuming approximate independence of weights is a natural direction as it seems a more realistic setting.

\subsubsection*{Acknowledgements}

The work of B. Goujaud and A. Dieuleveut is partially supported by ANR-19-CHIA-0002-01/chaire SCAI, and Hi!Paris.
This work was partly funded by the French government under management of Agence Nationale de la Recherche as part of the ``Investissements d’avenir'' program, reference ANR-19-P3IA-0001 (PRAIRIE 3IA Institute). D. Ferbach acknowledges a stipend from ENS Paris as fonctionnaire stagiaire. This work was partly done during an internship of D. Ferbach at Ecole Polytechnique. 

\bibliography{bibliography}

\begin{thebibliography}{48}
\providecommand{\natexlab}[1]{#1}
\providecommand{\url}[1]{\texttt{#1}}
\expandafter\ifx\csname urlstyle\endcsname\relax
  \providecommand{\doi}[1]{doi: #1}\else
  \providecommand{\doi}{doi: \begingroup \urlstyle{rm}\Url}\fi

\bibitem[Adilova et~al.(2023)Adilova, Fischer, and Jaggi]{adilova2023layerwise}
L.~Adilova, A.~Fischer, and M.~Jaggi.
\newblock Layerwise linear mode connectivity.
\newblock \emph{arXiv preprint arXiv:2307.06966}, 2023.

\bibitem[Ainsworth et~al.(2022)Ainsworth, Hayase, and Srinivasa]{ainsworth2022git}
S.~K. Ainsworth, J.~Hayase, and S.~Srinivasa.
\newblock Git re-basin: Merging models modulo permutation symmetries.
\newblock \emph{arXiv preprint arXiv:2209.04836}, 2022.

\bibitem[Akash et~al.(2022)Akash, Li, and Trillos]{akash2022wasserstein}
A.~K. Akash, S.~Li, and N.~G. Trillos.
\newblock Wasserstein barycenter-based model fusion and linear mode connectivity of neural networks.
\newblock \emph{arXiv preprint arXiv:2210.06671}, 2022.

\bibitem[Chizat and Bach(2018)]{chizat2018global}
L.~Chizat and F.~Bach.
\newblock On the global convergence of gradient descent for over-parameterized models using optimal transport.
\newblock \emph{Advances in neural information processing systems}, 31, 2018.

\bibitem[Draxler et~al.(2018)Draxler, Veschgini, Salmhofer, and Hamprecht]{pmlr-v80-draxler18a}
F.~Draxler, K.~Veschgini, M.~Salmhofer, and F.~Hamprecht.
\newblock Essentially no barriers in neural network energy landscape.
\newblock In J.~Dy and A.~Krause, editors, \emph{Proceedings of the 35th International Conference on Machine Learning}, volume~80 of \emph{Proceedings of Machine Learning Research}, pages 1309--1318. PMLR, 10--15 Jul 2018.
\newblock URL \url{https://proceedings.mlr.press/v80/draxler18a.html}.

\bibitem[Entezari et~al.(2021)Entezari, Sedghi, Saukh, and Neyshabur]{entezari2021role}
R.~Entezari, H.~Sedghi, O.~Saukh, and B.~Neyshabur.
\newblock The role of permutation invariance in linear mode connectivity of neural networks.
\newblock \emph{arXiv preprint arXiv:2110.06296}, 2021.

\bibitem[Ferbach et~al.(2022)Ferbach, Tsirigotis, Gidel, and Bose]{ferbach2022general}
D.~Ferbach, C.~Tsirigotis, G.~Gidel, and A.~Bose.
\newblock A general framework for proving the equivariant strong lottery ticket hypothesis.
\newblock \emph{arXiv preprint arXiv:2206.04270}, 2022.

\bibitem[Frankle and Carbin(2018)]{frankle2018lottery}
J.~Frankle and M.~Carbin.
\newblock The lottery ticket hypothesis: Finding sparse, trainable neural networks.
\newblock \emph{arXiv preprint arXiv:1803.03635}, 2018.

\bibitem[Frankle et~al.(2020)Frankle, Dziugaite, Roy, and Carbin]{frankle2020linear}
J.~Frankle, G.~K. Dziugaite, D.~Roy, and M.~Carbin.
\newblock Linear mode connectivity and the lottery ticket hypothesis.
\newblock In \emph{International Conference on Machine Learning}, pages 3259--3269. PMLR, 2020.

\bibitem[Freeman and Bruna(2016)]{freeman2016topology}
C.~D. Freeman and J.~Bruna.
\newblock Topology and geometry of half-rectified network optimization.
\newblock \emph{arXiv preprint arXiv:1611.01540}, 2016.

\bibitem[Garipov et~al.(2018)Garipov, Izmailov, Podoprikhin, Vetrov, and Wilson]{garipov2018loss}
T.~Garipov, P.~Izmailov, D.~Podoprikhin, D.~P. Vetrov, and A.~G. Wilson.
\newblock Loss surfaces, mode connectivity, and fast ensembling of dnns.
\newblock \emph{Advances in neural information processing systems}, 31, 2018.

\bibitem[Goodfellow et~al.(2014)Goodfellow, Vinyals, and Saxe]{goodfellow2014qualitatively}
I.~J. Goodfellow, O.~Vinyals, and A.~M. Saxe.
\newblock Qualitatively characterizing neural network optimization problems.
\newblock \emph{arXiv preprint arXiv:1412.6544}, 2014.

\bibitem[Gotmare et~al.(2018)Gotmare, Keskar, Xiong, and Socher]{gotmare2018using}
A.~Gotmare, N.~S. Keskar, C.~Xiong, and R.~Socher.
\newblock Using mode connectivity for loss landscape analysis.
\newblock \emph{arXiv preprint arXiv:1806.06977}, 2018.

\bibitem[Izmailov et~al.(2018)Izmailov, Podoprikhin, Garipov, Vetrov, and Wilson]{izmailov2018averaging}
P.~Izmailov, D.~Podoprikhin, T.~Garipov, D.~Vetrov, and A.~G. Wilson.
\newblock Averaging weights leads to wider optima and better generalization.
\newblock \emph{arXiv preprint arXiv:1803.05407}, 2018.

\bibitem[Jacot et~al.(2018)Jacot, Gabriel, and Hongler]{jacot2018neural}
A.~Jacot, F.~Gabriel, and C.~Hongler.
\newblock Neural tangent kernel: Convergence and generalization in neural networks.
\newblock \emph{Advances in neural information processing systems}, 31, 2018.

\bibitem[Juneja et~al.(2022)Juneja, Bansal, Cho, Sedoc, and Saphra]{juneja2022linear}
J.~Juneja, R.~Bansal, K.~Cho, J.~Sedoc, and N.~Saphra.
\newblock Linear connectivity reveals generalization strategies.
\newblock \emph{arXiv preprint arXiv:2205.12411}, 2022.

\bibitem[Keskar et~al.(2016)Keskar, Mudigere, Nocedal, Smelyanskiy, and Tang]{keskar2016large}
N.~S. Keskar, D.~Mudigere, J.~Nocedal, M.~Smelyanskiy, and P.~T.~P. Tang.
\newblock On large-batch training for deep learning: Generalization gap and sharp minima.
\newblock \emph{arXiv preprint arXiv:1609.04836}, 2016.

\bibitem[Kuditipudi et~al.(2019)Kuditipudi, Wang, Lee, Zhang, Li, Hu, Ge, and Arora]{kuditipudi2019explaining}
R.~Kuditipudi, X.~Wang, H.~Lee, Y.~Zhang, Z.~Li, W.~Hu, R.~Ge, and S.~Arora.
\newblock Explaining landscape connectivity of low-cost solutions for multilayer nets.
\newblock \emph{Advances in neural information processing systems}, 32, 2019.

\bibitem[Lubana et~al.(2023)Lubana, Bigelow, Dick, Krueger, and Tanaka]{lubana2023mechanistic}
E.~S. Lubana, E.~J. Bigelow, R.~P. Dick, D.~Krueger, and H.~Tanaka.
\newblock Mechanistic mode connectivity.
\newblock In \emph{International Conference on Machine Learning}, pages 22965--23004. PMLR, 2023.

\bibitem[Lucas et~al.(2021)Lucas, Bae, Zhang, Fort, Zemel, and Grosse]{lucas2021analyzing}
J.~Lucas, J.~Bae, M.~R. Zhang, S.~Fort, R.~Zemel, and R.~Grosse.
\newblock Analyzing monotonic linear interpolation in neural network loss landscapes.
\newblock \emph{arXiv preprint arXiv:2104.11044}, 2021.

\bibitem[Malach et~al.(2020)Malach, Yehudai, Shalev-Schwartz, and Shamir]{malach2020proving}
E.~Malach, G.~Yehudai, S.~Shalev-Schwartz, and O.~Shamir.
\newblock Proving the lottery ticket hypothesis: Pruning is all you need.
\newblock In \emph{International Conference on Machine Learning}, pages 6682--6691. PMLR, 2020.

\bibitem[Mei et~al.(2018)Mei, Montanari, and Nguyen]{mei2018mean}
S.~Mei, A.~Montanari, and P.-M. Nguyen.
\newblock A mean field view of the landscape of two-layer neural networks.
\newblock \emph{Proceedings of the National Academy of Sciences}, 115\penalty0 (33):\penalty0 E7665--E7671, 2018.

\bibitem[Mei et~al.(2019)Mei, Misiakiewicz, and Montanari]{mei2019mean}
S.~Mei, T.~Misiakiewicz, and A.~Montanari.
\newblock Mean-field theory of two-layers neural networks: dimension-free bounds and kernel limit.
\newblock In \emph{Conference on Learning Theory}, pages 2388--2464. PMLR, 2019.

\bibitem[Mirzadeh et~al.(2020)Mirzadeh, Farajtabar, Gorur, Pascanu, and Ghasemzadeh]{mirzadeh2020linear}
S.~I. Mirzadeh, M.~Farajtabar, D.~Gorur, R.~Pascanu, and H.~Ghasemzadeh.
\newblock Linear mode connectivity in multitask and continual learning.
\newblock \emph{arXiv preprint arXiv:2010.04495}, 2020.

\bibitem[Neyshabur et~al.(2020)Neyshabur, Sedghi, and Zhang]{neyshabur2020being}
B.~Neyshabur, H.~Sedghi, and C.~Zhang.
\newblock What is being transferred in transfer learning?
\newblock \emph{Advances in neural information processing systems}, 33:\penalty0 512--523, 2020.

\bibitem[Nguyen and Pham(2023)]{nguyen2023rigorous}
P.-M. Nguyen and H.~T. Pham.
\newblock A rigorous framework for the mean field limit of multilayer neural networks.
\newblock \emph{Mathematical Statistics and Learning}, 6\penalty0 (3):\penalty0 201--357, 2023.

\bibitem[Pensia et~al.(2020)Pensia, Rajput, Nagle, Vishwakarma, and Papailiopoulos]{pensia2020optimal}
A.~Pensia, S.~Rajput, A.~Nagle, H.~Vishwakarma, and D.~Papailiopoulos.
\newblock Optimal lottery tickets via subset sum: Logarithmic over-parameterization is sufficient.
\newblock \emph{Advances in neural information processing systems}, 33:\penalty0 2599--2610, 2020.

\bibitem[Peyr{\'e} et~al.(2019)Peyr{\'e}, Cuturi, et~al.]{peyre2019computational}
G.~Peyr{\'e}, M.~Cuturi, et~al.
\newblock Computational optimal transport: With applications to data science.
\newblock \emph{Foundations and Trends{\textregistered} in Machine Learning}, 11\penalty0 (5-6):\penalty0 355--607, 2019.

\bibitem[Phillips(2012)]{phillips2012chernoff}
J.~M. Phillips.
\newblock Chernoff-hoeffding inequality and applications.
\newblock \emph{arXiv preprint arXiv:1209.6396}, 2012.

\bibitem[Pittorino et~al.(2022)Pittorino, Ferraro, Perugini, Feinauer, Baldassi, and Zecchina]{pittorino2022deep}
F.~Pittorino, A.~Ferraro, G.~Perugini, C.~Feinauer, C.~Baldassi, and R.~Zecchina.
\newblock Deep networks on toroids: removing symmetries reveals the structure of flat regions in the landscape geometry.
\newblock In \emph{International Conference on Machine Learning}, pages 17759--17781. PMLR, 2022.

\bibitem[Qin et~al.(2022)Qin, Qian, Yi, Chen, Lin, Han, Liu, Sun, and Zhou]{qin2022exploring}
Y.~Qin, C.~Qian, J.~Yi, W.~Chen, Y.~Lin, X.~Han, Z.~Liu, M.~Sun, and J.~Zhou.
\newblock Exploring mode connectivity for pre-trained language models.
\newblock \emph{arXiv preprint arXiv:2210.14102}, 2022.

\bibitem[Rame et~al.(2022)Rame, Kirchmeyer, Rahier, Rakotomamonjy, Gallinari, and Cord]{rame2022diverse}
A.~Rame, M.~Kirchmeyer, T.~Rahier, A.~Rakotomamonjy, P.~Gallinari, and M.~Cord.
\newblock Diverse weight averaging for out-of-distribution generalization.
\newblock \emph{Advances in Neural Information Processing Systems}, 35:\penalty0 10821--10836, 2022.

\bibitem[Sagun et~al.(2017)Sagun, Evci, Guney, Dauphin, and Bottou]{sagun2017empirical}
L.~Sagun, U.~Evci, V.~U. Guney, Y.~Dauphin, and L.~Bottou.
\newblock Empirical analysis of the hessian of over-parametrized neural networks.
\newblock \emph{arXiv preprint arXiv:1706.04454}, 2017.

\bibitem[Shevchenko and Mondelli(2020)]{shevchenko2020landscape}
A.~Shevchenko and M.~Mondelli.
\newblock Landscape connectivity and dropout stability of sgd solutions for over-parameterized neural networks.
\newblock In \emph{International Conference on Machine Learning}, pages 8773--8784. PMLR, 2020.

\bibitem[Singh and Jaggi(2020)]{singh2020model}
S.~P. Singh and M.~Jaggi.
\newblock Model fusion via optimal transport.
\newblock \emph{Advances in Neural Information Processing Systems}, 33:\penalty0 22045--22055, 2020.

\bibitem[Tatro et~al.(2020)Tatro, Chen, Das, Melnyk, Sattigeri, and Lai]{tatro2020optimizing}
N.~Tatro, P.-Y. Chen, P.~Das, I.~Melnyk, P.~Sattigeri, and R.~Lai.
\newblock Optimizing mode connectivity via neuron alignment.
\newblock \emph{Advances in Neural Information Processing Systems}, 33:\penalty0 15300--15311, 2020.

\bibitem[Venturi et~al.(2019)Venturi, Bandeira, and Bruna]{venturi2019spurious}
L.~Venturi, A.~S. Bandeira, and J.~Bruna.
\newblock Spurious valleys in one-hidden-layer neural network optimization landscapes.
\newblock \emph{Journal of Machine Learning Research}, 20:\penalty0 133, 2019.

\bibitem[Villani et~al.(2009)]{villani2009optimal}
C.~Villani et~al.
\newblock \emph{Optimal transport: old and new}, volume 338.
\newblock Springer, 2009.

\bibitem[Vlaar and Frankle(2022)]{vlaar2022can}
T.~J. Vlaar and J.~Frankle.
\newblock What can linear interpolation of neural network loss landscapes tell us?
\newblock In \emph{International Conference on Machine Learning}, pages 22325--22341. PMLR, 2022.

\bibitem[Wang et~al.(2020)Wang, Yurochkin, Sun, Papailiopoulos, and Khazaeni]{wang2020federated}
H.~Wang, M.~Yurochkin, Y.~Sun, D.~Papailiopoulos, and Y.~Khazaeni.
\newblock Federated learning with matched averaging.
\newblock \emph{arXiv preprint arXiv:2002.06440}, 2020.

\bibitem[Weed and Bach(2019)]{weed2019sharp}
J.~Weed and F.~Bach.
\newblock Sharp asymptotic and finite-sample rates of convergence of empirical measures in wasserstein distance.
\newblock \emph{Advances in Neural Information Processing Systems}, 2019.

\bibitem[Wortsman et~al.(2022)Wortsman, Ilharco, Gadre, Roelofs, Gontijo-Lopes, Morcos, Namkoong, Farhadi, Carmon, Kornblith, et~al.]{wortsman2022model}
M.~Wortsman, G.~Ilharco, S.~Y. Gadre, R.~Roelofs, R.~Gontijo-Lopes, A.~S. Morcos, H.~Namkoong, A.~Farhadi, Y.~Carmon, S.~Kornblith, et~al.
\newblock Model soups: averaging weights of multiple fine-tuned models improves accuracy without increasing inference time.
\newblock In \emph{International Conference on Machine Learning}, pages 23965--23998. PMLR, 2022.

\bibitem[Wu(2016)]{classpacking}
Y.~Wu.
\newblock {Packing, covering, and consequences on minimax risk}.
\newblock http://www.stat.yale.edu/\(\sim \)yw562/ teaching/598/lec14.pdf, 2016.
\newblock [Online; accessed October-10-2023].

\bibitem[Yunis et~al.(2022)Yunis, Patel, Savarese, Vardi, Frankle, Walter, Livescu, and Maire]{yunis2022convexity}
D.~Yunis, K.~K. Patel, P.~H.~P. Savarese, G.~Vardi, J.~Frankle, M.~Walter, K.~Livescu, and M.~Maire.
\newblock On convexity and linear mode connectivity in neural networks.
\newblock In \emph{OPT 2022: Optimization for Machine Learning (NeurIPS 2022 Workshop)}, 2022.

\bibitem[Yurochkin et~al.(2019)Yurochkin, Agarwal, Ghosh, Greenewald, Hoang, and Khazaeni]{yurochkin2019bayesian}
M.~Yurochkin, M.~Agarwal, S.~Ghosh, K.~Greenewald, N.~Hoang, and Y.~Khazaeni.
\newblock Bayesian nonparametric federated learning of neural networks.
\newblock In \emph{International conference on machine learning}, pages 7252--7261. PMLR, 2019.

\bibitem[Zhao et~al.(2020)Zhao, Chen, Das, Ramamurthy, and Lin]{zhao2020bridging}
P.~Zhao, P.-Y. Chen, P.~Das, K.~N. Ramamurthy, and X.~Lin.
\newblock Bridging mode connectivity in loss landscapes and adversarial robustness.
\newblock \emph{arXiv preprint arXiv:2005.00060}, 2020.

\bibitem[Zhou et~al.(2023{\natexlab{a}})Zhou, Zhang, and Tsang]{zhou2023mode}
T.~Zhou, J.~Zhang, and D.~H. Tsang.
\newblock Mode connectivity and data heterogeneity of federated learning.
\newblock \emph{arXiv preprint arXiv:2309.16923}, 2023{\natexlab{a}}.

\bibitem[Zhou et~al.(2023{\natexlab{b}})Zhou, Yang, Yang, Yan, and Hu]{zhou2023going}
Z.~Zhou, Y.~Yang, X.~Yang, J.~Yan, and W.~Hu.
\newblock Going beyond linear mode connectivity: The layerwise linear feature connectivity.
\newblock \emph{arXiv preprint arXiv:2307.08286}, 2023{\natexlab{b}}.

\end{thebibliography}
\bibliographystyle{abbrvnat}

\appendix
\newpage
\onecolumn

\section{EXTENDED RELATED WORK}

\citet{frankle2020linear} were the ones to coin the term Linear Mode Connectivity and the first to recognize the importance of this structure. Notably, they studied this structure through its connections with pruning and the lottery ticket hypothesis \citep{frankle2018lottery,malach2020proving,pensia2020optimal,ferbach2022general}.

It is worth noting that some recent works in the literature study linear mode connectivity specifically in a layer-wise manner, especially due to the permutation symmetries we described in \cref{section:intro}. For example, \citet{zhou2023going,adilova2023layerwise} both study the effect of layer-wise averaging when connecting two deep neural networks. This is aligned with our theoretical study since we recursively align deep networks layer after layer.

The mode connectivity framework has been used as a tool to better understand the similarity between two models or the effect of a training procedure on the trained model. For example, \cite{lubana2023mechanistic} introduces mechanistic similarity to quantify how two models react to the same alteration of the data in latent space. They show relations between mode connectivity and mechanistic similarity. Especially they prove that if two models cannot be linear mode connected, then they are mechanistically dissimilar. Moreover, \cite{mirzadeh2020linear} use the mode connectivity framework to study whether continual and sequential learning (two different training procedure for multitask learning) are converging to a similar solution. Finally, \cite{qin2022exploring} explores mode connectivity of pretrained langage models, especially how hyper-parameters affect mode connectivity and how mode connectivity evolves during training.

Computational optimal transport has been leveraged by \cite{singh2020model,akash2022wasserstein} to find paths between two models in parameter space. The latter formulates the model fusion problem as a Wasserstein barycenter problem.

Past works have studied mode connectivity through the lens of model averaging with applications in federated learning \citep{yurochkin2019bayesian,wang2020federated}. \cite{zhou2023mode} studies the effect of data heterogeneity in federated learning on mode connectivity of global modes.

\section{PROOFS AND DETAILS ABOUT OPTIMAL TRANSPORT THEORY FOR LMC}

\subsection{Background on optimal transport and convexity lemma}

Optimal transport is a mathematical framework that aims at quantifying distances between distributions. We refer to the books \citet{villani2009optimal} and \citet{peyre2019computational} (focused on computational aspects and applications to data science) for an extensive overview of this topic.

\begin{definition}[Wasserstein distance \citep{villani2009optimal}]
\label{def:wass}
    Let $(\mathcal{X},d)$ a Polish metric space, $p\in [1,\infty)$. $\forall \mu,\nu \in \gP_1(\mathcal{X})$, define the $p$-Wasserstein distance between $\mu$ and $\nu$ by:
    \begin{equation*}
\gW_p(\mu,\nu)=\left(\inf_{\pi\in\Pi(\mu,\nu)}\int_{\mathcal{X}^2}d^p(x,y)d\pi(x,y)\right)^{\nicefrac{1}{p}}
    \end{equation*}
\end{definition}

Recall that $\Pi(\mu,\nu)$ denotes the set of coupling between $\mu$ and $\nu$, i.e.,

\begin{equation*}
    \pi \in \Pi(\mu,\nu)\Leftrightarrow \pi \in \gP_1(\mathcal{X}^2) \text{ with marginals } \mu,\nu
\end{equation*}

This defines a distance and especially it satisfies the triangular inequality. Moreover we state a Jensen type inequality proved in \citet{villani2009optimal}:

\begin{lemma}[Convexity of the optimal cost (Theorem 4.8 in \citet{villani2009optimal})]
\label{lemma:villani}
    Let $d:\sR^n\times \sR^n\rightarrow \sR_+$ a distance and for $p\in [1,\infty)$, $\gW_p: \gP_1(\sR^n)^2\rightarrow\sR_+$ its associate Wasserstein distance. Let $(\Theta,\lambda)$ be a probability space and $\mu_\theta, \nu_\theta$ two measurable functions on that space taking values in $\gP_1(\sR^n)$. Then,
    \begin{equation*}
        \gW_p^p\left(\int_\Theta \mu_\theta \lambda(d\theta),\int_\Theta \nu_\theta \lambda(d\theta)\right)\leq \int_\Theta \gW_p^p(\mu_\theta,\nu_\theta)\lambda(d\theta)
    \end{equation*}
\end{lemma}

\begin{proof}
    To apply Theorem $4.8$ in \citet{villani2009optimal}, we just need to notice that $d^p(\cdot,\cdot)$ is continuous, $d^p(\cdot,\cdot)\geq 0$ and the associated optimal cost functional is $\gW_p^p(\cdot,\cdot)$.
\end{proof}

\subsection{Validity of the OT viewpoint: Birkhoff's theorem}
\label{validity_OT_viewpoint}
We have motivated in \Cref{section:preliminaries} the following minimization problem (\Cref{eq:P_l}):

\begin{equation*}
    \begin{split}
            \Pi_\ell&=\argmin_{\Pi\in \gS_{m_\ell}} \|W_A^\ell-\Pi W_B^\ell \Pi_{\ell -1}^T\|_2^2\\
            &=\argmin_{\pi \in \gS_{m_\ell}} \frac{1}{m_\ell}\sum_{i=1}^{m_\ell}\|[W_{A}^\ell]_{i:}-[W_B^\ell \Pi^T_{\ell -1}]_{\pi_i:}\|_2^2
    \end{split}
\end{equation*}

Since LMC effectiveness will be related to the effectiveness of this optimization problem, we want to quantify the minimization error:

\begin{equation*}
\min_{\Pi\in \gS_{m_\ell}} \|W_A^\ell-\Pi W_B^\ell \Pi_{\ell -1}^T\|_2^2
\end{equation*}

We highlight the similarity with previous \Cref{def:wass}. The main difference being that the latter minimizes the cost among all couplings $\pi \in \Pi(\mu,\nu)$ between the two distributions, especially transport plans that can split mass. Permutation must be on the other hand seen as Monge maps, i.e. deterministic maps. This ambiguity is solved with the following lemma:

\begin{lemma}[Proposition 2.1 in \citet{peyre2019computational}]
    \label{lemma:birkhoff}
    Let $m,n$ integers and $x_1, ..., x_m, y_1, ..., y_m$, $2m$ points of $\sR^n$. Let $\hat{\mu}_m=\frac{1}{m}\sum_{i=1}^m\delta_{x_i}, \hat{\nu}_m=\frac{1}{m}\sum_{i=1}^m\delta_{y_i}$ their associated empirical measures. Consider $d: \sR^n \times \sR^n \rightarrow \sR_+$ a distance function and for $p\in [1,\infty), \gW_p$ its associated Wasserstein distance.
    Then one has:
    \begin{equation*}
        \gW_p:=\inf_{\pi \in \Pi(\hat{\mu}_m, \hat{\nu}_m)}\left(\int_{\sR^n\times \sR^n} d(x,y)^p d\pi(x,y)\right)^{\frac{1}{p}}=\min_{\pi \in \gS_m}\left( \frac{1}{m}\sum_{i=1}^m d(x_i,y_{\sigma_i})^p\right)^{\frac{1}{p}}
    \end{equation*}
\end{lemma}

\begin{proof}
$\Pi(\hat{\mu}_m, \hat{\nu}_m)$ is the convex envelope of its extremal points which are described by permutations by Birkhoff's theorem. Moreover,
\begin{equation*}
    \pi \in \gP_1(\sR^n\times \sR^n) \rightarrow \int d(x,y)^p d\pi(x,y)
\end{equation*}
is linear and therefore its infimum is attained on one extremal point of $\Pi(\hat{\mu}_m, \hat{\nu}_m)$
\end{proof}

This lemma implies equality between the Wasserstein distance between two empirical measures and the minimum cost over the set of permutations. We can therefore restrict our study to permutations while still using tools from general optimal transport theory and convergence rates of empirical measures in Wasserstein distance.

\subsection{Technical lemmas}

The following lemma will be very useful in the following and shows that Binomial variables are concentrated around their expectation.

\begin{lemma}[Hoeffding inequality for Binomial variables]
\label{lemma:binomial}
    Let $\gB_{\frac{p}{2},m}\sim \gB(\frac{p}{2},m)$ a binomial variable with $p\in [0,1]$. Then,
    \begin{equation*}
        \sP(\frac{\gB_{\frac{p}{2},m}}{m}\geq p)\leq \exp(-\frac{p^2m}{2})
    \end{equation*}
\end{lemma}

Proof: This is a simple application of Hoeffding concentration inequality \citep{phillips2012chernoff} since a Binomial variable is a sum of independent Bernoulli variables.

\begin{lemma}
\label{cut_transport_plan}
    Let $a,b> 0$ such that $a+b=1$ and let $\mu_1,\mu_2,\nu_1,\nu_2\in \gP_1(\mathcal{X})$ with $(\gX,d)$ a Polish space.
Then, $\forall p\in [1,\infty)$:

\begin{equation}
    \gW_p^p(a\mu_1+b\mu_2,a\nu_1+b\nu_2)\leq a\gW_p^p(\mu_1,\nu_1)+b\gW_p^p(\mu_2,\nu_2)
\end{equation}
\end{lemma}

\begin{proof}
    Just apply \Cref{lemma:villani} with $\mu_\Theta,\nu_\Theta$ such that $\sP\left((\mu_\Theta,\nu_\Theta)=(\mu_1,\nu_1)\right)=a$ and $\sP\left((\mu_\Theta,\nu_\Theta)=(\mu_2,\nu_2)\right)=b$.
\end{proof}

\begin{lemma}[Hölder, Remark $6.6$ in \citet{villani2009optimal}]
    \label{Hölder} Let $(\mathcal{X},d)$ a Polish space, let $p, q \in [1,\infty]$ such that $p\leq q$:
    \begin{equation}
        \forall \mu,\nu\in \gP_1(\mathcal{X}),\, \gW_p\leq \gW_q
    \end{equation}
\end{lemma}

\begin{proof}
    Just apply Hölder's inequality.
\end{proof}

\subsection{Definitions and technical lemma on packing numbers}

Let $n \in \sN^*$, $S \subset \sR^n$ and $d(\cdot,\cdot)$ a distance of $\sR^n$. We recall the definitions of packing numbers and covering numbers as well as two lemmas stated and proved in  \citet{classpacking}:

\begin{definition}[$\epsilon$-covering number \citep{weed2019sharp}]
    The $\epsilon$-covering of a set $S$ denoted $\gN_\epsilon(S)$ is the minimum number $m$ of closed balls $B_1,\ldots,B_m$ of radius $\epsilon$ such that $S\subset \bigcup_{i=1}^mB_i$.
\end{definition}

\begin{definition}[$\epsilon$-packing number]
    The $\epsilon$-packing number of a set $S$ denoted $\gP_\epsilon(S)$ is the maximum number $m$ of distinct points $\theta_1,...,\theta_m\in S$ such that $\forall i \neq j$, $\|\theta_i-\theta_j\|> \epsilon$.
\end{definition}

\begin{lemma}
\label{lemma:packing}
    If the distance $d(\cdot,\cdot)$ comes from a norm, ($d(x,y)=\|x-y\|)$, denoting $\Leb_n$ the Lebesgue measure in $\sR^n$ we have: $\forall S\subset \sR^n, \forall \epsilon > 0$,
    \begin{equation*}
        \gN_\epsilon(S)\leq \frac{\Leb_n(S+\gB(0,\epsilon/2))}{\Leb_n(\gB(0,\epsilon/2))}
    \end{equation*}
\end{lemma}

\begin{proof}
    We prove first $\gN_\epsilon(S)\leq \gP_{\epsilon}(S)$. Indeed considering $m=\gP_\epsilon(S)$ and $\theta_1,\ldots,\theta_m$ associated, we know by definition of the $\epsilon$-packing number that $\forall \theta \in S,\exists i \in [m]$ such that $\|\theta_i-\theta\|\leq \epsilon$. This shows that $S \subset \bigcup_{i=1}^m\gB(\theta_i,\epsilon)$

    Now, on the other hand, we know that all balls $\gB(\theta_i,\frac{\epsilon}{2})$ are disjoint and $\bigcup_{i=1}^m\gB(\theta_i,\frac{\epsilon}{2})\subset S+\gB(0,\epsilon/2)$. This yields the result by a volume (i.e., Lebesgue measure) argument.
\end{proof}

\begin{lemma}
    \label{lemma:packing2}
     If the distance $d(\cdot,\cdot)$ comes from a norm, ($d(x,y)=\|x-y\|)$ we have: $\forall S\subset \sR^n, \forall \epsilon > 0$,
    \begin{equation*}
        \gN_\epsilon(S)\geq \left(\frac{1}{\epsilon}\right)^n\frac{\Leb_n(S)}{\Leb_n(\gB(0,1))}
    \end{equation*}
\end{lemma}

\begin{proof}
    Notice that given a covering $\bigcup_{i=1}^{\gN_\epsilon(S)}\gB(x_i,\epsilon)\supset S$, by a volume argument we get 
    \begin{equation*}
        \Leb_n(S)\leq \epsilon^n\gN_\epsilon(S)\Leb_n(\gB(0,1))
    \end{equation*}
\end{proof}

where we have used homogeneity of the norm.

\subsection{Lemmas on convergence rates of empirical measures}

This section is devoted to proving convergence rates in Wasserstein distance of empirical measure towards the underlying distribution whose points are drawn. More precisely, given $\mu\in \gP_1(\sR^n)$ a probability measure on an euclidean space and $p\in [1,\infty)$, we will focus on bounding the quantity $\E\left[\gW_p^p(\hat{\mu}_m,\mu)\right]$ as $m$ grows, where $\hat{\mu}_m$ is a random empirical measure i.e., $\hat{\mu}_m=\frac{1}{m}\sum_{i=1}^m\delta_{x_i}$ where $x_i \ssim \mu$.

We will first prove the following lemma:

\begin{lemma}
\label{exp_ext}
    Consider $X_n\sim \gN\left(0,\tfrac{I_n}{n}\right)$ a random variable whose law is parametrized by $n\in \sN^*$. There exists a universal constant $D_0$ such that 
    \begin{equation*}
        \forall n \in \sN^*, \forall c > 1, \E[\|X_n\|_2^2\;|\;\|X_n\|_2> c]\leq D_0c^2
    \end{equation*}
    Jensen inequality implies:
    \begin{equation*}
        \forall n \in \sN^*, \forall c > 1, \E[\|X_n\|_2;|\;\|X_n\|_2> c]\leq \sqrt{D_0}c
    \end{equation*}
\end{lemma}

\begin{proof}
    We write 
    \begin{equation*}
    \begin{split}
        \E[\|X_n\|_2^2|\|X_n\|_2> c]&=\frac{\int_c^{+\infty} r^2 r^{n-1} e^{-\tfrac{r^2n}{2}}dr}{\int_c^{+\infty} r^{n-1} e^{-\tfrac{r^2n}{2}}dr}\\
        &\leq 4c^2+\sum_{m=2}^{\infty}\frac{\int_{mc}^{(m+1)c} r^2 r^{n-1} e^{-\tfrac{r^2n}{2}}dr}{\int_c^{2c} r^{n-1} e^{-\tfrac{r^2n}{2}}dr}\\
        &\leq 4c^2 +\sum_{m=2}^{\infty}\frac{\int_{0}^{1} ((m+1)c)^{n+1} e^{-\tfrac{((m+t)c)^2n}{2}}cdt}{\int_0^{1} c^{n-1} e^{-\tfrac{((1+t)c)^2n}{2}}cdt}
    \end{split}  
    \end{equation*}

But, $\forall t \in [0,1], \forall m \geq 2$, \begin{equation*}
    \frac{((m+1)c)^{n+1} e^{-\tfrac{((m+t)c)^2n}{2}}}{c^{n-1} e^{-\tfrac{((1+t)c)^2n}{2}}}\leq c^2 (m+1)^{n+1}e^{-\tfrac{c^2n}{2}(m^2-1)}\leq c^2 (2m)^{2n}e^{-\tfrac{c^2n}{2}(m^2-1)}\leq c^2 e^{-n\left(\tfrac{c^2}{2}(m^2-1)-2\log(2m)\right)}
\end{equation*}

It is clear that uniformly in $n$, $\sum_{m=2}^{\infty}c^2 e^{-n\left(\tfrac{c^2}{2}(m^2-1)-2\log(2m)\right)}\xrightarrow[c\rightarrow \infty]{} 0$ which proves the lemma.
\end{proof}

We will now prove a bound on the rate of convergence in Wasserstein distance of an empirical measure to the underlying distribution when this one has a bounded support.
Denote $\gB_2^k(0,r)$ the euclidean ball centered around $0$ of radius $r$ in dimension $k$.

\begin{lemma}
\label{supp:ball}
    Let $\mu\in \gP_1(\sR^n)$ be a measure whose support is included in $\gB_2^n\left(0,\frac{1}{12}\right)\subset \sR^n$ with $n\geq 5$
    Then, $\forall m \geq 1$ we have
    \begin{equation*}
        \E[\gW_2^2(\hat{\mu}_m, \mu)]\leq D_1 \left(\frac{1}{m}\right)^{2/n}
    \end{equation*}
Where $D_1=27^2\left(2+\frac{1}{\sqrt{3}-1}\right)$

\end{lemma}

\begin{proof}
    We know from \Cref{lemma:packing} that when considering $\|\cdot\|_2$ the distance for defining covering number, $\forall \epsilon' \leq \tfrac{1}{6}$: 
    \begin{equation*}
\gN_{\epsilon'}(\gB_2^n(0,\nicefrac{1}{12}))=\left(\frac{\tfrac{1}{12}+\tfrac{\epsilon'}{2}}{\tfrac{\epsilon'}{2}}\right)^n\leq \left(\frac{\tfrac{1}{6}+\epsilon'}{\epsilon'}\right)^n\leq (3\epsilon')^{-n}
    \end{equation*}
    and therefore also when $\epsilon'\leq \tfrac{1}{27}$
Applying Proposition 15 from \citet{weed2019sharp} we get that that since $\Supp(\mu)\subset \gB_2^n(0,\nicefrac{1}{12})\subset \gB_2^n(0,\nicefrac{1}{12})+\gB_2^n(0,\epsilon)$ for any $\epsilon> 0$,
if $n\geq 5$,
\begin{equation*}
    \forall m \geq 1, \E[\gW_2^2(\hat{\mu}_m, \mu)]\leq 27^2\left(2+\frac{1}{\sqrt{3}-1}\right)  \left(\frac{1}{m}\right)^{2/n}
\end{equation*}
as well as if $n\geq 3$,

\begin{equation*}
    \forall m \geq 1, \E[\gW_1(\hat{\mu}_m, \mu)]\leq 27\left(2+\frac{1}{\sqrt{3}-1}\right)  \left(\frac{1}{m}\right)^{1/n}
\end{equation*}

\end{proof}

We can prove the same kind of inequality when $\mu$ concentrates mass around an approximately low dimensional set.

\begin{lemma}
\label{supp:ball_approx}
    Let $\mu\in \gP_1(\sR^n)$ be a measure whose support is included in $\gB_2^k\left(0,\nicefrac{1}{12}\right)\times \gB^{n-k}_2(0,r)$ with $k\geq 5$.
    Then, $\forall m \leq (3r)^{-k}$ we have
    \begin{equation*}
        \E[\gW_2^2(\hat{\mu}_m, \mu)]\leq D_1\left(\frac{1}{m}\right)^{2/k}
    \end{equation*}
    where $D_1=27^2\left(2+\frac{1}{\sqrt{3}-1}\right)$

\end{lemma}

\begin{proof}
    This is the same proof as before, just notice that $\Supp(\mu)\subset \gB_2^k(0,\nicefrac{1}{12})\times \{0\}^{n-k}+\gB^n_2(0,r)$ and as before:

    \begin{equation}
        \gN_{\epsilon'}(\gB_2^k(0,\nicefrac{1}{12})\times \{0\}^{n-k})\leq \gN_{\epsilon'}(\gB_2^k(0,\nicefrac{1}{12}))\leq (3\epsilon')^{-k}
    \end{equation}
if $\epsilon'\leq \tfrac{1}{27}$.
\end{proof}

We will now extend our results to unbounded variables very concentrated around bounded sets, beginning with multivariate normal random variable.

\begin{lemma}
\label{expect_wass_normal}
    Consider a centered multivariate normal distribution $\mu$ on $\sR^n$ with covariance matrix $\Diag(\lambda_1,\ldots,\lambda_n)$ where $\frac{1}{n}\geq \lambda_1\geq \ldots \geq \lambda_n\geq 0$. There exists two universal constants $D_2,E_2$ such that $\forall n \geq 5, \forall m \in \sN^*$, if $m\geq E_2^{n}$ then, 
    \begin{equation*}
        \E[\gW_2^2(\hat{\mu}_m,\mu)]\leq \frac{D_2}{n}\log(m)\left(\frac{1}{m}\right)^{2/n}
    \end{equation*}
In that case,

\begin{equation*}
        \E[\gW_1(\hat{\mu}_m,\mu)]\leq \frac{\sqrt{D_2}}{\sqrt{n}}\sqrt{\log(m)}\left(\frac{1}{m}\right)^{1/n}
\end{equation*}

\end{lemma}

\begin{proof}
    We will use previous \Cref{supp:ball}. The problem is that it applies only for a bounded distribution. Therefore we will have to bound the mass of a multivariate Gaussian outside of a euclidean ball. We will prove the lemma for $\lambda_1=...=\lambda_n=\frac{1}{n}$ by noticing that it extends for smaller eigenvalues by rescaling the axis.

    Let $f \in (0,1), c> 0, X \sim \mu$.

    Lemma 1 from \citet{weed2019sharp}
    tells us:
    \begin{equation*}
        \sP\left(\|X\|_2^2\geq c^2 \sum_{i=1}^n\lambda_i\right)\leq e^{-\tfrac{c^2}{4}}
        \end{equation*}

        Noticing $\sum_{i=1}^n \lambda_i\leq 1$ and taking $c=2\sqrt{\log\left(\tfrac{1}{f}\right)}$ we get that:

        \begin{equation*}
            \sP(\|X\|_2^2\geq c^2 )\leq f
        \end{equation*}

        Using \Cref{lemma:binomial} with $p=2f$ we get that with probability at least $1-\exp(-2f^2m)$, a fraction at least $1-2f$ of vectors $x_i$ lies in $\gB_2(0,c)$. We denote $H_f$ this event such that $\sP(H_f)\geq 1-\exp(-2f^2m)$.

Further denote $I$ the corresponding set of indices for $x_i$, $\hat{\mu}_{m,I}=\frac{\sum_{i\in I}\delta_{x_i}}{|I|}$ and $\hat{\mu}_{m,I^c}=\frac{\sum_{i\notin I}\delta_{x_i}}{|I^c|}$. Finally for a Borel set $U\subset \sR^n$ denote $\mu_{|U}$ the renormalized restricted measure $\mu$ on $U$: $\mu_{|U}=\frac{1}{\mu(U)}\mu\mathds{1}_U$

We will now consider two cases:

\paragraph{1st case} We consider the case $\frac{|I|}{m}\leq \mu(\gB(0,c))$ and denote $\Case_1$ this set.

In that case, we can write using \Cref{cut_transport_plan}:

    \begin{equation*}
    \begin{cases}
\gW_2^2(\hat{\mu}_m,\mu)&\leq \frac{|I|}{m}\gW_2^2(\hat{\mu}_{m,I}, \mu_{|\gB_2(0,c)})+\frac{|I^c|}{m}\gW_2^2\left(\hat{\mu}_{m,I^c}, \frac{\mu_{|\gB_2(0,c)}-\frac{|I|}{m}\mu_{|\gB_2(0,c)}}{\frac{|I^c|}{m}}\right)\\
\gW_1(\hat{\mu}_m,\mu)&\leq \frac{|I|}{m}\gW_1(\hat{\mu}_{m,I}, \mu_{|\gB_2(0,c)})+\frac{|I^c|}{m}\gW_1\left(\hat{\mu}_{m,I^c}, \frac{\mu_{|\gB_2(0,c)}-\frac{|I|}{m}\mu_{|\gB_2(0,c)}}{\frac{|I^c|}{m}}\right)

    \end{cases}
\end{equation*}

By previous \Cref{supp:ball}, we know the existence of $D_1$ such that if $n\geq 5$:

\begin{equation*}
    \E\left[\gW_2^2(\hat{\mu}_{m,I}, \mu_{|\gB_2(0,c)})\right]\leq D_1 (12c)^2 \left(\frac{1}{|I|}\right)^{2/n}\leq (144D_1)c^2 \left(\frac{1}{\frac{|I|}{m}m}\right)^{2/n}
\end{equation*}

Therefore we get since $\frac{2}{n}\leq 1$ and $\frac{|I|}{n}\leq 1$:

\begin{equation*}
    \E\left[\frac{|I|}{m}\gW_2^2(\hat{\mu}_{m,I}, \mu_{|\gB_2(0,c)})\right]\leq (144D_1)c^2\left(\frac{1}{m}\right)^{2/n}
\end{equation*}

We know from \Cref{exp_ext} the existence of a universal constant $D_0$ such that: $\forall c \geq 1, \E[\|X\|_2^2|\|X\|_2\geq c]\leq D_0c^2$.
Using a triangular inequality

\begin{equation*}
    \begin{split}
        \E\left[\gW_2^2(\hat{\mu}_{m,I^c}, \mu_{|\gB_2(0,c)^c})\right]&\leq 4D_0c^2
    \end{split}
\end{equation*}

Finally, conditioned on the event $H_f$ and that we are in $\Case_1$, we get that if $c\geq 1$:

\begin{equation*}
\begin{cases}
    \E[\gW_2^2(\hat{\mu}_m,\mu)|H_f\bigcap \Case_1]\leq (144D_1)c^2\left(\frac{1}{m}\right)^{2/n}+8fD_0c^2\\
    c=2\sqrt{\log(\tfrac{1}{f})}
\end{cases}
\end{equation*}

\paragraph{2nd case} We consider the case $\frac{|I|}{m}> \mu(\gB_2(0,c))$ and denote $\Case_2$ this set.

In that case, we denote $I'\subset I$ taken randomly uniformly, such that $|I'|=\max\left\{k\geq 1, \frac{k}{m}\leq \mu(\gB(0,c))\right\}$ and denote $\hat{\mu}_{m,I'}$ the renormalized empirical measure with points in $I'$ and $\hat{\mu}_{m,I\bigcap I^{'c}}$ the renormalized empirical measure with points in $I\bigcap I'^c$.

We can write:

    \begin{equation*}
    \begin{split}
\gW_2^2(\hat{\mu}_m,\mu)&\leq \frac{|I'|}{m}\gW_2^2(\hat{\mu}_{m,I'}, \mu_{|\gB_2(0,c)})+\frac{|I|-|I'|}{m}\gW_2^2(\hat{\mu}_{m,I\bigcap I'^c}, \mu_{|\gB_2(0,c)^c})\\
&+\left(\mu\big(\gB_2(0,c)\big)-\frac{|I'|}{m}\right)\gW_2^2(\hat{\mu}_{m,I\bigcap I'^c}, \mu_{|\gB_2(0,c)})+\frac{|I^c|}{m}\gW_2^2\left(\hat{\mu}_{m,I^c}, \mu_{|\gB_2(0,c)^c}\right)
    \end{split}
\end{equation*}

Provided $fm> 1$, we know that $\frac{|I'|}{m}\geq 1-3f$.

We can repeat all the previous arguments and get that if $c\geq 1$ and $m$:

\begin{equation*}
\begin{cases}
    \E[\gW_2^2(\hat{\mu}_m,\mu)|H_f\bigcap \Case_2]\leq (144D_1)c^2\left(\frac{1}{m}\right)^{2/n}+16fD_0c^2\\
    c=2\sqrt{\log(\tfrac{1}{f})}
\end{cases}
\end{equation*}

Finally,

\begin{equation*}
    \E[\gW_2^2(\hat{\mu}_m,\mu)]\leq \sP(H_f)\E[\gW_2^2(\hat{\mu}_m,\mu)|H_f]+(1-\sP(H_f))\E[\gW_2^2(\hat{\mu}_m,\mu)|H_f^c]
\end{equation*}

We can easily bound as before 
        $\E[\gW_2^2(\hat{\mu}_m,\mu)|H_f^c]\leq 4D_0c^2$

which yields finally:

\begin{equation*}
    \E[\gW_2^2(\hat{\mu}_m,\mu)]\leq (144D_1)4\log(\frac{1}{f})\left(\frac{1}{m}\right)^{2/n}+16fD_04\log(\frac{1}{f})+2\exp(-f^2m)4D_04\log(\frac{1}{f})
\end{equation*}

Taking $f=\left(\frac{1}{m}\right)^{2/n}$ we get
\begin{equation*}
\E[\gW_2^2(\hat{\mu}_m,\mu)]\leq\frac{1152D_1}{n}\log(m)\left(\frac{1}{m}\right)^{2/n}+D_0\frac{128}{n}\log(m)\left(\frac{1}{m}\right)^{2/n}+\frac{64}{n}\exp(-m^{1-4/n})\log(m)
\end{equation*}

Note now that there exists a universal constant $C> 0$ such that $\forall n \geq 5, \forall m \geq 1$:

\begin{equation*}
    \exp(-m^{1-4/n})\leq \exp(-m^{1/5})\leq C \left(\frac{1}{m}\right)^{2/5}\leq C\left(\frac{1}{m}\right)^{2/n}
\end{equation*}

Finally we get the existence of universal constants $D_2,E_2$ such that if $\left(\frac{1}{m}\right)^{2/n}\leq \frac{1}{E_2^2}$ (to ensure $c\geq 1$ take for example $E_2=\exp(\frac{1}{4})$),
\begin{equation*}
    \E[\gW_2^2(\hat{\mu}_m,\mu)]\leq \frac{D_2}{n}\log(m)\left(\frac{1}{m}\right)^{2/n}
\end{equation*}

To prove the second part of the lemma just apply \Cref{Hölder} and Jensen inequality.

\end{proof}

We can finally extend \Cref{supp:ball_approx} to unbounded distributions as we just extended \Cref{supp:ball} to unbounded distributions.

\begin{lemma}
\label{expect_wass_lowdim}
    Let $\lambda_1\geq\ldots \geq \lambda_n$ and $\mu=\gN(0, \Diag((\lambda_i)_{i=1}^n))$, with $k\geq5$. Suppose $1\geq \sum_{i=1}^k\lambda_i$.
    Denote $\eta=\frac{\sqrt{\sum_{i=k+1}^n\lambda_i}}{4\sqrt{\sum_{i=1}^k}\lambda_i}$. We know the existence of two universal constants $D_2',E'_2$ such that if $\eta^{-k}\geq m\geq E_2'^{k}$, then:
    \begin{equation*}
        \E[\gW_2^2(\hat{\mu}_m,\mu)]\leq
        \frac{D'_2}{k}\log(m)\left(\frac{1}{m}\right)^{2/k}]
    \end{equation*}

   In that case,
    \begin{equation*}
        \E[\gW_1(\hat{\mu}_m,\mu)]\leq
        \frac{\sqrt{D'_2}}{\sqrt{k}}\sqrt{\log(m)}\left(\frac{1}{m}\right)^{1/k}]
    \end{equation*}

\end{lemma}

\begin{proof}
    We will follow the same steps as previously.
    Let $X\sim \mu$ and denote $\ubar{X}=(X_1,\ldots,X_k)\in \sR^k$, $\bar{X}=(X_{k+1},\ldots,X_n)\in \sR^{n-k}$.
    \begin{equation*}
    \begin{split}
        &\sP\left(\|\ubar{X}\|_2^2\geq c^2\sum_{i=1}^k\lambda_i\right)\leq e^{-\frac{c^2}{4}}\\
        &\sP\left(\|\bar{X}\|_2^2\geq c^2\sum_{i=k+1}^n\lambda_i\right)\leq e^{-\frac{c^2}{4}}
    \end{split}
    \end{equation*}
    
Take $c=2\sqrt{\log(\tfrac{2}{f})}$. Then, by the same arguments as before, using \Cref{lemma:binomial} and union bounds, with probability at least $1-2\exp(-f^2m/2)$ a fraction at least $1-2f$ of points $x_i$ are in $\gB_c:=\gB_k(0, c\sqrt{\sum_{i=1}^k\lambda_i})\times \gB_{n-k}(0, c\sqrt{\sum_{i=k+1}^n\lambda_i})$. We denote $H_f$ such an event and $I$ such a set of indices.

By using \Cref{supp:ball_approx}, we know that in that case we can bound

\begin{equation*}
     \E\left[\gW_2^2(\hat{\mu}_{m,I}, \mu_{|\gB_c})|H_f\right]\leq \left(12c\sqrt{\sum_{i=1}^k\lambda_i}\right)^2C_1\left(\frac{1}{m}\right)^{2/k}\leq (12c)^2C_1\left(\frac{1}{m}\right)^{2/k} 
\end{equation*}

if $m\leq \left(\frac{3\sqrt{\sum_{i=k+1}^m\lambda_i}}{12\sqrt{\sum_{i=1}^k\lambda_i}}\right)^{-k}$ where we have used $\sum_{i=1}^k\lambda_i\leq 1$.

Moreover, we know that

\begin{equation*}
\begin{split}
      \E\left[\|X\|_2^2|X\notin \gB^k_2(0, c\sqrt{\sum_{i=1}^k\lambda_i})\times \gB_2^{n-k}(0, c\sqrt{\sum_{i=k+1}^n\lambda_i})\right]&\leq  \E\left[\|\ubar{X}\|_2^2|\ubar{X}\notin \gB^k_2(0, c\sqrt{\sum_{i=1}^k\lambda_i})\right]\\
      &+ \E\left[\|\bar{X}\|_2^2|\bar{X}\notin \gB^{n-k}_2(0, c\sqrt{\sum_{i=k+1}^n\lambda_i})\right]\\
      &\leq D_0c^2\sum_{i=1}^k\lambda_i+D_0c^2\sum_{i=k+1}^n\lambda_i\\
      &\leq 2D_0 c^2 \sum_{i=1}^k\lambda_i \\
      &\leq 2D_0c^2
\end{split}
\end{equation*}

We just need to differentiate the same two cases as in the proof of \Cref{expect_wass_normal} to get that finally, if $m\leq \left(\frac{\sqrt{\sum_{i=k+1}^m\lambda_i}}{4\sqrt{\sum_{i=1}^k\lambda_i}}\right)^{-k}$ we can bound as before, with $c=2\sqrt{\log(\tfrac{2}{f})}$:

\begin{equation*}
    \E[\gW_2^2(\hat{\mu}_m,\mu)]\leq (12c)^2C_1\left(\frac{1}{m}\right)^{2/k} +16f*4C_0c^2+2\exp(-\frac{f^2m}{2})*4*2C_0c^2
\end{equation*}

Taking $f=\left(\frac{1}{m}\right)^{2/k}$, we see as before the existence of $E_2',D_2'$ such that if $E_2'^{k}\leq m\leq \eta^{-k}$ then,

\begin{equation*}
    \E[\gW_2^2(\hat{\mu}_m,\mu)]\leq \frac{D'_2}{k}\log(m)\left(\frac{1}{m}\right)^{2/k}
\end{equation*}

We choose $E'_2=\sqrt{2e^{\nicefrac{-1}{4}}}$ such that $c> 1\Leftrightarrow 2\sqrt{\log(\tfrac{2}{f})}> 1\Leftrightarrow 2\sqrt{\log(\tfrac{2}{\left(\frac{1}{m}\right)^{2/k}})}> 1\Leftrightarrow m> \sqrt{2e^{\nicefrac{-1}{4}}}^k$

For the second part of the lemma, apply \Cref{Hölder} and Jensen inequality.
    
\end{proof}

\subsection{Proof of \Cref{newbiglemma}}
\label{appendix:proof_newbiglemma}

\begin{mdframed}[style=MyFrame2]
    \newbiglemma*
\end{mdframed}

\begin{proof}

We know from \Cref{new_assumption:1} the existence of a random empirical measure with $\Tilde{m}_{l+1}$ points $\hat{\mu}_{\Tilde{m}_{l+1}}$ such that $\E[\gW_2^2(\hat{\mu}_A^{\gI^\ell},\hat{\mu}^{\gI^\ell}_{\Tilde{m}_{l+1}})]\leq C_1$. Therefore by using \Cref{lemma:birkhoff} and a carefully chosen permutation, we can consider the (random) matrix $W_{\Tilde{m}_{l+1}}$ associated to $W_A^{\ell+1}$ such that $\E[\|W_{\Tilde{m}_{l+1}}^{\gI^\ell}-W_A^{\gI^\ell}\|_2^2]\leq C_1m_{\ell+1}$. Note that since $W_{\Tilde{m}_{\ell+1}}$ comes from an empirical measure with only $\Tilde{m}_{\ell+1}$ points one can denote $\gI^{\ell+1}=\{I^{\ell+1}_1,\ldots,I^{\ell+1}_{\Tilde{m}_{\ell+1}}\}$ the (random) equi-partition of $[m_{\ell+1}]$ delimiting its equal rows.
In the same way, since $A$ and $B$ have the same weights distribution, we can find a permutation $\Pi_{\ell+1}$ of the layer $\ell+1$ of  network $B$ such that denoting $\Tilde{W}^{\ell+1}_B=\Pi_{\ell+1} W_B^{\ell+1} \Pi_{\ell}^T$ and taking expectations over the choice of the weights matrices, we have $\E[\|\Tilde{W}^{\ell+1}_B-W_{\Tilde{m}_{\ell+1}}\|_2^2]\leq m_{\ell+1}C_1$. Consider $W^{\ell+1}_{M_t}=tW_A^{\ell+1}+(1-t)\Tilde{W}_B^{\ell+1}$ and denote $\ubar{\phi}^{\ell+1}(x)=\sigma(W_{\Tilde{m}_{\ell+1}}\ubar{\phi}^\ell(x)), \phi_A^{\ell+1}(x)=\sigma(W^{\ell+1}_A\phi^\ell_A(x)), \phi^{\ell+1}_B(x)=\sigma(W_B^{\ell+1}\phi^\ell_B(x)), \phi^{\ell+1}_{M_t}(x)=\sigma(W^{\ell+1}_{M_t}\phi^\ell_{M_t}(x))$. It is clear that $\forall k \in [\Tilde{m}_{\ell+1}], \forall i,j \in I^{\ell+1}_k, \ubar{\phi}^{\ell+1}_i(x)=\ubar{\phi}^{\ell+1}_j(x)$ by definition of the choice of the equi-partition $\gI^{\ell+1}$.

We will finally denote $\ubar{\phi}'^{\ell+1}(x)\in \sR^{\Tilde{m}_{\ell+1}}, \ubar{\phi}'^\ell(x)\in \sR^{\Tilde{m}_\ell}$ the vectors $\ubar{\phi}^{\ell+1}(x), \ubar{\phi}^\ell(x)$ where we have kept only one index in each of the elements of the partitions respectively $\gI^{\ell+1},\gI^{\ell}$.

Moreover using that the non-linearity is pointwise $1$-Lipschitz and $\sigma(0)=0$,

\begin{equation*}
    \E[\|\ubar{\phi}^{\ell+1}(x)\|_2^2]\leq \E_{P,Q}[\E[\|W_{\Tilde{m}_{\ell+1}}\ubar{\phi}^\ell(x)\|_2^2|\ubar{\phi}^\ell(x)]]\leq \E_{P,Q}[C_2\frac{m_{\ell+1}}{m_\ell}\|\ubar{\phi}^\ell(x)\|_2^2]\leq m_{\ell+1}C_2\ubar{E}_\ell
\end{equation*}
which yields $\ubar{E}_{\ell+1}=C_2\ubar{E}_\ell$ where we have used \Cref{new_assumption:2}.

Then

\begin{equation*}
\begin{split}
    \E_{P,Q}[\|\phi^{\ell+1}_A(x)-\ubar{\phi}^{\ell+1}(x)\|_2^2]&\leq \E_{P,Q}[\|W^{\ell+1}_A\phi^\ell_A(x)-W_{\Tilde{m}_{\ell+1}}\ubar{\phi}^\ell(x)\|_2^2\\
    &\leq 2\E_{P,Q}[\|W_A^{\ell+1}(\phi^\ell_A(x)-\ubar{\phi}^\ell(x))\|_2^2+\|(W_A^{\ell+1}-W_{\Tilde{m}_{\ell+1}})\ubar{\phi}^\ell\|_2^2]\\
    &\leq 2m_{\ell+1}C_2E_\ell+2\E_{P,Q}[\|(W_A^{\gI^\ell}-W^{\gI^\ell}_{\Tilde{m}})\ubar{\phi}'^\ell\|_2^2]\\
    &\leq 2m_{\ell+1}C_2E_\ell+2\E_{P,Q}[\|(W_A^{\gI_\ell}-W^{\gI_\ell}_{\Tilde{m}_{\ell+1}})\|_2^2]\E_{P,Q}\|\ubar{\phi}'^\ell\|_2^2]\\
    &\leq 2m_{\ell+1}C_2E_\ell+2m_{\ell+1}C_1\frac{\Tilde{m}_{\ell}}{m_\ell}\E_{P,Q}\|\ubar{\phi}^\ell(x)\|_2^2\\
    &\leq 2m_{\ell+1}C_2E_\ell+2m_{\ell+1}C_1\Tilde{m}_\ell\ubar{E}_\ell
\end{split} 
\end{equation*}

where we have used $(W^{\ell+1}_A-W_{\Tilde{m}_{\ell+1}})\ubar{\phi}^\ell(x)=(W_A^{\gI^\ell}-W^{\gI^\ell}_{\Tilde{m}_{\ell+1}})\ubar{\phi}'^\ell$ and $\|\ubar{\phi}'^\ell(x)\|_2^2=\frac{\Tilde{m}_\ell}{m_\ell}\|\ubar{\phi}^\ell(x)\|_2^2$.

We do the same computations for $\E_{P,Q}[\|\phi^{\ell+1}_B(x)-\ubar{\phi}^{\ell+1}(x)\|_2^2]$.

Finally,

\begin{equation*}
    \begin{split}
        \E_{P,Q}[\|\phi^{\ell+1}_{M_t}(x)-\ubar{\phi}^{\ell+1}(x)\|_2^2]&\leq t\E_{P,Q}[\|W^{\ell+1}_A\phi^\ell_{M_t}(x)-W_{\Tilde{m}_{\ell+1}}\ubar{\phi}^\ell(x)\|_2^2]+(1-t)\E_{P,Q}[\|W^{\ell+1}_B\phi_{M_t}(x)-W_{\Tilde{m}_{\ell+1}}\ubar{\phi}^\ell(x)\|_2^2]\\
        &\leq 2C_2E_\ell m_{\ell+1}+2C_1\Tilde{m}_\ell\ubar{E}_\ell m_{\ell+1}
    \end{split}
\end{equation*}
where we have used convexity for the first inequality and then the same proof as above for both terms.
This yields $E_{\ell+1}=2C_2E_\ell+2C_1\Tilde{m}_\ell\ubar{E}_\ell$.

\end{proof}

\subsection{Proof of \Cref{normal:newbiglemmasuccessive}}
\label{appendix:new_LMC_normal}

\begin{lemma}[Version of \Cref{new_assumption:1} for normal distribution]
\label{normal:new_assumption1}
Consider $\mu_{\ell}\in \gP_1(\sR^{m_{\ell-1}})$ a multivariate Gaussian distribution with covariance matrix $\Sigma^\ell=\Diag(\lambda^\ell_1 I_{p_{\ell-1}}, \ldots, \lambda^\ell_{\Tilde{m}_\ell}I_{p_{\ell-1}})$ where $p_{\ell-1}=\frac{m_{\ell-1}}{\Tilde{m}_{\ell-1}}$. Suppose that $\frac{1}{m_{\ell-1}}\geq \lambda^\ell_1,\geq\ldots\geq\lambda_{\Tilde{m}_{\ell-1}}^\ell$ with $\Tilde{m}_{\ell-1}\geq 5$. Then, there exists two universal constants $D_3,E_3$ such that $\forall \Tilde{m}_\ell\ge E_3^{\Tilde{m}_{\ell-1}}$ there exists a random empirical measure $\hat{\mu}_{\Tilde{m}}$ with only $\Tilde{m}_\ell$ points such that $\forall m_\ell \geq \Tilde{m}_\ell$

\begin{equation*}
    \E[\gW_2^2(\hat{\mu}^{\gI^{\ell-1}}_A, \hat{\mu}^{\gI^{\ell-1}}_{\Tilde{m}_\ell})]\leq \frac{D_3}{\Tilde{m}_{\ell-1}}\log(\Tilde{m}_\ell)\left(\frac{1}{\Tilde{m}_\ell}\right)^{\nicefrac{2}{\Tilde{m}_{\ell-1}}}
\end{equation*}
In that case:

\begin{equation*}
    \E[\gW_1(\hat{\mu}^{\gI^{\ell-1}}_A, \hat{\mu}^{\gI^{\ell-1}}_{\Tilde{m}_\ell})]\leq \frac{\sqrt{D_3}}{\sqrt{\Tilde{m}_{\ell-1}}}\sqrt{\log(\Tilde{m}_\ell)}\left(\frac{1}{\Tilde{m}_\ell}\right)^{1/\Tilde{m}_{\ell-1}}
\end{equation*}

\end{lemma}
\begin{proof}
    We know that the distribution on the rows of $W^{\gI^{\ell-1}}_A$ in $\sR^{\Tilde{m}_{\ell-1}}$ is multivariate Gaussian with covariance matrix $\frac{I_{\Tilde{m}_{\ell-1}}}{\Tilde{m}_{\ell-1}}$ since each parameters is obtained by summing the $p_{\ell-1}$ corresponding parameters of the row of $W_A^\ell$ which has covariance matrix $\frac{I_{m_{\ell-1}}}{m_{\ell-1}}$ by hypothesis.
    Therefore using \Cref{expect_wass_normal}, we know the  existence of constants $D_2,E_2$ such that if $m_\ell\geq E_2^{\Tilde{m}_{\ell-1}}$,
    $\E[\gW_2^2(\hat{\mu}^{\gI^{\ell-1}}_A, \mu_\ell^{\gI^{\ell-1}})]\leq \frac{D_2}{\Tilde{m}_{\ell-1}}\log(m_\ell)\left(\frac{1}{m_\ell}\right)^{\nicefrac{2}{\Tilde{m}_{\ell-1}}}$.

Therefore considering $\hat{\mu}_{\Tilde{m}_\ell}$ with the same law but only a fixed number $\Tilde{m}_\ell$ of elements, we get for $m_\ell\geq \Tilde{m}_\ell$:

\begin{equation*}
\begin{split}
    &\E[\gW_2^2(\hat{\mu}^{\gI^{\ell-1}}_A, \mu^{\gI^{\ell-1}}_\ell)]\leq \frac{D_2}{\Tilde{m}_{\ell-1}}\log(m_\ell)\left(\frac{1}{m_\ell}\right)^{\nicefrac{2}{\Tilde{m}_{\ell-1}}}\leq \frac{D_2}{\Tilde{m}_{\ell-1}}\log(\Tilde{m}_\ell)\left(\frac{1}{\Tilde{m}_\ell}\right)^{\nicefrac{2}{\Tilde{m}_{\ell-1}}}\\
        &\E[\gW_2^2(\hat{\mu}_{\Tilde{m}_\ell}^{\gI^{\ell-1}}, \mu_\ell^{\gI^{\ell-1}})]\leq \frac{D_2}{\Tilde{m}_{\ell-1}}\log(\Tilde{m}_\ell)\left(\frac{1}{\Tilde{m}_\ell}\right)^{\nicefrac{2}{\Tilde{m}_{\ell-1}}}
\end{split}
\end{equation*}

Indeed, the first inequality can be obtained by noticing that $\left(x\mapsto \log(x)\left(\frac{1}{x}\right)^{\nicefrac{2}{\Tilde{m}_{\ell-1}}}\right)$ is decreasing for $x \geq \sqrt{e}^{\Tilde{m}_{\ell-1}}$ and hence one can just increase the constant $E_3$ considered: taking $E_3=\max\{\sqrt{e},E_2\}$ and $D_3=4D_2$, by triangular inequality, 

\begin{equation*}
    \E[\gW_2^2(\hat{\mu}^{\gI^{\ell-1}}_A, \hat{\mu}^{\gI^{\ell-1}}_{\Tilde{m}_\ell})]\leq 2\left(\E[\gW_2^2(\hat{\mu}^{\gI^{\ell-1}}_A, \mu_\ell^{\gI^{\ell-1}})]+\E[\gW_2^2(\hat{\mu}^{\gI^{\ell-1}}_{\Tilde{m}_\ell}, \mu_\ell^{\gI^{\ell-1}})]\right)\leq \frac{D_3}{\Tilde{m}_{\ell-1}}\log(\Tilde{m}_\ell)\left(\frac{1}{\Tilde{m}_\ell}\right)^{\nicefrac{2}{\Tilde{m}_{\ell-1}}}
\end{equation*}

For the second part of the lemma, just apply \Cref{Hölder} and Jensen inequality.

\end{proof}

\begin{lemma}[Version of \Cref{new_assumption:2} for normal variable]
\label{normal:new_assumption2}
$\forall X \in \sR^{m_{\ell-1}}$ we have:

\begin{equation*}
\begin{split}
    &\E[\|W^\ell_AX\|_2^2]\leq \frac{m_\ell}{m_{\ell-1}}\|X\|_2^2\\
    &\E[\|W_{\Tilde{m}_\ell}X\|_2^2]\leq \frac{m_\ell}{m_{\ell-1}}\|X\|_2^2
\end{split}
\end{equation*}    
\end{lemma}

\begin{proof}
    First, given $X\in \sR^{m_{\ell-1}}, \forall i \in [m_\ell], (W^\ell_AX)_i=\sum_{j=1}^{m_{\ell-1}}(W^\ell_A)_{i,j}X_j$ where $(W^\ell_A)_{i,j}$ are iid following $\gN(0,\frac{1}{m_{\ell-1}})$.
    Therefore, $\forall i \in [m_\ell], \E[(W^\ell_AX)_i^2]=\frac{1}{m_{\ell-1}}\|X\|_2^2$.
    Finally,

    \begin{equation*}
\E[\|W^\ell_AX\|_2^2]=\sum_{i=1}^{m_{\ell}}\E[(W_A^\ell X)_i^2]=\frac{m_\ell}{m_{\ell-1}}\|X\|_2^2
    \end{equation*}

    Since the weight matrix associated to $\hat{\mu}_{\Tilde{m}_\ell}$ denoted  $W'_{\Tilde{m}_\ell}\in \gM_{\Tilde{m}_\ell,m_{\ell-1}}(\sR)$ has only $\Tilde{m}_\ell$ raws, we have expanded it to a matrix $W_{\Tilde{m}_\ell}\in \gM_{m_\ell,m_{\ell-1}}(\sR)$ by cloning raws in the same element of the partition $\gI^\ell$ given by the pairing between $\hat{\mu}^{\gI^\ell}_A, \hat{\mu}_{\Tilde{m}_{\ell-1}}^{\gI^\ell}$ that minimizes the Wasserstein distance.
    Since $W'_{\Tilde{m}_\ell}\in \gM_{\Tilde{m}_\ell,m_{\ell-1}}(\sR)$ built in the proof of \Cref{normal:new_assumption1} has the same law on raws $\mu_\ell$ as $W^\ell_A,W^\ell_B$ but with only $\Tilde{m}_{\ell}$ raws, we can use what preceeds to get:
\begin{equation*}
    \forall X \in \sR^{m_{\ell-1}}, \E[\|W'_{\Tilde{m}_{\ell}}X\|_2^2]\leq \frac{\Tilde{m}_\ell}{m_{\ell-1}}\|X\|_2^2
\end{equation*}

    Noting that $\|W_{\Tilde{m}_\ell}X\|_2^2=\frac{m_\ell}{\Tilde{m}_\ell}\|W'_{\Tilde{m}_\ell}X\|_2^2$, we get 

    \begin{equation*}
         \forall X \in \sR^{m_{\ell-1}}, \E[\|W_{\Tilde{m}_\ell}X\|_2^2]\leq \frac{m_\ell}{m_{\ell-1}}\|X\|_2^2
    \end{equation*}
    which concludes the proof of the lemma.
\end{proof}

Having the two assumptions we need, we can prove \Cref{normal:newbiglemmasuccessive}.

We recall it here:

\begin{mdframed}[style=MyFrame2]
    \normalnewbiglemmasuccessive*
\end{mdframed}

\begin{proof}
    From $\E_{x \sim P}[\|x\|_2^2]\leq m_0$ we get immediately \Cref{new_property} at the input layer with $\ubar{E}_0=1, E_0=0$.

    By the recursive relation of \Cref{newbiglemma} and using \Cref{normal:new_assumption1,normal:new_assumption2}, we get \Cref{new_property} at each hidden layer $\ell\in [L]$ with $\Tilde{m}_\ell$ to be chosen later with $m_\ell\geq \Tilde{m}_\ell\geq \min\{5, E_3^{\Tilde{m}_{l-1}}\}$ and:

\begin{equation*}
    \begin{cases}
        \ubar{E}_\ell=1\\
        E_\ell=\sum_{i=1}^\ell 2^{\ell+1-i}D_3\log(\Tilde{m}_i)\left(\frac{1}{\Tilde{m}_i}\right)^{\nicefrac{2}{\Tilde{m}_{i-1}}}\ubar{E}_{i-1}
    \end{cases}
\end{equation*}

Therefore, just take 
$\forall i \in [L], 
\log(\Tilde{m}_i)\left(\frac{1}{\Tilde{m}_i}\right)^{\frac{2}{\Tilde{m}_{i-1}}}\leq \epsilon^2\frac{1}{2^{L+1-i}L}$, i.e. 

\begin{equation*}
    \Tilde{m}_i=\Tilde{\gO}\left(\frac{T_i}{\epsilon}\right)^{\Tilde{m}_{i-1}}
\end{equation*}

where $T_i=\sqrt{2^{L+1-i}L}$ and the notation $\Tilde{\gO}(\cdot)$ hides logarithmic terms.

In that case,

\begin{equation*}
    \begin{cases}
        \ubar{E}_L = 1\\
        E_L = \epsilon^2
    \end{cases}
\end{equation*}

\end{proof}

\subsection{Proof of \Cref{normal:new_theorem_LMC}}
\label{appendix:new_theorem_normal_LMC}
We prove here \Cref{normal:new_theorem_LMC} that we recall:

\begin{mdframed}[style=MyFrame2]
            Under normal initialization of the weights, for $m_1\geq \Tilde{m}_1,\ldots,m_L\geq \Tilde{m}_L$ as defined in \Cref{normal:newbiglemmasuccessive}, $m_0\geq 5$, and under \Cref{loss_convex} we know that $\forall t \in [0,1]$, with $Q$-probability at least $1-\delta_{Q}$, there exists permutations of hidden layers $1,\ldots,L$ of network $B$ that are independent of $t$, such that: 
            \begin{equation*}
                \E_P\left[\gL\left(\hat{f}_{M_t}(x),y\right)\right]\leq t\E_P\left[\gL\left(\hat{f}_A(x),y\right)\right]+
                (1-t)\E_P\left[\gL\left(\hat{f}_B(x),y\right)\right]+\tfrac{4\sqrt{m_{L+1}}}{\delta_Q^2}\epsilon
            \end{equation*}
\end{mdframed}

\begin{proof}

Under assumptions of \Cref{normal:newbiglemmasuccessive}, given $A,B$, we know the existence of (random) permutations of the hidden layers $\Pi_1,\ldots,\Pi_L$ such that for $1\leq \ell \leq L$, denoting $M_t$ the mean network of weight matrix at layer $\ell$: $tW_A^\ell+(1-t)\Pi_\ell W_B^\ell \Pi_{\ell-1}^T$ we know the existence of $\ubar{\phi}^L:\sR^{m_0}\rightarrow \sR^{m_L}$ such that:

\begin{equation*}
\begin{split}
        &\E_{P,Q}[\|\phi_{M_t}^L(x)-\phi_A^L(x)\|_2^2]\leq \epsilon^2 m_L \\
        &\E_{P,Q}[\|\phi_{M_t}^L(x)-\phi_B^L(x)\|_2^2]\leq \epsilon^2 m_L
\end{split}
\end{equation*}

Then, by convexity, we get at the last layer:

\begin{equation*}
\begin{split}
        \E_{P,Q}[\|\left(tW_A^{L+1}+(1-t)W_B^{L+1}\Pi_{L}^T\right)]\phi_{M_t}^L(x)&-tW_A^{L+1}\phi_A^L(x)-(1-t)W_B^{L+1}\Pi_{L}^T\phi_B^L(x)\|_2^2]\\
        &\leq t\E[\|W_A^{L+1}(\phi_{M_t}^{L}(x)-\phi_A^L(x))\|_2^2]\\
        &+(1-t)\E[\|W_B^{L+1}\Pi_L^T(\phi_{M_t}^{L}(x)-\phi_B^L(x))\|_2^2]]\\
        &\leq \epsilon^2 m_{L+1}
\end{split}
\end{equation*}

Finally, by Jensen inequality, 

\begin{equation*}
    \E[\|\left(tW_A^{L+1}+(1-t)W_B^{L+1}\Pi_{L}^T\right)]\phi_{M_t}^L(x)-tW_A^{L+1}\phi_A^L(x)-(1-t)W_B^{L+1}\Pi_{L}^T\phi_B^L(x)\|_2]\leq \sqrt{m_{L+1}}\epsilon
\end{equation*}

Therefore, we get by applying two successive Markov lemma, that with probability at least $1-\delta_Q$ over the choice of the networks $A,B$:

\begin{equation*}
    \E_{x \sim P}[\|\left(tW_A^{L+1}+(1-t)W_B^{L+1}\Pi_{L}^T\right)]\phi_{M_t}^L(x)-tW_A^{L+1}\phi_A^L(x)-(1-t)W_B^{L+1}\Pi_{L}^T\phi_B^L(x)\|_2]\leq \frac{\sqrt{m_{L+1}}\epsilon}{\left(\tfrac{\delta_Q}{2}\right)^2}
\end{equation*}
Indeed remember that we have introduced an intermediate random measure $\hat{\mu}_{\Tilde{m}_\ell}$ which the permutations depend on and that intervenes in the expectation.

Using convexity of the loss and $1$-Lipschitzness we get for all $t\in [0,1]$ that with probability at least $1-\delta_Q$ over the choice of the networks $A,B$:

\begin{equation*}
    \E_{x\sim P}[\gL(f_{M_t}(x),y)]\leq t\E[\gL(f_A(x),y)]+(1-t)\E[\gL(f_B(x),y)]+\frac{4\sqrt{m_{L+1}}\epsilon}{\delta_Q^2}
\end{equation*}

\end{proof}

\subsection{Approximately low dimensional underlying weights distribution}

\subsubsection{Motivation on the structure of the covariance matrix}

\label{appendix:approx_low_dim_cov_matrix}

Remind that we are given a partition of the input layer of the weights $[m_{\ell-1}]$ in $\Tilde{m}_{\ell-1}$ different groups of the same size $\gI^{\ell-1}=\{I^{\ell-1}_1,...,I^{\ell-1}_{\Tilde{m}_{\ell-1}}\}$. Suppose we have already permuted the first layer of network $A$ and $B$ we can suppose that $I^{\ell-1}_1=\{1,\ldots,p_{\ell-1}\},\ldots$ where $p_{\ell-1}:=\frac{m_{\ell-1}}{\Tilde{m}_{\ell-1}}$. We want a covariance matrix that respects the fact that incoming neurons in a given group behave the same. Therefore the covariance matrix must be invariant under the permutations of indices inside a set of the equi-partition.
We will write the Kroenecker product $\otimes$.
\begin{lemma}
    To respect symmetries of the incoming layer, the covariance matrix of weights is necessarily of the form
    \begin{equation*}
        \Sigma_\ell=D_\ell \otimes I_{p_{\ell-1}}+B_\ell \otimes \mathds{1}_{p_{\ell-1}}
    \end{equation*}
    where $D_\ell\in \gM_{\Tilde{m}_{\ell-1}}(\sR)$ is diagonal, $B_\ell\in \gS_{\Tilde{m}_{\ell-1}}(\sR)$ is symmetric.
\end{lemma}

\begin{proof}
    Denoting the matrix of covariances by blocks like:
\begin{equation*}
    \Sigma_\ell=\begin{pmatrix}
A_{11} & A_{12} & ... & A_{1\Tilde{m}_{\ell-1}}\\
A_{21} & A_{22} & ... & A_{2\Tilde{m}_{\ell-1}}\\
... &...&...&...\\
A_{\Tilde{m}_{\ell-1}1} & A_{\Tilde{m}_{\ell-1}2} & ... & A_{\Tilde{m}_{\ell-1}\Tilde{m}_{\ell-1}}
\end{pmatrix}
\end{equation*}

we get the relation $\forall \Pi_1,...,\Pi_{\Tilde{m}_{\ell-1}}\in \gS_{p_{\ell-1}}$ permutation matrices, denoting $\Pi=\Diag(\Pi_1,\ldots,\Pi_{\Tilde{m}_{\ell-1}})$

\begin{equation*}
\Sigma_\ell=\E\left[XX^T\right]=\E\left[(\Pi X)(\Pi X)^T\right]=\Pi\Sigma \Pi^T=\Diag(\Pi_1,...,\Pi_{\Tilde{m}_{\ell-1}})\Sigma\Diag(\Pi_1^T,...,\Pi_{\Tilde{m}_{\ell-1}}^T)
\end{equation*}

Evaluating this relation for any $\Pi_1,\Pi_2,...,\Pi_{\Tilde{m}_{\ell-1}}$ we get $\forall \Pi_1 \in \gS_{p_{\ell-1}}, \Pi_1A_{11}\Pi_1^T=A_{11}$ and therefore $A_{11}$ is of the form $d_{11}I_{p_{\ell-1}}+b_{11}\mathds{1}_{p_{\ell-1}}$. 

We do the same for $A_{ii}, i\geq 2$.

Moreover we get for $A_{12}$ that:

\begin{equation*}
        \forall \Pi_1,\Pi_2\in \gS_{p_{\ell-1}}, \Pi_1A_{12}\Pi_2^T=A_{12}
\end{equation*}

which brings that $A_{12}$ is of the form $b_{12}\mathds{1}_{p_{\ell-1}}$.
We do the same for all $A_{ij}$ where $i\neq j$.
This concludes the proof.

Finally by summing over columns inside the partition we get:

\begin{equation*}
    \Sigma_\ell^{\gI^{\ell-1}}=p_{\ell-1}D_\ell +p_{\ell-1}^2B_\ell 
\end{equation*}

\end{proof}

The model that we chose in \Cref{section:low_dim} is a particular case that corresponds to choosing $D_\ell=\Diag(\lambda^\ell_1, ..., \lambda^\ell_{\Tilde{m}_{\ell-1}})$ and $B_\ell=0$. This is not the most general since it implies independence between weights coming from different groups but is sufficient to show the influence of low feature dimensionality on LMC efficiency.

\subsubsection{Non diagonal model}

A natural direction is to consider the case where the matrix $B_\ell$ is non zero.
Since $\Sigma_\ell^{\gI^{\ell-1}}$  is symmetric positive, one can orthogonally change the basis where it becomes diagonal. 
The arguments to prove \Cref{new_assumption:1} remains unchanged since they rely exclusively on the eigenvalues of $\Sigma_\ell^{\gI^{\ell-1}}$.

However one important point to check, is about \Cref{new_assumption:2}. Indeed having covariance between weights of a given block implies that the errors at all neurons of a given layer may sum up. Take for example $X=(1,\ldots,1)^T, D_\ell=0, B_\ell=\mathds{1}_{\Tilde{m}_{\ell-1}}$. In the general case, the constant $C_2$ in \Cref{new_assumption:2} will be a depending on the matrix $B_\ell$ and therefore potentially on the dimension. We propose a way to address this issue in \Cref{apprendix:discuss_low_dim}.

\subsection{Proof of \Cref{low_dim:newtheoremLMC}}
\label{appendix:proof_low_dim_LMC}

\begin{lemma}[Version of \Cref{new_assumption:1} for approximately low dimensional distribution]
\label{low_dim:new_assumption1}

Denote $\mu_\ell$ the law of a multivariate normal distribution of covariance matrix $\Diag(\lambda_1^\ell I_{p_{\ell-1}}, \ldots ,\lambda^\ell_{\Tilde{m}_{\ell-1}}I_{p_{\ell-1}})$ where $p_{\ell-1}=\frac{m_{\ell-1}}{\Tilde{m}_{\ell-1}}$ and $\frac{1}{k_{\ell-1}}\geq \lambda_1^\ell\ldots\geq\lambda^\ell_{\Tilde{m}_{\ell-1}}$. Let $k_{\ell-1} \geq 5$ and denote $\eta:=\frac{\sqrt{\sum_{i=k_{\ell-1}+1}^m\lambda_i^\ell}}{4\sqrt{\sum_{i=1}^{k_{\ell-1}}\lambda_i^\ell}}$.

There exists two universal constants $D'_3,E'_3$ such that $\forall \Tilde{m}_\ell\geq 1$ such that $E_3^{'k_{\ell-1}}\leq \Tilde{m}_\ell \leq \eta^{-k_{\ell-1}}$, there exists a random empirical measure $\hat{\mu}_{\Tilde{m}_\ell}$ with only $\Tilde{m}_{\ell}$ points such that $\forall m_\ell \geq 1$ such that $\Tilde{m}_\ell\leq m_\ell\leq \eta^{-k_{\ell-1}}$ we have:

\begin{equation*}
    \E[\gW_2^2(\hat{\mu}^{\gI^{\ell-1}}_A, \hat{\mu}^{\gI^{\ell-1}}_{\Tilde{m}_\ell})]\leq \frac{D'_3}{k_{\ell-1}}\log(\Tilde{m}_{\ell})\left(\frac{1}{\Tilde{m}_\ell}\right)^{\nicefrac{2}{k_{\ell-1}}}
\end{equation*}
\end{lemma}

\begin{proof}
    We do exactly the same as for the proof of \Cref{normal:new_assumption1}, i.e. a triangular inequality but now we use rate of convergence of empirical measures in Wasserstein distance with approximately low dimensional support as expressed in \Cref{expect_wass_lowdim}.
\end{proof}

\begin{lemma}[Version of \Cref{new_assumption:2} for approximately low dimensional distribution]
\label{low_dim:new_assumption2}

Denote $\mu_\ell$ the law of a multivariate normal distribution of covariance matrix $\Diag(\lambda^\ell_1I_{p_{\ell-1}}, \ldots ,\lambda^\ell_{\Tilde{m}_{\ell-1}}I_{p_{\ell-1}})$ where $p_{\ell-1}=\frac{m_{\ell-1}}{\Tilde{m}_{\ell-1}}$ and $\frac{1}{k_{\ell-1}}\geq \lambda_1^\ell\ldots\geq\lambda^\ell_{\Tilde{m}_{\ell-1}}$. $\forall X \in \sR^{m_{\ell-1}}$ we have:

\begin{equation*}
\begin{split}
    &\E[\|W^\ell_AX\|_2^2]\leq \frac{\Tilde{m}_{\ell-1}}{k_{\ell-1}}\|X\|_2^2\\
    &\E[\|W^\ell_{\Tilde{m}_{\ell}}X\|_2^2]\leq \frac{\Tilde{m}_{\ell-1}}{k_{\ell-1}}\|X\|_2^2
\end{split}
\end{equation*}    
\end{lemma}

\begin{proof}
    Just notice that $\lambda_1\leq \frac{1}{k_{\ell-1}}$ and repeat the same steps as in the proof of \Cref{normal:new_assumption2}.
\end{proof}

We will now re-state and prove \Cref{low_dim:newtheoremLMC}:

\begin{mdframed}[style=MyFrame2]
    \begin{theorem}      
    Under \Cref{approx_lowdim,loss_convex}, given $\epsilon > 0$, if $em_0\geq 5$ there exists minimal widths $\Tilde{m}_1, \ldots, \Tilde{m}_L$ such that if $\eta^{-k_0}\geq m_1\geq \Tilde{m}_1,\ldots,\eta^{-k_{L-1}}\geq m_L\geq \Tilde{m}_L$, \Cref{new_property} is verified at the last hidden layer $L$ for $\ubar{E}_L=1,E_L=\epsilon^2$.
    Moreover, $\forall \ell \in [L], \,\exists T'_\ell$ which does only depend on $L,e,\ell,$ such that one can define recursively $\Tilde{m}_\ell$ as
    \begin{equation*}
\Tilde{m}_{\ell}=\Tilde{\gO}\left(\frac{T'_\ell}{\epsilon}\right)^{k_{\ell-1}}=\Tilde{\gO}\left(\frac{T'_\ell}{\epsilon}\right)^{e\Tilde{m}_{\ell-1}}
    \end{equation*}
    where $\Tilde{m}_0=m_0$. Moreover $\forall t \in [0,1]$, with $Q$-probability at least $1-\delta_{Q}$, there exists permutations of hidden layers $1,\ldots,L$ of network $B$ s.t., 
    \begin{equation*}
\E_P\left[\gL\left(\hat{f}_{M_t}(x),y\right)\right]\leq t\E_P\left[\gL\left(\hat{f}_A(x),y\right)\right]
+(1-t)\E_P\left[\gL\left(\hat{f}_B(x),y\right)\right]+\tfrac{4\sqrt{m_{L+1}}}{\sqrt{e}\delta_Q^2}\epsilon
    \end{equation*}
    \end{theorem}
\end{mdframed}

\begin{proof}
    We just need to prove the first part of the theorem as proving the similarity of loss is exactly the same as in the proof of \Cref{normal:new_theorem_LMC} when we have proved that \Cref{new_property} holds at layer $L$ with $E_L=\epsilon^2$. The only change comes from the constant $C_2$ in \Cref{low_dim:new_assumption2} which is not $1$ anymore but $\frac{1}{e}$, hence the additional factor $e$.

To prove \Cref{new_property} at layer $L$ with $E_L=\epsilon^2$ we just combine $L$ different times the two previous lemma \Cref{low_dim:new_assumption1,low_dim:new_assumption2} and \Cref{newbiglemma}.

\Cref{low_dim:new_assumption2} brings that at layer $\ell$ the constant $C_2$ is $\frac{\Tilde{m}_{\ell-1}}{k_{\ell-1}}\leq\frac{1}{e}$ by \Cref{approx_lowdim}

Moreover, \Cref{low_dim:new_assumption1} brings that at each layer $\ell$, $C_1=\frac{D'_3}{k_{\ell-1}}\log(\Tilde{m}_\ell)\left(\frac{1}{\Tilde{m}_\ell}\right)^{2/k_{\ell-1}}$

It brings:

\begin{equation*}
    \begin{cases}
        \ubar{E}_{i+1}=\frac{1}{e}\ubar{E}_i\\
        E_{i+1}=\frac{2}{e}E_i+2\frac{D'_3}{k_{
        i}}\log(\Tilde{m}_{i+1})\left(\frac{1}{\Tilde{m}_{i+1}}\right)^{2/k_{i}}\Tilde{m}_i\ubar{E}_i
    \end{cases}
\end{equation*}

From there we see that if we have chosen at each layer 

\begin{equation*}
D'_3\log(\Tilde{m}_{i+1})\left(\frac{1}{\Tilde{m}_{i+1}}\right)^{2/k_{i}}= \frac{\sqrt{\frac{L2^{L-i+1}}{e^L}}}{\epsilon^2}
\end{equation*}

if 

\begin{equation*}
    \Tilde{m}_i=\Tilde{\gO}\left(\frac{T'_i}{\epsilon}\right)^{e\Tilde{m}_{i-1}}
\end{equation*}
where $\Tilde{\gO}(\cdot)$ hides logarithmic terms and we define $T'_i=\sqrt{\frac{D'_3e^L}{L2^{L-i+1}}}$

\end{proof}

\subsection{Proof of \Cref{theorem:lower_bound}}
\label{appendix:lower_bound}

\begin{mdframed}[style=MyFrame2]
    \theoremlowerbound*
\end{mdframed}

\begin{proof}
   First since $\Sigma$ is full rank we see that when writing $\Sigma=ODO^T$ where $OO^T=I_{\Tilde{n}}$, the problem is equivalent to consider the matrices $W_AO, W_BO$ and $\Sigma=D$. In that case, the raws of $W_A,W_B$ are still i.i.d. and follow the same law as $\mu$ modulo a non-degenerated dilatation. We can then still assume $\frac{d\mu}{d\Leb}\leq F_1$ for a certain constant $F_1$.
    
    Let $\tau=\frac{1}{2}$. $\forall S \subset \sR^n$ a Borel set such that $\mu(S)\geq 1-\tau$, we know that $\Leb(S)\geq \frac{1-\tau}{F_1}$.
    In that case, applying \Cref{lemma:packing2} we get $\gN_\epsilon(S)\geq \left(\frac{1}{\epsilon}\right)^n\frac{\frac{1-\tau}{F_1}}{\Leb(\gB_2(0,1)}$.
Denote $F_2=\left(\frac{\frac{1-\tau}{F_1}}{\Leb(\gB_2(0,\epsilon)}\right)^{\nicefrac{1}{n}}$. Using notations of \citet{weed2019sharp} we get $\gN_\epsilon(\mu,\tau)\geq \left(\frac{F_2}{\epsilon}\right)^{n}$.

Applying Proposition $6$ from \citet{weed2019sharp} we get that 

\begin{equation*}
    \gW_2^2(\hat{\mu}_A, \mu)\geq F_3\left(\frac{1}{m}\right)^{\nicefrac{2}{n}}
\end{equation*}

Finally noticing that $\E_{W_B}[\hat{\mu}_B]=\mu$ and applying \Cref{lemma:villani}, we get that:

\begin{equation*}
    \E_{W_A,W_B}[\gW_2^2(\hat{\mu}_A,\hat{\mu}_B)] \geq \E_{W_A}[\gW_2^2(\hat{\mu}_A,\E_{W_B}[\hat{\mu}_B])]\geq  F_3\left(\frac{1}{m}\right)^{\nicefrac{2}{n}}
\end{equation*}

Finally, remark that sinc $\Sigma=D=\Diag(\lambda_1,\ldots,\lambda_{\Tilde{n}})$ and noting $\lambda(\Sigma)=\min\{d_i, 1\leq i\leq n\}> 0$ the smallest eigenvalue of $\Sigma$,

\begin{equation*}
    \begin{split}
        \E_{W_A,W_B}[\min_{\Pi\in \gS_m}\E_{x\sim P}\|(W_A-\Pi W_B)x\|_2^2]&\leq \E_{W_A,W_B}[\min_{\Pi\in \gS_m}\E_{x\sim P}\tr((W_A-\Pi W_B)^T(W_A-\Pi W_B)D)]\\
        &\geq \lambda(\Sigma)\E_{W_A,W_B}[\min_{\Pi\in \gS_m}\E_{x\sim P}\|(W_A-\Pi W_B)\|_2^2]\\
        &\geq \lambda(\Sigma)F_3\left(\frac{1}{m}\right)^{\nicefrac{2}{n}}
    \end{split}
\end{equation*}
\end{proof}

\subsection{Discussion about a model with no growth in the width needed}
\label{apprendix:discuss_low_dim}

For proving \Cref{low_dim:newtheoremLMC}, we have used constants $C_1,C_2$ in \Cref{newbiglemma} given by \Cref{low_dim:new_assumption1,low_dim:new_assumption2}.
However we would like to highlight that \Cref{low_dim:new_assumption2} is very sub-optimal, though it can not really be improved in the general case. Indeed, $\frac{1}{e^\ell}$ grows to infinity while $\E[\|\phi^\ell(x)\|_2^2]$ is supposed to remain bounded for $1\leq \ell$.
Therefore from now on suppose that we have the following version of \Cref{low_dim_no_growth:new_assumption2}:
\begin{lemma}[Extending \Cref{new_assumption:2} for approximately low dimensional distribution]
\label{low_dim_no_growth:new_assumption2}

Denote $\mu_\ell$ the law of a multivariate normal distribution of covariance matrix $\Diag(\lambda_1^\ell I_{p_{\ell-1}}, \ldots ,\lambda^\ell_{\Tilde{m}_{\ell-1}}I_{p_{\ell-1}})$ where $p_{\ell-1}=\frac{m_{\ell-1}}{\Tilde{m}_{\ell-1}}$ and $\frac{1}{k_{\ell-1}}\geq \lambda_1^1\ldots\geq\lambda_{\Tilde{m}_{\ell-1}}^\ell$. Suppose we have:

\begin{equation*}
\begin{split}
    \E[\|W_A(\phi^{\ell-1}_B(x)-\phi_A^{\ell-1}(x)\|_2^2]\leq  \frac{m_\ell}{m_{\ell-1}}\E[\|\phi^{\ell-1}_B(x)-\phi_A^{\ell-1}(x)]\|_2^2
\end{split}
\end{equation*}    
\end{lemma}

Notice that it is equivalent to making an assumption on the distribution of $\phi^{\ell-1}_B(x)-\phi_A^{\ell-1}(x)$ which must not put too much mass on the worst case coordinates $(\lambda^\ell_i$ for $i$ small).

In that case, adapting the proof as in \Cref{low_dim:newtheoremLMC}, we could get an inequality on $\Tilde{m}_i$ of the form

\begin{equation*}
    \Tilde{m}_i = \Tilde{\gO}\left(\frac{T''_i}{\epsilon}\right)^{k_{i-1}}
\end{equation*}
where $T''_i$ is independent of $e$ which lead to controlled bounds as $L \rightarrow \infty $ (without the exponent $e\Tilde{m}_i$ as in \Cref{low_dim:newtheoremLMC}).

\subsection{Extension of \Cref{normal:new_theorem_LMC,low_dim:newtheoremLMC} to sub-Gaussian variables}
\label{appendix:new_subgaussian}

Still under the setting of~\Cref{new_assumption:0}, suppose now that at a given layer $\ell$, all the parameters of $W_A^\ell$ are still drawn independently but no longer from $\gN(0,\frac{1}{m_{\ell-1}})$ Instead we assume that the underlying distribution $\mu_\ell$ verifies for each layer $\ell\in [L+1]$: if $X\sim \mu_\ell$ then, $\forall j\neq k \in [m_{l-1}], X_j \amalg X_k$. Moreover $\forall i \in [\Tilde{m}_{\ell-1}], \forall j,k \in I^{\ell-1}_i$,
\begin{equation*}
     \E[X_j^2]=\E[X_k^2]=\lambda_i^\ell
\end{equation*}
Finally suppose the variables are sub-Gaussian i.e., $\exists K > 0, \forall i \in [\Tilde{m}_{\ell-1}], \forall j\in I^{\ell-1}_i,\, \forall c > 0$,
\begin{equation*}
    \sP(|X_j|\geq c)\leq 2\exp(-\frac{c^2}{K\lambda_i^\ell})
\end{equation*}

Further suppose that we are in the setting of \Cref{normal:new_theorem_LMC} (The case of \Cref{low_dim:newtheoremLMC} is treated similarly): $\tfrac{1}{m_{\ell-1}}\geq \lambda_1^\ell\geq\ldots\lambda_{\Tilde{m}_{\ell-1}^\ell}$.

It is clear that \Cref{normal:new_assumption2} is still valid for a constant $C_2=1$, the proof being exactly the same.

We therefore just need to prove \Cref{normal:new_assumption1} for $C_1$ to be determined.
To prove \Cref{normal:new_assumption1}, one just needs an equivalent of \Cref{expect_wass_normal} for sub-Gaussian variables. To prove \Cref{expect_wass_normal}, recall that we have used the fact that a normal distribution doesn't put to much mass outside of a ball of radius $c$ when c grows logarithmically. More precisely we have used the property, that if $X\sim \gN(0,\frac{I_{m_{l-1}}}{m_{l-1}})$, then:

\begin{equation*}
    \forall c > 1, \sP(\|X\|_2^2\geq c^2)\leq e^{-\frac{c^2}{4}}
\end{equation*}

In our case, for a sub-Gaussian distribution, we know the existence of a constant $K$ such that $\forall c > 0$:

\begin{equation*}
    \sP(|X_j|\geq c)\leq 2\exp(-\frac{c^2}{K\lambda_i^\ell})
\end{equation*}

Therefore, by plugging this into the proof, and scaling parameter $c$ by $\sqrt{K}$, we get exactly the same version of \Cref{normal:new_assumption1} for sub-Gaussian variable, with different constants scaled by a factor $\sqrt{K}$.

Propagating \Cref{new_property} with 
the recurrence formula of \Cref{newbiglemma} we get LMC for networks with sub-Gaussian distributions in the same form as for normal variables.

\paragraph{Remark} Finally notice that results for sub-Gaussian variables can be extended in the same way to variables whose tail decreses sufficiently fast (exponentially, polynomially, etc...). The assymptotics of the tail will affect the convergence rate in Wasserstein of the corresponding empirical measure.

\subsection{Link with dropout stability}
\label{appendix:dropout_stability}

We relate now our previous study to a line of work exploring mode connectivity through dropout stability.

\citet{kuditipudi2019explaining} define $\epsilon$-dropout stable networks, as networks $\hat{f}(\cdot;\theta)$ as defined in \Cref{eq:network} for which there exists in each layer $\ell \in [L]$, a subset of at most $\frac{m_\ell}{2}$ of neurons (i.e., rows of the weight matrix $W^\ell$) such that after renormalizing each layer, the expected loss of the new network increases by no more than $\epsilon$ with respect to the original loss. \citet{kuditipudi2019explaining} shows that two $\epsilon$-dropout stable networks are mode connected (with error barrier height $\epsilon$) and \citet{shevchenko2020landscape} uses this result to show that two wide enough two-layer neural networks trained with SGD are mode connected (where the continuous path may be non-linear).

Recall that we have shown in \Cref{section:mean_field} the stronger statements that two such networks are in the same local minima modulo permutation symmetries. However, note that \citet{shevchenko2020landscape} don't allow permutations of neurons). We discuss here how to embrace in the same view our framework with dropout stability results, showing how networks with independent neuron's weights become dropout stable in the same asymptotics of large width than the condition of \Cref{normal:newbiglemmasuccessive}. 

Consider the simplified setting of a $1$-hidden layer neural network with $1$-Lipschitz activation where the weights of the second layer are fixed to $\frac{1}{N}$: $\hat{f}(x;\theta)=\frac{1}{N}\sum_{i=1}^N\sigma(w_ix)$
where $w_i=W_{i:}\in \sR^d$ is the $i-th$ row of the weight matrix $W$. Suppose that $w_i$ are sampled independently from a sub-Gaussian distribution and the data follows a distribution $(x,y)\sim P$ with $\Supp(P)\subset \gB_2(0,1)$. Denote $\gA=\left[\frac{N}{2}\right]$. Dropout stability can be quantified by controlling the error between the correctly renormalized network with weights in $\gA$ and the original one,

\begin{equation}
\label{eq:dropout}
        \E\left[\left|\frac{2}{N}\sum_{i\in \gA}\sigma(w_ix)-\frac{1}{N}\sum_{i=1}^N\sigma(w_ix)\right|\right]\leq \gW_1(W^\gA,W^{\gA^c})
\end{equation}

where we have denoted $W^\gA$ (respectively $W^{\gA^c}$) the matrix $W$ where we have kept only the rows in $\gA$ (respectively ($\gA^c$)). 
The right hand term can be connected to convergence rates of empirical measure (\Cref{normal:new_assumption1} and the extension to sub-Gaussian distribution discussed in \Cref{appendix:new_subgaussian}):
\begin{equation*}
\begin{split}
         \gW_1(W^\gA,W^{\gA^c})\approx \left(\frac{1}{N}\right)^{1/d}
\end{split}
\end{equation*}

In a nutshell, showing that previous \Cref{eq:dropout} is tight would provide a formal connection between dropout stability and our results. It is an interesting direction for future work and note that it has strong connections with the dual expression of the Wasserstein $1$ distance.

In that case, the bound on the dropout error evolves as $\left(\frac{1}{N}\right)^{1/d}$, as for the linear mode connectivity error. Hence networks become dropout stable in the same asymptotics as to exhibit linear mode connectivity.

This is consistent with the idea that LMC requires the information to be distributed evenly among neurons without any neuron responsible for the particular behavior of one layer. This is similar to the intuitive requirement for dropout stability.

\section{PROOF OF LMC FOR TWO-LAYER NEURAL NETWORKS IN THE MEAN FIELD REGIME}
\label{appendix:mean_field}

\subsection{Description of the mean field regime}

\label{appendix:background_mean_field}

When training a two-layer neural network with fixed input and output dimensions but with a very wide hidden layer using SGD, the parameters of each neuron can be seen as particles evolving independently one from each other: the dynamic of each neuron's weights depends only on the average distribution of the weights and itself.

The main object of study is therefore the empirical distribution of the neurons weights in the intermediate layer after $k$ Stochastic Gradient Descent (SGD) steps. We denote it $\rho_k^{(N)}=\frac{1}{N}\sum_{i=1}^N\delta_{\theta^k_i}$ where $\theta^k_i=(w^k_i,a^k_i)\in \sR^{d+1}$.

Multiple works \citep{chizat2018global,mei2018mean,mei2019mean} show that if the hidden layer's width $N$ is big, the learning rate $s_k$ is small and setting $T=\sum_{i=1}^{k}s_i$ a time re-normalization after $k$ steps, then $\rho_k^{(N)}$ can be well approximated by $\rho_t$ which follows the following partial differential equation (PDE):

    \begin{align*}
\partial_t\rho_t=2\xi(t)\nabla_\theta \cdot(\rho_t\nabla_\theta\Psi(\theta;\rho_t)), \text{           }& \Psi(\theta;\rho_t)=V(\theta)+\int U(\theta, \Tilde{\theta})\rho_t(d\Tilde{\theta})\\
    V(\theta) = -\sE[y\sigma_*(x; \theta)],   \text{          }&  U(\theta_1,\theta_2)=\sE[\sigma_*(x;\theta_1)\sigma_*(x;\theta_2)]       
    \end{align*}
Here $\xi(t)$ represent a scaling of the learning rate where $s_k=\epsilon\xi(k\epsilon)$. $U(\theta_1, \theta_2)$ represents a correlation between neurons. $V(\theta)$ is an energy quantifying the alignment of a neuron function with the data. In the following, as in \citet{mei2019mean} we will work with $\xi=\frac{1}{2}$ a constant step size function. As in \citet{mei2019mean}, we highlight that the proof remains valid under \Cref{ass:noiseless_SGD}.

When considering noisy SGD, the limit PDE becomes:

\begin{align*}
&\partial_t\rho_t=2\xi(t)\nabla_\theta \cdot(\rho_t(\theta)\nabla_\theta\Psi_\lambda(\theta;\rho_t))+2\xi(t)\tau d^{-1}\Delta_\theta\rho_t\\
&\Psi_\lambda(\theta;\rho)=\Psi(\theta;\rho)+\frac{\lambda}{2}\|\theta\|_2^2
\end{align*}

The crucial point about the mean field view is to show that the empirical distribution of parameters is well enough approximated by $\rho_t$. Then the study of the neural network can be reduced to the study of the partial differential equation. For example global convergence of the test loss results can be deduced as in \citet{chizat2018global}. This view is also convenient to get insights of typical behaviours of the dynamics while smoothing the effects of local minima \citep{mei2019mean}. In our case, the mean field view allows us to use convergence results in Wasserstein distance of empirical measures towards the underlying distribution. We can then show Linear Mode Connectivity for two-layer neural networks independently trained in the mean field regime.

\subsection{Proving LMC for noiseless regularization-free SGD}
To prove our results we have to show that the empirical distribution of weights can be well approximated by the solution of the mean field PDE. To achieve this, \citet{mei2019mean} introduce four intermediate dynamics that stay close one of each other.

First note that our $\Cref{ass:noiseless_SGD}$ implies Assumptions $1$ to $4$ of \citet{mei2019mean}. Especially the non-linearity being Lipschitz implies its gradient distribution on the data is bounded and hence sub-Gaussian.

\subsubsection{Intermediate dynamics}

\citet{mei2019mean} introduce 4 different intermediate dynamics between the empirical distribution of weights optimized by SGD and the solution of the PDE that we recall here:

\textbf{Nonlinear dynamics}

Let consider $\bar{\theta}_i^t$ with initialization $\bar{\theta}_i^0\sim \rho_0$ i.i.d. and which follows the dynamics 

\begin{equation*}
\bar{\theta}_i^t=\bar{\theta}_i^0+2\int_0^t\xi(s)G(\bar{\theta}_i^s;\rho_s)ds 
\end{equation*}
or equivalently 

\begin{equation*}
    \frac{d}{dt}\Bar{\theta_i^t}=-2\xi(t)\left[\nabla V(\Bar{\theta}_i^t)+\int \nabla_1 U(\bar{\theta}_i^t,\theta)\rho_t(d\theta)\right]
\end{equation*}

where $G(\theta;\rho)=-\nabla\Psi(\theta;\rho)$.
An important fact is that $\bar{\theta}_i^t$ is random because of the random initialization. Moreover its law at time $t$ is $\rho_t$. It corresponds to the evolution of particles under a velocity field $-2\xi(t)\left[\nabla V(\Bar{\theta}_i^t)+\int \nabla_1 U(\bar{\theta}_i^t,\theta)\rho_t(d\theta)\right]$ which depends only on the position of the optimized particle and the overall distribution of all particles.

\textbf{Particle Dynamics}

Let $\ubar{\theta}_i^t$ have the same initialization as the nonlinear dynamics $\ubar{\theta}_i^0=\bar{\theta}_i^0$, and $\ubar{\rho}_t^{(N)}=\frac{1}{N}\sum_{i=1}^N\delta_{\ubar{\theta}_i^t}$ denote the empirical distribution of $\ubar{\theta}_i^t$. Then the particle dynamics is given by:

\begin{equation*}   \ubar{\theta}_i^t=\ubar{\theta}_i^0+2\int_0^t\xi(s)G(\ubar{\theta}_i^s;\ubar{\rho}_s^{(N)})ds 
\end{equation*}

or equivalently 

\begin{equation*}
    \frac{d}{dt}\ubar{\theta}_i^t=-2\xi(t)\left[\nabla V(\ubar{\theta}_i^t)+\frac{1}{N}\sum_{j=1}^N \nabla_1 U(\ubar{\theta}_i^t,\ubar{\theta}_j^t)\right]
\end{equation*}

\textbf{Gradient descent dynamics}

Let $\Tilde{\theta}_i^{k}$ with initialization $\Tilde{\theta}_i^0=\bar{\theta}_i^0$ following the dynamics:

\begin{equation*}    \Tilde{\theta}_i^{k}=\Tilde{\theta}_i^0+2\epsilon\sum_{l=0}^{k-1}\xi(l\epsilon)G(\Tilde{\theta}_i^l;\tilde{\rho}_l^{(N)})
\end{equation*}

or equivalently:

\begin{equation*}
\Tilde{\theta}_i^{k+1}=\Tilde{\theta}_i^k-2s_k\left[\nabla V(\tilde{\theta}_i^k)+\frac{1}{N}\sum_{j=1}^N\nabla_1 U(\tilde{\theta}_i^k, \tilde{\theta}_j^k)\right]
\end{equation*}

\textbf{Stochastic Gradient Descent Dynamics}

Consider $\theta_i^k$ with initialization $\theta_i^0=\bar{\theta}_i^0$ following the dynamics:

\begin{equation*}   \theta_i^k=\theta_i^0+2\epsilon\sum_{l=0}^{k-1}\xi(l\epsilon)F_i(\theta^l;z_{l+1})
\end{equation*}

or equivalently

\begin{equation*}
    \theta_i^{k+1}=\theta_i^k-2s_kF_i(\theta^k;z_{k+1})
\end{equation*}
where $F_i(\theta^k;z_{k+1})=(y_{k+1}-\hat{y}_{k+1})\nabla_\theta\sigma_*(x_{k+1};\theta_i^k)$, $z_k=(x_k,y_k)$ and $\hat{y}_{k+1}=\frac{1}{N}\sum_{j=1}^N\sigma_*(x_{k+1};\theta_j^k)$.

One can use Proposition 26,28,29 from \citet{mei2019mean} to show the following lemma:

\begin{lemma}
\label{lemma:mean_field_closeness} Consider a two-layer neural network trained with noiseless regularization-free SGD for an underlying time $T$. Then under \Cref{ass:noiseless_SGD}, there exists constants $K$ and $K_0$ such that, if $\epsilon \leq \min\left\{\frac{1}{K_0e^{K_0(1+T)^3}}, \frac{1}{K_0(d+\log(N)+z^2)e^{K_0(1+T)^3}} \right\}$, then with probability at least $1-3e^{-z^2}$ we have:

    \begin{equation*}
        \begin{split}
            \max_{k\in [0,T/\epsilon]\bigcap \sN}\max_{i\in [N]}\|\theta_i^k-\bar{\theta}_i^{k\epsilon}\|_2&\leq Ke^{K(1+T)^3}\frac{1}{\sqrt{N}}[\sqrt{\log(NT)}+z]\\
            &+Ke^{K(1+T)^2T}\epsilon+Ke^{K(1+T)^2T}\sqrt{\epsilon}[\sqrt{d+\log(N)}+z]
        \end{split}
    \end{equation*}
\end{lemma}

\begin{proof}
    This is direct application from Proposition 26,28,29 from \citet{mei2019mean} by doing two union bounds and two triangular inequalities.
\end{proof}

We moreover recall that as highlighted in \citet{mei2019mean}, $\{\bar{\theta}_i^{t},1\leq i\leq N\}$ are independent from each other and each follows the distribution $\rho_{t}$ when initialized i.i.d. as $\rho_0$.
Therefore, when considering two two-layer neural networks initialized randomly as $\rho_0$ and trained for the same underlying time $T$ with noiseless regularization-free SGD, we know from the previous lemma that the parameters of both networks are close to two samples from $\rho_t$.

\subsubsection{Proof of \Cref{theorem:mean_field} in the case of noiseless regularization-free SGD}

We now prove \Cref{theorem:mean_field} in case of noiseless regularization-free SGD.

\begin{mdframed}[style=MyFrame2]
\theoremmeanfield*
\end{mdframed}

\begin{proof}    
    We know from \Cref{lemma:mean_field_closeness} that with probability at least $1-3e^{-z^2}$, if $\epsilon \leq \min\left\{\frac{1}{K_0e^{K_0(1+T)^3}}, \frac{1}{K_0(d+\log(N)+z^2)e^{K_0(1+T)^3}} \right\}$,

    \begin{equation*}
        \begin{split}
            \max_{k\in [0,T/\epsilon]\bigcap \sN}\max_{i\in [N]}\|\theta_{A,i}^k-\bar{\theta}_{A,i}^{k\epsilon}\|_2&\leq Ke^{K(1+T)^3}\frac{1}{\sqrt{N}}[\sqrt{\log(NT)}+z]\\
            &+Ke^{K(1+T)^2T}\epsilon+Ke^{K(1+T)^2T}\sqrt{\epsilon}[\sqrt{d+\log(N)}+z]
        \end{split}
    \end{equation*}
    which means that $\theta_{A,i}$ is close to the non linear dynamics which are samples from $\rho_t$.
    By a union bound, with probability $1-6e^{-z^2}$ this is true for both networks $A$ and $B$.
    
    Denoting as before $\theta_{A,i}=(w_{A,i},a_{A,i})\in \sR^{d+1}$ (respectively $\theta_{B,i}$), $A_A=(a_{A,1},\ldots,a_{A,N})$ (respectively $A_B$) and $W_A\in \gM_{N,d}(\sR)$ the concatenation of vectors $w_{A,i}\in \sR^d$ (respectively $W_B$). Given $t \in [0,1]$, we aim at finding a permutation $\Pi \in \gS_N$ of the second network's hidden layer to get $\tilde{\theta}_B=(\Tilde{A}_B, \Tilde{W}_B)=(A_{B}\Pi^T, \Pi W_B)$ bounding
    \begin{align}
        & |t\hat{f}(x;\theta_{A})+(1-t)\hat{f}(x;\theta_{B})-\hat{f}(x;t\theta_{A}+(1-t)\tilde{\theta}_B)| \nonumber \\
        & =\frac{1}{N}|tA_A\sigma(W_AX)+(1-t)\tilde{A}_B\sigma(\tilde{W}_BX)-(tA_A+(1-t)\tilde{A}_B)\sigma((tW_A+(1-t)\Tilde{W}_B)X)| \\
        & \leq t \left|\frac{A_A}{N}(\sigma(W_AX)-\sigma((tW_A+(1-t)\Tilde{W}_B)X))\right| +(1-t)\left|\frac{\Tilde{A}_B}{N}(\sigma(\tilde{W}_BX)-\sigma((tW_A+(1-t)\Tilde{W}_B)X))\right| \nonumber \\
        &\leq t\|A_A\|_\infty\frac{\|\sigma(W_AX)-\sigma((tW_A+(1-t)\Tilde{W}_B)X)\|_1}{N}+(1-t)\|\tilde{A}_B\|_\infty\frac{\|\sigma(\tilde{W}_BX)-\sigma((tW_A+(1-t)\Tilde{W}_B)X)\|_1}{N} \nonumber
    \end{align}

Both terms $\frac{\|\sigma(W_AX)-\sigma((tW_A+(1-t)\Tilde{W}_B)X)\|_1}{N}$ and $\frac{\|\sigma(\tilde{W}_BX)-\sigma((tW_A+(1-t)\Tilde{W}_B)X)\|_1}{N}$ can be bounded.

Indeed, first using lemma 22 from \citet{mei2019mean} and that from \Cref{ass:noiseless_SGD} $\Supp(\rho_0)$ is bounded, we get that \begin{equation*}
\label{eq:bounded}
\Supp(\rho_t)\subset \gB_2(0,K((1+T^2)T)+1)
\end{equation*}
is bounded with a diameter depending only on the initialization $\Supp(\rho_0)$ and underlying time $T$.

Therefore we can apply Theorem 1 from \citet{weed2019sharp} and get for $s> d$ the existence of a constant $C$ such that with probability at least $1-\frac{\delta}{2}$, there exists a permutation $\gamma \in \gS_N$ such that by considering $\|\cdot\|_1$ as a distance for the Wasserstein:

\begin{equation*}
    \gW_1(\hat{\mu}_A,\hat{\mu}_B)=\frac{1}{N}\sum_{i=1}^N\|\bar{\theta}_{A,i}-\Bar{\theta}_{B,\gamma_i}\|_1\leq \frac{C}{\delta}N^{-1/s}
\end{equation*}

Note that, while $C$ is independent of $N$, it depends on the distribution $\rho_t$ and therefore on $d$, $\diam(\Supp(\rho_t))$ (i.e.,\@ $T$) and on the constants from \Cref{ass:noiseless_SGD}.

Recall that we suppose the data distribution P bounded and denote $\Supp(P)\subset [-H_x,H_x]^d\times [-H_y,H_y]$.

Therefore, we get that with probability at least $1-\frac{\delta}{2}-6e^{-z^2}$:

\begin{align}
    & \forall X \in [-H_x,H_x]^d, \frac{\|\sigma(W_{A}X)-\sigma((tW_{A}+(1-t)\Tilde{W}_{B})X)\|_1}{N} \leq g_2(T,z,\delta,N,\epsilon)\\
    & \eqdef L_\sigma\Bigg(H_x\frac{C}{\delta}N^{-\frac{1}{s}} +2H_x\sqrt{d}\bigg[\frac{Ke^{K(1+T)^3}}{\sqrt{N}}[\sqrt{\log(NT)}+z] +Ke^{K(1+T)^2T}\epsilon +Ke^{K(1+T)^2T}\sqrt{\epsilon}[\sqrt{d+\log(N)}+z]\bigg]\Bigg) \nonumber
\end{align}

and same for the other term:

\begin{equation*}
\begin{split}
    \forall X \in [-H_x,H_x]^d, \frac{\|\sigma(\Tilde{W}_{B}X)-\sigma((tW_{A}+(1-t)\Tilde{W}_{B})X)\|_1}{N}\leq g_2(T,z,\delta,N,\epsilon)
\end{split}
\end{equation*}

Using lemma 20 from \citet{mei2019mean} we know that $\forall i \in [N], \bar{a}_i^T\leq K(1+T)$ for a certain constant $K$. Therefore we can bound $\|A_A\|_{\infty}, \|A_B\|_{\infty}\leq g_1(T,z,N,\epsilon):=K(1+T)+Ke^{K(1+T)^3}\frac{1}{\sqrt{N}}[\sqrt{\log(NT)}+z]
+Ke^{K(1+T)^2T}\epsilon
+Ke^{K(1+T)^2T}\sqrt{\epsilon}(\sqrt{d+\log(N)}+z)$.

Taking $z=\sqrt{\log\left(\frac{12}{\delta}\right)}$, we have shown the existence of a permutation $\gamma$ with probability at least $1-\delta$ such that almost surely on the choice of $x\sim P$ and $\forall t \in [0,1]$, we have:

\begin{equation*}
    \begin{split}
        |t\hat{f}(x&;\theta_{A})+(1-t)\hat{f}(x;\theta_{B})-\hat{f}(x;t\theta_{A}+(1-t)\tilde{\theta}_B)|\leq g_1(T,z,N,\epsilon)g_2(T,z,\delta,N,\epsilon)\\
        &\leq \left(K(1+T)+Ke^{K(1+T)^3}\frac{1}{\sqrt{N}}[\sqrt{\log(NT)}+z]
+Ke^{K(1+T)^2T}\epsilon
+Ke^{K(1+T)^2T}\sqrt{\epsilon}(\sqrt{d+\log(N)}+z)\right)\\
&\Big(L_\sigma(H_x\frac{C}{\delta}N^{-\frac{1}{s}}+2H_x\sqrt{d}(Ke^{K(1+T)^3}\frac{1}{\sqrt{N}}[\sqrt{\log(NT)}+z]\\
&+Ke^{K(1+T)^2T}\epsilon
    +Ke^{K(1+T)^2T}\sqrt{\epsilon}(\sqrt{d+\log(N)}+z)))\Big)
    \end{split}
\end{equation*}

For fixed $T,\delta$ and $z=\sqrt{\log\left(\frac{12}{\delta}\right)}$, denote $\err(N,\epsilon)$ the right hand term.

It is clear that:

\begin{equation*}
\begin{split}
       \forall \err> 0 \exists N_{\min}\forall N\geq N_{\min}\exists \epsilon_{\max}(N)\forall \epsilon\leq \epsilon_{\max}(N), \,
       \err(N,\epsilon)\leq \err
\end{split}
\end{equation*}

This brings the first part of the theorem.

\textbf{Discussion}:
Let's look more closely at the term $Ke^{K(1+T)^3}\frac{1}{\sqrt{N}}[\sqrt{\log(NT)}+z]+Ke^{K(1+T)^2T}\epsilon
        +Ke^{K(1+T)^2T}\sqrt{\epsilon}(\sqrt{d+\log(N)}+z))$ from \citet{mei2019mean} which comes from the mean field approximation. When taken alone, this term yields an error which is independent of the input dimension $d$, since taking $N$ large leads to small error (provided $\epsilon$ is small). However, here the growth of the hidden layer $N$ depends on the input dimension $d$ through the exponent $s> d^*_1(\mu)$ (with $d^*_1(\mu)\leq d$)  using notations from \citet{weed2019sharp}. This is due to Wasserstein convergence rates of empirical measures in dimension $d$. Without any further assumption on the weight distribution or precise study of the PDE we have to consider that $\Supp(\mu))$ has dimension $d$. To remove this dependence, one could study a precise model for the data and look more closely at the PDE evolution to better understand the support of the distribution $\rho_t$.
\end{proof}

To prove the second part of the theorem, first notice that as already mentioned, $\|A_A\|_{\infty}, \|A_B\|_{\infty}\leq g_1(T,z,N,\epsilon)$. Moreover, the data distribution is bounded and the weights of the first layer of the approximating PDE live in a bounded set thanks to \Cref{eq:bounded}.
It brings the existence of a constant $K(T)$ such that:
\begin{equation*}
    \forall t \leq T, \forall w \in \Supp(\rho(t)), \forall X \in \Supp(P), |w.x| \leq K(T)
\end{equation*}

Using we see that if $\epsilon \leq \min\{\frac{1}{K_0e^{K_0(1+T)^3}}, \frac{1}{K_0(d+\log(N)+z^2)e^{K_0(1+T)^3}} \}$, with probability at least $1-\frac{\delta}{2}$, setting $z=\sqrt{\tfrac{12}{\delta}}$

    \begin{equation*}
        \begin{split}
            \max_{k\in [0,T/\epsilon]\bigcap \sN}&\max_{i\in [N]}|w_{A,i}^kx|\leq H_x\sqrt{d}(Ke^{K(1+T)^3}\\
            &\frac{1}{\sqrt{N}}[\sqrt{\log(NT)}+z]      +Ke^{K(1+T)^2T}\epsilon+Ke^{K(1+T)^2T}\sqrt{\epsilon}(\sqrt{d+\log(N)}+z))+K(T)
        \end{split}
    \end{equation*}

Up to changing previous $N_{\min}, \epsilon_{\max}$, we can suppose that $\forall N\geq N_{\min} , \forall \epsilon \leq \epsilon_{\max}(N), \, Ke^{K(1+T)^3}\frac{1}{\sqrt{N}}[\sqrt{\log(NT)}+z]      +Ke^{K(1+T)^2T}\epsilon+Ke^{K(1+T)^2T}\sqrt{\epsilon}(\sqrt{D+\log(N)}+z)\leq 1$

Therefore, since the data distribution is bounded we know that there exists a constant $K(T)$ such that if $N\geq N_{\min}, \epsilon\leq \epsilon_{\max}(N), \, P-\text{almost-surely}$:

\begin{equation*}
    \begin{cases}
    |y|\leq K(T)\\
    |\hat{f}(x:\theta_A)|\leq K(T)\\
    |\hat{f}(x:\theta_B)|\leq K(T)\\
    \end{cases}
\end{equation*}

We make a general assumption on the loss of the form: $\forall y \in \sR (x \rightarrow  \gL(x,y))$ is convex and $\forall K > 0, \exists C_K, \forall x_1,x_2,y\in [-K,K], |\gL(x_1,y)-\gL(x_2,y)|\leq C_K|x_1-x_2|_2$. In particular this is true for the square loss.

Combining convexity of the loss, the first part of the theorem already proved and Lipschitzness on a compact domain of the loss, we get the second part of the theorem with a term $C_{K(T)}\err$ instead of $\err$. To solve this, just consider $\min\left\{\frac{\err}{C_T}, \err\right\}$ in the first part of the theorem.

We have supposed in the beginning that the non-linearity was bounded. But the previous study shows that with probability at least $1-\delta$, $|w_{A,i}x|$ is upper bounded for all $i$ during the training up to time T by some constant depending only on $T,\delta$ provided $N\geq N_{\min}, \epsilon\leq \epsilon_{\max}(N)$
Therefore assuming that the non-linearity is bounded on a big enough compact set is enough to get the the result since it doesn't change the dynamics of the parameters considered. However the size of this set is not made explicit here.

\subsection{Proving LMC for general SGD}

We will now study LMC of neural networks trained under general SGD using Theorem 4 part B from \citet{mei2019mean}. The study is very similar to the case of noiseless SGD and will yield similar results.

More precisely in that case the PDE writes as:

\begin{align*}
&\partial_t\rho_t=2\xi(t)\nabla_\theta \cdot(\rho_t(\theta)\nabla_\theta\Psi_\lambda(\theta;\rho_t))+2\xi(t)\tau d^{-1}\Delta_\theta\rho_t\\
&\Psi_\lambda(\theta;\rho)=\Psi(\theta;\rho)+\frac{\lambda}{2}\|\theta\|_2^2
\end{align*}

Notice that our \Cref{ass:noiseless_SGD,ass:noisy_SGD} imply assumptions 1 to 6 in \citet{mei2019mean}.

\subsubsection{Intermediate dynamics for general SGD}

\citet{mei2019mean} define as before intermediate dynamics:

\textbf{Non linear dynamics}

Let consider $\bar{\theta}_i^t$ with initialization $\bar{\theta}_i^0\sim \rho_0$ i.i.d. which follows the dynamics 

\begin{equation*}
    \bar{\theta}_i^t=\bar{\theta}_i^0+2\int_0^t\xi(s)G(\bar{\theta}_i^s;\rho_s)ds + \int_0^t\sqrt{2\xi(s)\tau d^{-1}}dW_i(s)
\end{equation*}

where $G(\theta;\rho)=-\nabla\Psi_\lambda(\theta;\rho)$.
An important fact is that $\bar{\theta}_i^t$ is random because of the random initialization and its law at time $t$ is $\rho_t$. It corresponds to the evolution of particles under a field which depends only on the position of the optimized particle and the overall distribution of all particles plus and a diffusion term.

\textbf{Particle Dynamics}

Let $\ubar{\theta}_i^t$ with initialization $\ubar{\theta}_i^0=\bar{\theta}_i^0$ with the following dynamics where $\ubar{\rho}_t^{(N)}=\frac{1}{N}\sum_{i=1}^N\delta_{\ubar{\theta}_i^t}$ denote the empirical distribution of $\ubar{\theta}_i^t$.

\begin{equation*}   \ubar{\theta}_i^t=\ubar{\theta}_i^0+2\int_0^t\xi(s)G(\ubar{\theta}_i^s;\ubar{\rho}_s^{(N)})ds + \int_0^t\sqrt{2\xi(s)\tau d^{-1}}dW_i(s)
\end{equation*}

\textbf{Gradient descent dynamics}

Let $\Tilde{\theta}_i^{k}$ with initialization $\Tilde{\theta}_i^0=\bar{\theta}_i^0$ with the following dynamics:

\begin{equation*}    \Tilde{\theta}_i^{k}=\Tilde{\theta}_i^0+2\epsilon\sum_{l=0}^{k-1}\xi(l\epsilon)G(\Tilde{\theta}_i^l;\tilde{\rho}_l^{(N)})+ \int_0^{k\epsilon}\sqrt{2\xi([s])\tau d^{-1}}dW_i(s)
\end{equation*}

\textbf{Stochastic Gradient Descent Dynamics}

Consider $\theta_i^k$ with initialization $\theta_i^0=\bar{\theta}_i^0$ that follows:

\begin{equation*}   \theta_i^k=\theta_i^0+2\epsilon\sum_{l=0}^{k-1}\xi(l\epsilon)F_i(\theta^l;z_{l+1})+ \int_0^{k\epsilon}\sqrt{2\xi([s])\tau d^{-1}}dW_i(s)
\end{equation*}

where $F_i(\theta^k;z_{k+1})=-\lambda \theta^k_i+(y_{k+1}-\hat{y}_{k+1})\nabla_{\theta_i}\sigma_*(x_{k+1};\theta_i^k)$, $z_k=(x_k,y_k)$ and $\hat{y}_{k+1}=\frac{1}{N}\sum_{j=1}^N\sigma_*(x_{k+1};\theta_j^k)$

As before we first control the distance between noisy SGD and non linear dynamics with the following lemma:

\begin{lemma}
    \label{lemma:mean_field_closeness_noisy} Consider a two-layer neural network with notations as before trained with noisy SGD for an underlying time $T$. Assume $T\geq 1$. Then under assumptions \Cref{ass:noiseless_SGD,ass:noisy_SGD}, there exists a constant $K$ such that with probability at least $1-3e^{-z^2}$ we have:

    \begin{equation*}
        \begin{split}
            \max_{k\in [0,T/\epsilon]\bigcap \sN}\max_{i\in [N]}\|\theta_i^k-\bar{\theta}_i^{k\epsilon}\|_2&\leq Ke^{e^{KT}[\sqrt{\log(N)}+z^2]}[\sqrt{d\log(N)}+z^3+\log^{3/2}(NT)]/\sqrt{N}\\
            &+Ke^{e^{KT}[\sqrt{\log(N)}+z^2]}[\log(N(T/\epsilon \vee 1))+z^4]\sqrt{\epsilon}\\
            &+Ke^{e^{KT}[\sqrt{\log(N)}+z^2]}[\sqrt{d}\log(N)+z^3+\log^{3/2}(N)]\sqrt{\epsilon}
        \end{split}     
    \end{equation*}
\end{lemma}

\begin{proof}
    Just apply proposition 47,49,50 from \citet{mei2019mean}.
\end{proof}

We moreover recall here Lemma 9 from \citet{mei2019mean} which bounds the value of the second layer coefficients $A^t=(a_1^t, \ldots,a_N^t)$.

\begin{lemma}[Lemma 19 in \citet{mei2019mean}]
\label{lemma:mean_field_boundedness_noisy}
    There exists a constant $K$ such that with probability at least $1-e^{-z^2}$ we have
\begin{equation*}
    \sup_{t\in [0,T]}\|A^t\|_\infty\leq Ke^{KT}[\sqrt{\log(N)}+z]
\end{equation*} 
\end{lemma}

\subsubsection{Proof of \Cref{theorem:mean_field} in the case of noisy regularized SGD}

We can now prove LMC for two two-layer networks trained with general SGD:

\begin{mdframed}[style=MyFrame2]
   \theoremmeanfield* 
\end{mdframed}

\begin{proof}
    We follow the same steps as before. Recall that the data distribution is bounded: $\Supp(P)\subset [-H_x,H_x]^d\times [-H_y,H_y]$.

The problem is that due to the stochasticity added in the noisy SGD, $\Supp(\rho_t)$ is not necessarily bounded anymore. However, using step 3 of the proof of lemma 41 in \citet{mei2019mean} and the fact that the initial distribution $\rho_0$ has bounded support (sub-Gaussian would be enough), the distribution of weights of the first layer at time $T$ is sub-Gaussian. 

Indeed if $\bar{\theta}^t_i\sim \rho_t$ and $\rho_0$ is bounded or sub-Gaussian, we get the existence of $K$ such that:

\begin{equation*}
    \sP(\|\bar{\theta}^T_i\|_2^2\geq Ke^{KT}(1+z)\sqrt{T})\leq e^{-z^2}
\end{equation*}

which proves that $\rho_T$ is sub-Gaussian.

We could adapt the proof done in \Cref{expect_wass_normal} for Gaussian variable to sub-Gaussian variable to show the existence of constants (\Cref{expect_wass_normal} dealt with $\gW_2^2$ but can be extended to $\gW_1$ because Proposition 15 of \citet{weed2019sharp} is valid for any $\gW_p^p$) $D_2', E_2'$ depending only on the constant $K$ of sub-Gaussianity of the previous distribution, and hence independent of $N$ such that by considering the norm $\|\cdot\|_2$, we can still bound for $m\geq E_2^{'d}$

\begin{equation*}
    \E[\gW_1(\hat{\mu}_A,\hat{\mu}_B)]\leq \sqrt{\frac{D'_2}{d}\log(N)}\left(\frac{1}{N}\right)^{1/d}
\end{equation*}

Therefore, with probability at least $1-\frac{\delta}{2}-6e^{-z^2}$:
    \begin{equation*}
    \begin{split}
        \forall X \in [-H_x,H_x]^D, &\frac{\|\sigma(W_AX)-\sigma((tW_A+(1-t)\Tilde{W}_B)X)\|_1}{N}\leq g_2(T,\epsilon,N,\delta)\\
        &\eqdef
        L_\sigma H_x \sqrt{d}\frac{1}{\delta}\sqrt{\frac{D'_2}{ d}\log(N)}N^{-1/d}\\
        &+2L_\sigma H_x \sqrt{d} (Ke^{e^{KT}[\sqrt{\log(N)}+z^2]}[\sqrt{d\log(N)}+z^3+\log^{3/2}(NT)]/\sqrt{N}\\
            &+Ke^{e^{KT}[\sqrt{\log(N)}+z^2]}[\log(N(T/\epsilon \vee 1))+z^4]\sqrt{\epsilon}\\
            &+Ke^{e^{KT}[\sqrt{\log(N)}+z^2]}[\sqrt{d}\log(N)+z^3+\log^{3/2}(N)]\sqrt{\epsilon})
    \end{split}
    \end{equation*}

and same for the second term.
We moreover have, using \Cref{lemma:mean_field_closeness_noisy,lemma:mean_field_boundedness_noisy} that with probability at least $1-2e^{-z^2}$:

\begin{equation*}
\begin{split}
    \max\{\|A_A\|_\infty,\|A_B\|_\infty\}&\leq g_1(T,z,N,\epsilon)\\
    &:=Ke^{KT}[\sqrt{\log(N)}+z]+Ke^{e^{KT}[\sqrt{\log(N)}+z^2]}[\sqrt{d\log(N)}+z^3+\log^{3/2}(NT)]/\sqrt{N}\\
            &+Ke^{e^{KT}[\sqrt{\log(N)}+z^2]}[\log(N(T/\epsilon \vee 1))+z^4]\sqrt{\epsilon}\\
            &+Ke^{e^{KT}[\sqrt{\log(N)}+z^2]}[\sqrt{d}\log(N)+z^3+\log^{3/2}(N)]\sqrt{\epsilon}
        \end{split}     
\end{equation*}

Taking $z=\sqrt{\log\left(\frac{16}{\delta}\right)}$ such that $8e^{-z^2}=\frac{\delta}{2}$
we get that with probability at least $1-\delta$:

\begin{equation*}
    \begin{split}
        |t\hat{f}(x;\theta_{A})&+(1-t)\hat{f}(x;\tilde{\theta}_{B})-\hat{f}(x;t\theta_{A}+(1-t)\tilde{\theta}_B)|\leq g_1(T,z,N,\epsilon)g_2(T,z,N,\epsilon)
    \end{split}
\end{equation*}

As before, for fixed $T,\delta$, denote $\err(N,\epsilon)$ the left hand term.

It is clear that:

\begin{equation*}
       \forall \err> 0 \exists N_{\min}\forall N\geq N_{\min} \exists \epsilon_{\min}(N)\forall \epsilon\leq \epsilon_{\min}(N),\,
       \err(N,\epsilon)\leq \err
\end{equation*}

Therefore, sending $N\rightarrow \infty$, $\epsilon \rightarrow 0$ brings immediately the first part of the theorem.

To get the second part of the theorem, we do the same procedure as for noiseless regularization-free SGD.

Namely, with probability $1-\delta$ we have both:

\begin{equation*}
        \begin{cases}
        P-\text{almost surely}, |t\hat{f}(x;\theta_{A})+(1-t)\hat{f}(x;\theta_{B})-\hat{f}(x;t\theta_{A}+(1-t)\Tilde{\theta}_B)|\leq \err \\
        \max\{\|A_A\|_\infty,\|A_B\|_\infty\}\leq g_1(T,N,z, \epsilon)
        \end{cases}
\end{equation*}
Up to changing $N_{\min}, \epsilon_{\max}(N)$ we can suppose that $\forall N \geq N_{\min}, \forall \epsilon\leq \epsilon_{\max}(N)$, we have 

\begin{equation*}
\begin{split}
    Ke^{e^{KT}[\sqrt{\log(N)}+z^2]}&[\sqrt{d\log(N)}+z^3+\log^{3/2}(NT)]/\sqrt{N}\\
    &+Ke^{e^{KT}[\sqrt{\log(N)}+z^2]}[\log(N(T/\epsilon \vee 1))+z^4]\sqrt{\epsilon}\\
    &+Ke^{e^{KT}[\sqrt{\log(N)}+z^2]}[\sqrt{d}\log(N)+z^3+\log^{3/2}(N)]\sqrt{\epsilon}\leq 1
\end{split}
\end{equation*}

which brings 
\begin{equation*}
\label{eq:boundeness_noisy}
    g_1(T,N,z, \epsilon)\leq Ke^{KT}[\sqrt{\log(N)}+z]+1
\end{equation*}

\Cref{eq:boundeness_noisy}, boundness of the activation, boundness of the input distribution $(x,y) \sim P$ by assumption imply the existence of $K'$ such that $\forall t \in [0,1]$, $P-$almost-surely and for $N$ large enough,

\begin{equation*}
        \begin{cases}
        |\hat{f}(x;\theta_{A})|\leq K'\left(Ke^{KT}[\sqrt{\log(N)}+z]\right)\\
        |\hat{f}(x;\theta_{B})|\leq K'\left(Ke^{KT}[\sqrt{\log(N)}+z]\right)\\
        |y|\leq H_y \leq K'\left(Ke^{KT}[\sqrt{\log(N)}+z]\right)
        \end{cases}
\end{equation*}

Using convexity and Lipschitzness of the squared loss on compact domains we get the existence (and this is a sufficient condition for the loss function with convexity) of $\Lip:\sR_+\rightarrow \sR_+$ such that: $\exists L_1,L_2, \forall H \in \sR_+, \, \Lip(H)\leq L_1+L_2\exp(H)$, $\forall x_1,x_2,y\in [-H,H], |\gL(x_1,y)-\gL(x_2,y)|\leq \Lip(H)|x_1-x_2|$ and such that with probability at least $1-\delta$:

\begin{equation*}
    \begin{split}
        \sE[\gL(\hat{f}(x;t\theta_A+(1-t)\Tilde{\theta}_B),y)]\leq \Lip(K'\left(Ke^{KT}[\sqrt{\log(N)}+z]+1\right))g_1(T,z,\epsilon,N) g_2(T,z,\epsilon,N) 
    \end{split}
\end{equation*}

Plugging this back and with the exact same discussion as before we get $\exists N_{\min} \forall N\geq N_{\min} \exists \epsilon_{\max}\forall\epsilon\leq\epsilon_{\max}$,

\begin{equation*}
    \begin{split}
        \sE[\gL(\hat{f}(x;t\theta_A+(1-t)\Tilde{\theta}_B),y)]&\leq \err +t\sE[\gL(\hat{f}(x;\theta_A),y)]+(1-t)\sE[\gL(\hat{f}(x;\theta_B),y)] 
    \end{split}
\end{equation*}

To get both part 1 and 2 of the theorem at the same time we just have to reconsider the $\max$ of both $N_{\min}$ and the $\min$ of both $\epsilon_{\max}$.

\end{proof}

\subsection{On the satisfiability of our assumptions}

Assumption 1 and 2 are non-trivial but standard in the mean field literature. They are made to ensure that the optimization of the two-layer neural networks happens in the mean-field regime. Indeed, as explained above the weights are then approximately indepedent and we can leverage Wasserstein convergence bounds of empirical measure to prove linear mode connectivity. We used conventioanl assumptions from the mean-field litterature 
(e.g. see assumptions A1 to A6 in \cite{mei2019mean}, A1 to A4 in \cite{mei2018mean}). $u,v,U,V$ are implicitly defined but once the non-linearity and the data distribution are fixed, $u,v,U,V$ are fully determined as functions of the parameters. 
Checking their derivability can be done using usual rules for derivation under the integral sign if the non-linearity is smooth. A particular case is to consider a two-layer network, with a sigmoïd activation, a bounded data distribution and a bounded uniform initialization over the parameters. We additionally mention that a lot of works empirically evidence the validity of the mean-field framework, hence we feel validating our approach. In the case of mltilayer networks, the main assumption is the independence of weights inside each layer (\cref{new_assumption:0}). Multiple recent works address this question by using a mean field view for multi-layer networks with bounds on the width needed during optimization with SGD (e.g. Th. 15 in \cite{nguyen2023rigorous}). Finally, we believe our results could be extended to approximated independence of weights with an additional error term for the error barrier on a linear path corresponding to the approximated independence. Quantifying the impact of correlation between weights constitutes a very interesting avenue for future work.

\section{EXPERIMENTS}

\subsection{Experiment on CIFAR10}

We compared activation and weight matching methods on the CIFAR10 dataset for a VGG16 model. Our experiment
again shows the correlation between small approximate
dimension of the support of the weight distribution and LMC effectiveness
hence supporting our main theoretical study. As it is non trivial to compute the covariance of the input with convolutional layers
as it is a high dimensional tensor we left the alternative weight matching methods as a future work. Providing a scalable
technique to estimate such a covariance for CNNs is a an interesting research direction beyond the scope of this paper.

\begin{figure*}[h!]
     \centering
     \begin{subfigure}[b]{0.48\textwidth}
         \includegraphics[width=\linewidth]{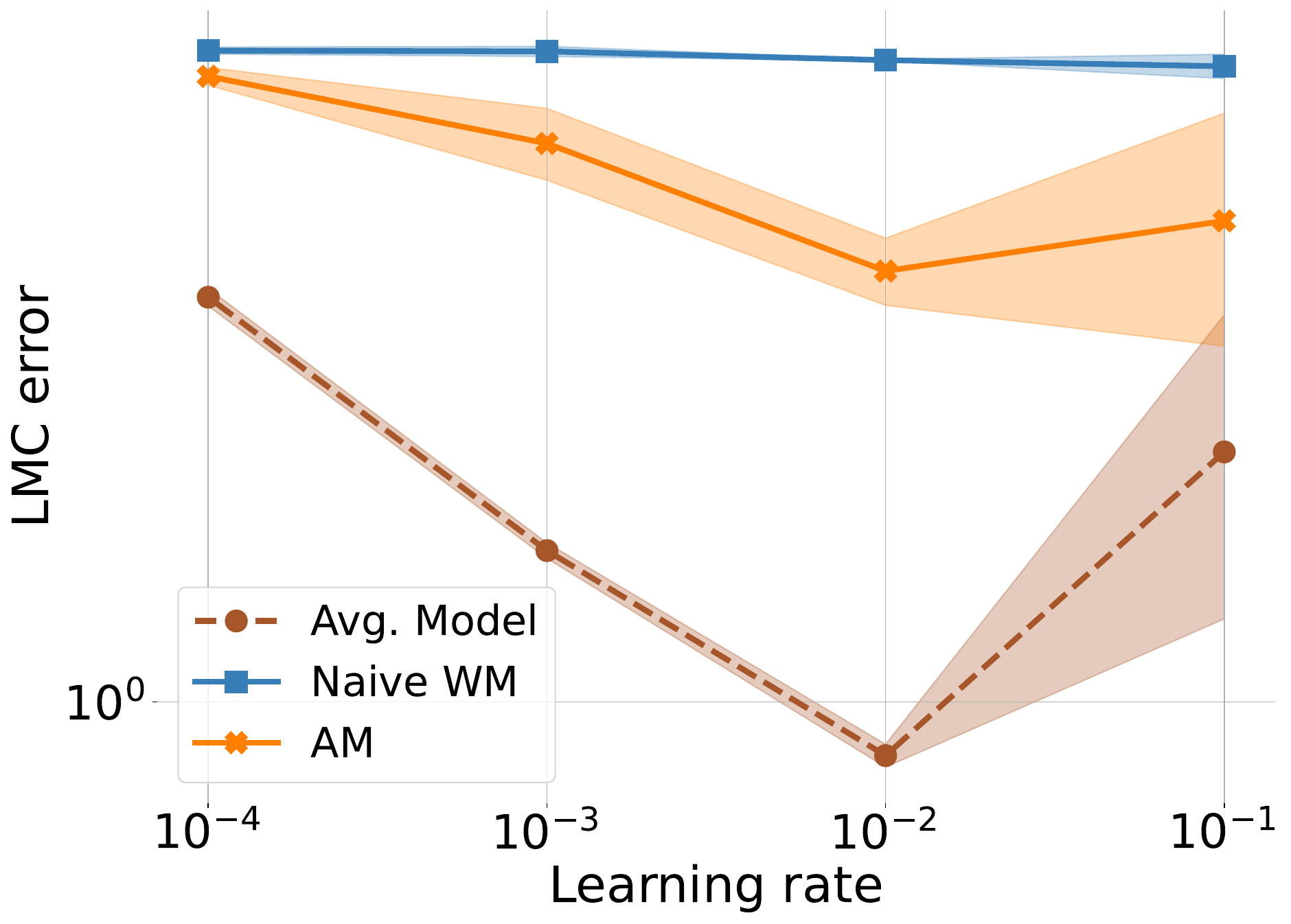}
         \caption{Mean test loss of the trained networks $A$ and $B$ and error barrier on the linear path $M_t, \, t\in [0,1]$ across different learning rate values for each matching problem.}
     \end{subfigure}
     \hfill
     \begin{subfigure}[b]{0.48\textwidth}
              \includegraphics[width=\linewidth]{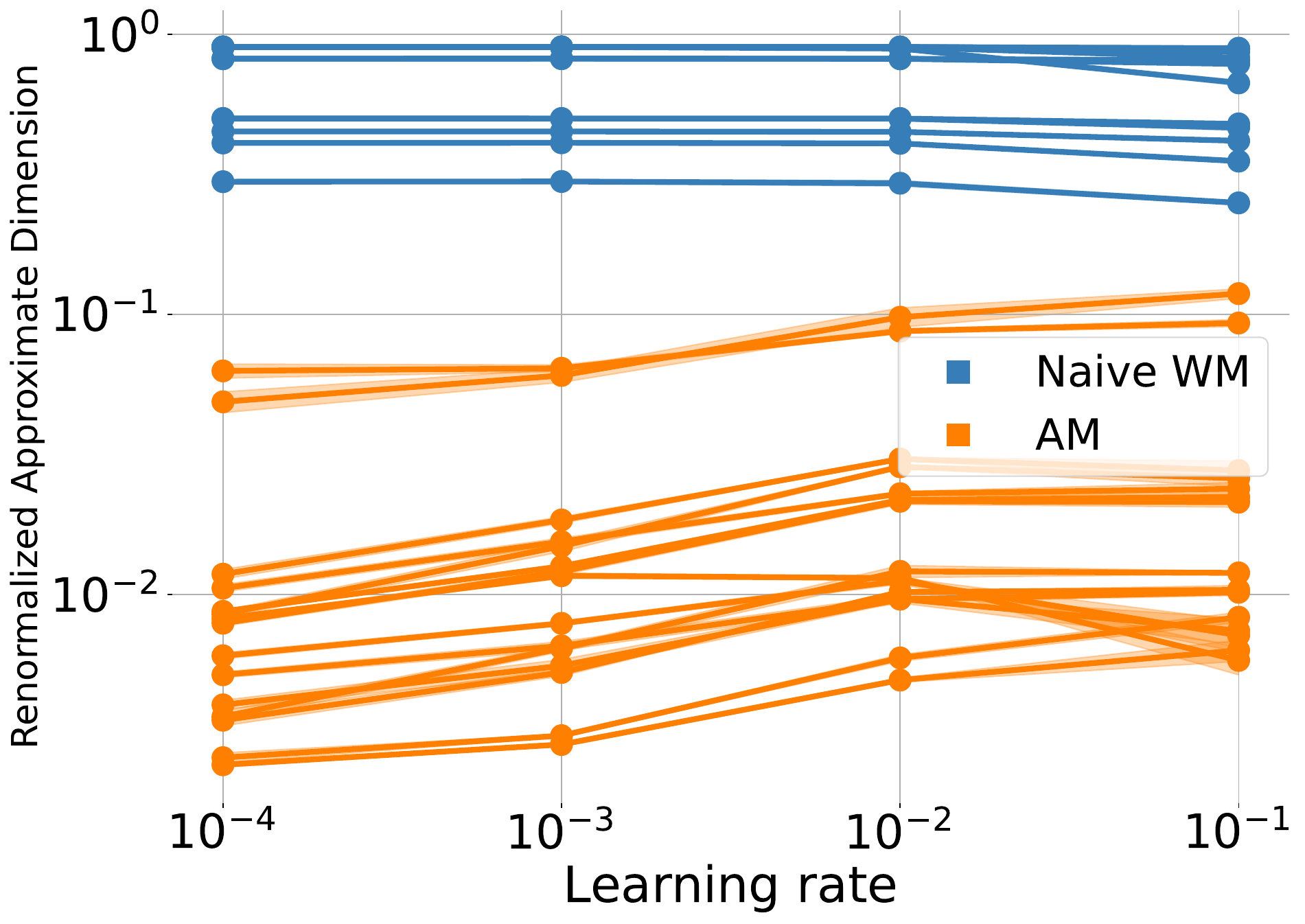}
     \caption{Approximate dimension $\Dim(S) \eqdef \nicefrac{\tr(S)^2}{\tr(S^2)}$ of the matrices considered in the matching problems at each layer.
     \newline ~
     }
     \label{fig:approx_dim_final_CIFAR10}
     \end{subfigure}
     \caption{\small
     Statistics of the average network $M$ over the linear path between networks $A$ and $B$ using respectively weight matching (blue) and activation matching (orange) \vspace{-.4cm}}
\label{fig:LMC_error_CIFAR10}
\end{figure*}

\subsection{Details about our new weight matching method}
\label{appendix:explanation_method}

Until now we have studied the influence of the dimension of the support of the underlying distribution of weights on the convergence rate in Wasserstein distance of the corresponding empirical measure. An interesting question is to look at the influence of the distance used to define the Wasserstein distance. 

More precisely, consider a single layer of two networks $A,B$ with input $X\in \sR^n$ and matrix weights $W_{A,B}\in \gM_{m,n}(\sR)$.
Consider that the input data follows a distribution $P$ with $\E_P[XX^T]=\Sigma$

The underlying method that we use in our proof and which is the one referred to as weight matching method in \citet{ainsworth2022git} consists in minimizing the distances for the euclidean norm between weights matrices, i.e. to find:

\begin{equation*}
    \argmin_{\Pi \in \gS_m}\|W_A-\Pi W_B\|_2
\end{equation*}

This is equivalent to finding:

\begin{equation*}
    \argmin_{\Pi\in \gS_m}\sqrt{\frac{1}{m}\|W_A-\Pi W_B\|_2^2\}}=\argmin_{\pi\in \gS_m} \sqrt{\frac{1}{m}\sum_{i=1}^m\|W_{A,i:}-W_{B,\pi_i:}\|_2^2}
\end{equation*}

We get an expected square error between network $A$ and network $B$ permuted at the output layer of:

\begin{equation*}
    \E\left[\|W_AX-\Pi W_BX\|_2^2\right]=(W_A-\Pi W_B)\Sigma(W_A-\Pi W_B)^T=\|W_A-\Pi W_B\|_{2,\Sigma}^2
\end{equation*}
where $\|\cdot\|_{2,\Sigma}$ is the semi-norm coming from $(X,Y)\rightarrow X^T \Sigma Y$ which is a symmetric positive bilinear product since $\Sigma$ is symmetric positive (and a norm when $\Sigma$ is definite positive i.e., when $\Span(\Supp(P))=\sR^n$).

Minimizing the cost $\|W_A-\Pi W_B\|_2$ contributes to minimizing the expected squared error $\|W_A-\Pi W_B\|_{2,\Sigma}^2$ but it appears more natural to directly minimize the cost $\|W_A-\Pi W_B\|_{2,\Sigma}^2$.

As explained in \Cref{theorem:gain} below, we can directly link the approximate dimension of the underlying covariance matrix of each method with the decay rate of LMC error barrier. The underlying covariance matrix of each method is $W_A^\ell[W_A^\ell]^T$ for WM (naive),  $W_A^\ell\Sigma_A^{\ell-1}[W_A^\ell]^T$ for WM (ours) and $\Sigma_A^{\ell}$ for AM.

Indeed,

\begin{itemize}
    \item for naive weight matching, each row of $W_A,W_B$ follows a distribution with covariance matrix $W_A^\ell[W_A^\ell]^T$,
    \item for weight matching (ours), the optimization problem can be seen (\Cref{change_basis}) as for naive weight matching but with covariance matrix $W_A^\ell\Sigma_A^{\ell-1}[W_A^\ell]^T$,
    \item for activation matching, each row of $Z_A^\ell$ follows a distribution with covariance matrix $\Sigma_A^{\ell}$.
\end{itemize}

\subsection{Gain of our new weight matching method}

This section is motivated by the following question:
\begin{center}
    \textbf{What is the gain of optimizing directly the cost $\|W_A-\Pi W_B\|_{2,\Sigma}$ when $\Sigma$ is low dimensional?} 
\end{center}

For example, let's suppose that $\Sigma=\Diag(1,1,0,...,0)$ and hence the support of $X$ is two dimensional. Suppose moreover that $W_A$ and $W_B$ are as in $\Cref{section:given_distribution_normal}$ initialized i.i.d. with a distribution $\gN(0,\frac{I_n}{n})$ on the weights.
Hence we have seen before that $\|W_A-\Tilde{W}_B\|_2\sim\left(\frac{1}{m}\right)^{1/n}$. Since the minimization procedure is unaware of the structure of $\Sigma$ it is clear by symmetry that $\sqrt{\|W_{A,1:}-\Tilde{W}_{B,1:}\|_2^2+\|W_{A,2:}-\Tilde{W}_{B,2:}\|_2^2}\sim \sqrt{\frac{2}{n}}\left(\frac{1}{m}\right)^{1/n}$. Therefore the convergence is still as $\left(\frac{1}{m}\right)^{1/n}$.
However if we had first aimed at minimizing $\|W_A-\Pi W_B\|_{2,\Sigma}$ it is clear that the problem becomes two dimensional and hence $\argmin_{\Pi \in \gS_m}\|W_A-\Pi W_B\|_{2,\Sigma}\sim \left(\frac{1}{m}\right)^{1/2}$ which is extremely faster when $n$ is large.

We want to apply this idea to our setting where we suspect the distribution of activations at each layer to be low dimensional. We now prove the following lemma:

\begin{lemma}
\label{change_basis}
Let $W_A,W_B\in \gM_{m,n}(\sR)$ satisfy \Cref{new_assumption:0} with underlying distribution $\mu$ and let $\Sigma\in \gM_n(\sR)$. Write $\Sigma=O\sqrt{\Sigma}^2O^T$ where O is orthogonal and $\sqrt{\Sigma}$ is diagonal. Then we get the equivalence between optimization problems:
    \begin{equation}
       \E\left[\min_{\Pi\in \gS_m}\|W_A-\Pi W_B\|_{2,\Sigma}^2\right]= \E\left[\min_{\Pi\in \gS_m}\|\hat{W}_A-\Pi \hat{W}_B\|_2^2\right]
    \end{equation}
    where $\hat{W}_B,\hat{W}_B$ satisfy \Cref{new_assumption:0} with underlying distribution $f_*\mu$ the image measure of $\mu$ by $f:X\mapsto O\sqrt{\Sigma}O^TX$
\end{lemma}

\begin{proof}
    Just notice that $\forall \Pi \in \gS_m$
    \begin{equation*}
        \|W_A-\Pi W_B\|_{2,\Sigma}^2= \tr[(W_A-\Pi W_B)O\sqrt{\Sigma}O^T(O\sqrt{\Sigma}O^T)^T(W_A-\Pi W_B)^T]
    \end{equation*}
    and do the change of variable $\hat{W}_A=W_AO\sqrt{\Sigma}O^T$ (repectively $\hat{W}_B$)
\end{proof}

\begin{restatable}{theorem}{gain}
\label{theorem:gain} Consider $X \in \sR^n\sim P\in \gP_1(\sR^n)$ such that $\E_P[XX^T]=\Sigma=\Diag(1,\ldots,1,0,\ldots,0)$, $\rank(\Sigma)=\Tilde{n}\leq n$ and $W_A, W_B$ random weight matrices satisfying \Cref{new_assumption:0} with underlying distribution $\gN\left(0,\frac{I_n}{n}\right)$
    Denote $\Pi_1,\Pi_2\in \gS_m$ random permutations that minimize the respective costs $\|W_A-\Pi W_B\|_2$ and $\|W_A-\Pi W_B\|_{2,\Sigma}$
    Then we have: 
    
    \begin{equation*}
        \begin{split}
            &\E\left[\|W_A-\Pi_1W_B\|_{2,\Sigma}^2\right]=\Tilde{\Omega}\left(\left(\frac{1}{m}\right)^{2/n}\right)\\
            &\E\left[\|W_A-\Pi_2W_B\|_{2,\Sigma}^2\right]=\Tilde{\Omega}\left(\left(\frac{1}{m}\right)^{2/\Tilde{n}}\right)\\
            &\E[\|W_A-\Pi_1W_B\|_{2,\Sigma}^2]= \Tilde{\gO}\left(\left(\frac{1}{m}\right)^{2/n}\right)\\
            &\E[\|W_A-\Pi_2W_B\|_{2,\Sigma}^2]=\Tilde{\gO}\left(\left(\frac{1}{m}\right)^{2/\Tilde{n}}\right)
        \end{split}
    \end{equation*}
\end{restatable}

\begin{proof}

    Using \Cref{change_basis}, we see that bounds $2$ and $4$ are just corollaries of \Cref{normal:new_assumption1,theorem:lower_bound}.

    To show bounds $1$ and $3$ just notice that:

    \begin{equation*}
        \Pi_1=\argmin_{\Pi \in \gS_m}\|W_A-\Pi W_B\|_2
    \end{equation*}
    is almost surely unique.

    By symmetry of the problem and \Cref{theorem:lower_bound} we therefore see that $\forall i \in [n]$:

    \begin{equation*}
        \E\|[W_A-\Pi_1 W_B]_{:i}\|_2^2 = \frac{1}{n}\E\|W_A-\Pi_1 W_B\|_2^2=\Tilde{\Omega}\left(\left(\frac{1}{m}\right)^{2/n}\right)
    \end{equation*}

    Finally noticing that $\Sigma=\Diag(1,\ldots,1,0,\ldots,0)$ we get by summing:
\begin{equation*}
    \E\left[\|W_A-\Pi_1W_B\|_{2,\Sigma}^2\right]=\frac{\Tilde{n}}{n}\Tilde{\Omega}\left(\left(\frac{1}{m}\right)^{2/n}\right)=\Tilde{\Omega}\left(\left(\frac{1}{m}\right)^{2/n}\right)
\end{equation*}  
Similarly, exploiting a.s. uniqueness of $\Pi_1$ and symmetry across dimensions, we get the third inequality.
\end{proof}

\end{document}